\documentclass{article}

\usepackage{microtype}
\usepackage{graphicx}
\usepackage{booktabs} % for professional tables

\usepackage{hyperref}
\usepackage{aisecure-math}
\usepackage{xspace}

\newcommand{\bo}[1]{\textcolor{blue}{Bo: #1}}
\newcommand{\jiawei}[1]{\textcolor{purple}{Jiawei: #1}}

\usepackage[final, nonatbib]{neurips_2022}

\usepackage[utf8]{inputenc} % allow utf-8 input
\usepackage[T1]{fontenc}    % use 8-bit T1 fonts
\usepackage{hyperref}       % hyperlinks
\usepackage{url}            % simple URL typesetting
\usepackage{booktabs}       % professional-quality tables
\usepackage{amsfonts}       % blackboard math symbols
\usepackage{nicefrac}       % compact symbols for 1/2, etc.
\usepackage{microtype}      % microtypography
\usepackage{wrapfig}
\usepackage{colortbl}
\usepackage{wrapfig}
\definecolor{tabgray}{gray}{0.90}
\usepackage{graphicx}
\usepackage{siunitx}
\usepackage{url}

\usepackage{support-caption}
\usepackage{subcaption}

\usepackage{dsfont}
\usepackage{nicefrac}
\usepackage{multirow}
\usepackage{microtype}
\usepackage{graphicx}
\usepackage{amsthm}
\usepackage{xspace}
\usepackage{mathtools}
\usepackage{bbm}
\usepackage{fancyhdr}
\usepackage{cleveref}
\usepackage{xcolor}
\usepackage{url}
\usepackage{xr}
\usepackage{float}
\usepackage{threeparttable}
\usepackage{enumitem}
\usepackage{algorithm}
\usepackage{algorithmic}
\usepackage{xspace}

\theoremstyle{plain}

\usepackage{enumitem}

\newcount\Comments  % 1 suppresses notes to selves in text
\usepackage{color}
\definecolor{darkgreen}{rgb}{0,0.5,0}
\definecolor{darkblue}{rgb}{0,0,0.5}
\definecolor{purple}{rgb}{1,0,1}

\newtheorem*{claim*}{Claim}
\theoremstyle{definition}

\newtheorem*{remark*}{Remark}
\usepackage{xspace}
\usepackage[normalem]{ulem}
\useunder{\uline}{\ul}{}
\RequirePackage{algorithm}
\RequirePackage{algorithmic}

\newcommand{\abs}[1]{\left\vert#1\right\vert}
\newcommand{\norm}[1]{\left\Vert#1\right\Vert}

\newcommand{\sr}{sensing-reasoning pipeline\xspace}

\usepackage{xifthen}
\renewcommand{\Pr}[2][]{ \ifthenelse{\isempty{#1}}
  {\mathbf{Pr}\left[#2\right]} {\mathbf{Pr}_{#1}\left[#2\right]} } % Use \Pr[a]{b} for \mathbf{Pr}_a[b], \Pr{b} for  \mathbf{Pr}[b]
\usepackage{amsmath}
\usepackage{amsfonts}
\usepackage{breakcites}

\def\reduceT{\le_\texttt{t}}

\title{Improving Certified Robustness
via Statistical Learning with Logical Reasoning}

\author{%
  Zhuolin Yang\thanks{The first two authors contribute equally to this work.} \\
  UIUC\\
  \texttt{zhuolin5@illinois.edu} \\
  \And
  Zhikuan Zhao$^*$ \\
  ETH Zürich\\
  \texttt{zhikuan.zhao@inf.ethz.ch} \\
  \And
  Boxin Wang \\
  UIUC\\
  \texttt{boxinw2@illinois.edu} \\
  \And
  Jiawei Zhang \\
  UIUC\\
  \texttt{jiaweiz7@illinois.edu} \\
  \And
  Linyi Li \\
  UIUC\\
  \texttt{linyi2@illinois.edu} \\
  \And
  Hengzhi Pei \\
  UIUC\\
  \texttt{hpei4@illinois.edu} \\
  \And
  Bojan Karlaš \\
  ETH Zürich \\
  \texttt{karlasb@inf.ethz.ch} \\
  \And
  Ji Liu \\
  Kwai Inc. \\
  \texttt{ji.liu.uwisc@gmail.com} \\
  \And
  Heng Guo \\
  University of Edinburgh \\
  \texttt{hguo@inf.ed.ac.uk} \\
  \And
  Ce Zhang \\
  ETH Zürich \\
  \texttt{ce.zhang@inf.ethz.ch} \\
  \And
  Bo Li \\
  UIUC \\
  \texttt{lbo@illinois.edu} \\
}

\begin{document}

\maketitle

\begin{abstract}

% Recent, intensive algorithmic efforts around
% certified robustness have
% enabled rapid improvements of certification
% accuracy for complex ML models. However,
% there are still many applications in which 
% even the most advanced certification 
% methods cannot provide high-enough 
% certification accuracy. 
% \textit{What else can we do in these cases?}

% In this paper, we focus on an orthogonal 
% perspective, complementary to 
% these recent algorithmic efforts --- \textit{Can we
% improve state-of-the-art certified robustness
% methods by integrating domain knowledge?}
% To this end, we propose a novel 
% framework in which the robustness
% certificates from 
% state-of-the-art methods
% are  
% further processed by a Markov logic network,
% which contains logic rules encoding 
% domain knowledge. 
% We first rigorously define 
% certified robustness under our framework and 
% theoretically analyze the computational complexity of robustness certification.
% We then
% derive certified robustness 
% bounds and provide a practical algorithm,
% showing that it is often possible to certify the robustness of the whole pipeline even against a large magnitude of perturbations. 
% Finally, we conduct extensive experiments on both image and natural language datasets.
% We show that with domain knowledge, we can
% provide significant improvements of 
% certified robustness over 
% state-of-the-art methods.

Intensive algorithmic efforts have been made to enable the rapid improvements of certificated
robustness for complex ML models recently. However,
current robustness certification methods are only able to certify under a limited perturbation radius. 
% there are still many applications in which 
% even the most advanced certification 
% methods cannot provide satisfactory 
% certificated accuracy. 
% \textit{What else can we do in these cases?}
Given that existing \textit{pure data-driven} statistical approaches have reached a bottleneck, in this paper, we propose to integrate statistical ML models with {knowledge} (expressed as logical rules) as a \textit{reasoning} component using Markov logic networks (MLN), so as to further improve the overall certified robustness.
This opens new research questions
about certifying the robustness of such a paradigm, especially the reasoning component (e.g., MLN).
As the first step towards understanding
these questions, we first prove that the computational complexity of certifying the robustness of MLN is \textsf{\#P}-hard. 
Guided by this hardness result, we then derive the first certified robustness bound for MLN by carefully analyzing different model regimes. 
% We show that for reasoning components
% such as MLN and a specific family of Bayesian networks it is possible to certify the robustness of the whole
% pipeline even against a large magnitude of perturbation.
%We also prove that the certified robustness bound for
%Bayesian networks is tight. 
Finally, we conduct extensive experiments on five datasets including both high-dimensional images and natural language
texts, and we show that the certified robustness with knowledge-based logical reasoning indeed
significantly
% We conduct extensive experiments on several high-dimensional datasets and show that our proposed method 
outperforms that of the state-of-the-arts.
\end{abstract}

\vspace{-3mm}
\section{Introduction}
\vspace{-2mm}

Given extensive studies on adversarial attacks against ML models recently~\cite{carlini2017towards,eykholt2018robust,qiu2020semanticadv,li2020qeba,zhang2021progressive,li2020nolinear,xiao2018characterizing}, building models that are robust against such 
attacks is an important and emerging topic. 
Thus, a plethora of \textit{empirical defenses} have been proposed to improve the ML robustness~\cite{ma2018characterizing,yangli2021trs,li2014feature,shafahi2019adversarial,xiao2018characterizing,xiao2019advit}; however, most of these are attacked again by  stronger adaptive attacks~\cite{carlini2017towards,athalye2018obfuscated,tramer2020adaptive}.
To end such repeated security cat-and-mouse games, there is a line of research focusing on developing \textit{certified defenses} for DNNs under certain adversarial constraints~\cite{cohen2019certified,li2021tss,li2020sok,xu2020automatic,li2019robustra,YangImp2022,li2022dsrs,yang2021trs,yang2022on}.

% Recent efforts fall into two categories,
% those that \textit{empirically} improve
% the robustness of ML models~\cite{ma2018characterizing,yangli2021trs,li2014feature,shafahi2019adversarial}
% and those that aims at 
% \textit{certifying the robustness}~\cite{cohen2019certified,li2019robustra,li2021tss,li2020sok}. These two lines of work complement each other and together they 
% form the state-of-the-art solutions
% of ML robustness. 

Though promising, existing \textit{certified defenses} are restricted to certifying the model robustness within a limited $\ell_p$ norm bounded perturbation radius~\cite{yang2020randomized,cohen2019certified}.
One potential reason for such limitations for existing robust learning approaches is inherent in the fact that most of them have been treating machine learning as a ``pure data-driven" technique that solely depends on a given training set, without interacting with the rich exogenous information such as domain knowledge (e.g., \textit{a stop sign should be of the octagon shape}); while we know human, who has knowledge and inference abilities, is resilient to such attacks.
Indeed, a recent 
seminal work~\cite{gurel2021knowledge} illustrates 
that integrating knowledge rules
can significantly improve 
the \textit{empirical} robustness
of ML models, while leaving
the \textit{certified robustness} 
completely unexplored.

%. To our b
%Thus, in this paper, we aim to explore the questions: \textit{Can we incorporate human knowledge to DNNs to improve their robustness? Can we provide certified robustness for such knowledge-enabled machine learning pipelines?} 

In this paper, we follow this promising 
\texttt{Learning+Reasoning} paradigm
 \cite{gurel2021knowledge} and conduct, to our best knowledge, the first 
study on certified robustness
for it.
Actually, such a \texttt{Learning+Reasoning} paradigm has enabled a diverse range of 
applications~\cite{poon2007joint, 10.1145/3060586, biba2011protein,10.1371/journal.pone.0113523,10.1093/bioinformatics/btv476,xu2019joint,gurel2021knowledge,santurkar2021editing}
including the ECCV'14 best paper~\cite{deng2014large} that encodes 
label relationships as a probabilistic graphical model and improves 
the \textit{empirical} performance of
 deep neural networks
on ImageNet.
In this work, we first provide a concrete \textit{Sensing-reasoning
pipeline} following such paradigm to integrate statistical learning with logical reasoning as illustrated in  Figure~\ref{fig:e2e}. In particular,
the {\em \underline{Sensing Component}} contains
a set of statistical ML models such as deep neural networks (DNNs) that output 
their predictions as a set of Boolean
random variables; and
the {\em \underline{Reasoning Component}}
takes this set of
Boolean random variables as inputs for  
logical inference models such as Markov
logic networks (MLN)~\cite{richardson2006markov} or Bayesian networks (BN)~\cite{pearl2011bayesian} to produce the final output.
We then prove the hardness of certifying the robustness of such a pipeline with MLN for reasoning. Finally, we provide an algorithm to certify the robustness of \sr and we evaluate it on five datasets including both image and text data.

%\begin{figure}[t!]
%\centering
\begin{wrapfigure}{r}{0.5\textwidth}
\vspace{-6mm}
\includegraphics[width=1\linewidth]{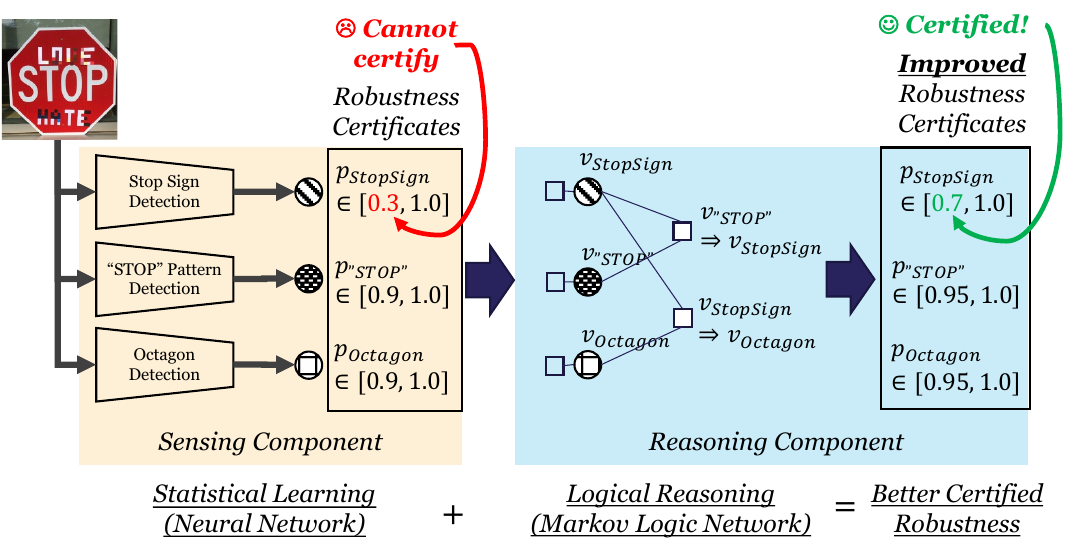}
\vspace{-8mm}
\caption{\small The sensing-reasoning pipeline, 
i.e., a \textit{sensing component} consists of DNNs and a \textit{reasoning component} is constructed 
as MLN. The goal of this paper is to provide certified robustness for such a pipeline, especially the reasoning component.
%The goal of this work is to rigorously reason about
%the robustness of the whole pipeline, given the
%perturbation on the input. As we will see in this paper,
%the end-to-end pipeline might be more
%robust than each individual neural network used for
%sensing, depending on the implemented reasoning logic.
}
\label{fig:e2e}
\vspace{-4mm}
%\end{figure}
\end{wrapfigure}

%This pipeline is general and can be applied to a wide range of applications such as reinforcement learning and information extraction~\cite{xu2019joint,deng2014large,poon2007joint}.

% Given the intensive interests on such knowledge integrated ML pipelines,
% certifying their robustness becomes an increasingly emerging challenge, 
However, certifying the robustness of \sr is challenging,
especially given the inference complexity of the reasoning component.
Our goal is 
to take the first 
step in tackling this challenge.
% , providing a 
% certificate in the following form: {\em During
% inference, given
% a perturbation  $\eta$ on the input $X$
% where $||\eta||_p < C_I$, 
% provide a certificate 
% for the output of the
% sensing-reasoning pipeline.}
In particular, the robustness certification of \sr can be expressed as the confidence interval of the 
marginal probability for the final output
of \textit{reasoning component}.
That is to say, we can use existing state-of-the-art
methods to certify the robustness
of the sensing component that 
contains DNNs or ensembles~\cite{cohen2019certified,salman2019provably,yang2021certified}.
% , and focus on the certification of the reasoning component.
Thus, to provide the end-to-end certification 
for the whole pipeline, 
what is left is to understand
\textit{how to certify the reasoning component},
which is the focus of this work.

Compared with
previous efforts focusing on certified robustness
of neural networks, the reasoning 
component brings its own challenges and
opportunities. Different from a neural network 
whose inference can be executed in polynomial time,
many reasoning models such as MLN can be \textsf{\#P}-complete
for inference. 
However,
as many reasoning models define a probability distribution
in the exponential family, we have more functional structures 
that could potentially make the robustness
optimization (which essentially solves a min-max problem)
easier. {\em In this paper, we provide the first treatment
to this problem characterized by these unique
challenges and opportunities.}

We focus on MLN as the \textit{reasoning component}, and explored three
technical questions, each of which corresponds to a technical contribution of this work.
% \vspace{-1mm}

{\em 1. Is certifying robustness for 
the reasoning component feasible when the inference
of the reasoning component is \textsf{\#P}-hard?}
(Section~\ref{sec:hardness})
Before any concrete algorithm can be proposed, it is important to understand the computational complexity of the
robustness certification.
We first prove that the famous problem of counting in
statistical inference~\cite{VALIANT1979189}
can be reduced to the problem of checking the certified robustness of general reasoning components and MLN. Therefore,
checking certified robustness is no easier than 
counting on the same
family of distribution.
In other words, when the reasoning component is  
a graphical model such as MLN,
checking certified robustness is no easier than
calculating the partition function of the 
underlying graphical model, which is \textsf{\#P}-hard.

{\em 2. Can we efficiently reason about the
certified robustness for the reasoning component 
when given an oracle for statistical inference?}
(Section~\ref{sec:alg})
Given the above hardness result, we focus on certifying 
the robustness given an inference oracle.
% , which has been intensively studied for decades~\cite{XXX,XXX,XXX,XXX}.
However, even when statistical inference can be done by a given oracle~\cite{kuzelka2020complex,huynh2009max}, it is still challenging to certify the robustness of MLN. Our second technical contribution
is to develop such an algorithm for 
MLN as the reasoning component.
We prove that providing certified robustness for MLN is possible because of the structure inherent in 
the probabilistic graphical models and distributions
in the exponential family, which could lead to  monotonicity and convexity properties under certain conditions for solving the certification optimization.

{\em 3. Can a reasoning component improve the 
certified robustness 
compared with the state-of-the-art 
certification methods?}
(Section~\ref{sec:exp})
We test our algorithms
on multiple sensing-reasoning pipelines, in which 
the sensing components
contain the state-of-the-art {\em deep neural networks}.
We construct these pipelines to cover a
range of applications including 
image classification and 
natural language processing tasks.
We show that based on our certification method on the reasoning component, the 
knowledge-enriched sensing-reasoning pipelines
achieves significantly higher certified robustness
than the state-of-the-art 
certification methods for DNNs.

The rest of the paper is organized as follows. We will first introduce the design of the sensing-reasoning pipeline in Section~\ref{sec:sensing-reasoning}, followed by concrete 
illustrations 
% examples 
%development and weight learning  
taking the Markov Logic Networks as an example of the reasoning component in Section~\ref{sec:mln_pipeline}. Next, to certify the robustness of the sensing-reasoning pipeline, especially for the reasoning component, we first prove that certifying the robustness of the reasoning component itself is \textsf{\#P}-complete (Section~\ref{sec:hardness}), and therefore we propose a certification algorithm to upper/lower bound the certification in Section~\ref{sec:certification}, We provide the evaluation of our robustness certification considering different tasks in Section~\ref{sec:exp}.

\vspace{-2mm}
\section{Robust Statistical Learning with Logical Reasoning}
\label{s:method}
\vspace{-0.4em}

\vspace{-2mm}
In this section, we first provide  
a sensing-reasoning pipeline
and then 
formally defined its certified robustness,
and particularly links it to certifying the robustness
for the reasoning component. 

%\begin{figure}
%\vspace{-0.75em}
\begin{wrapfigure}{r}{0.5\textwidth}
\centering
\vspace{-2em}
\includegraphics[width=1\linewidth]{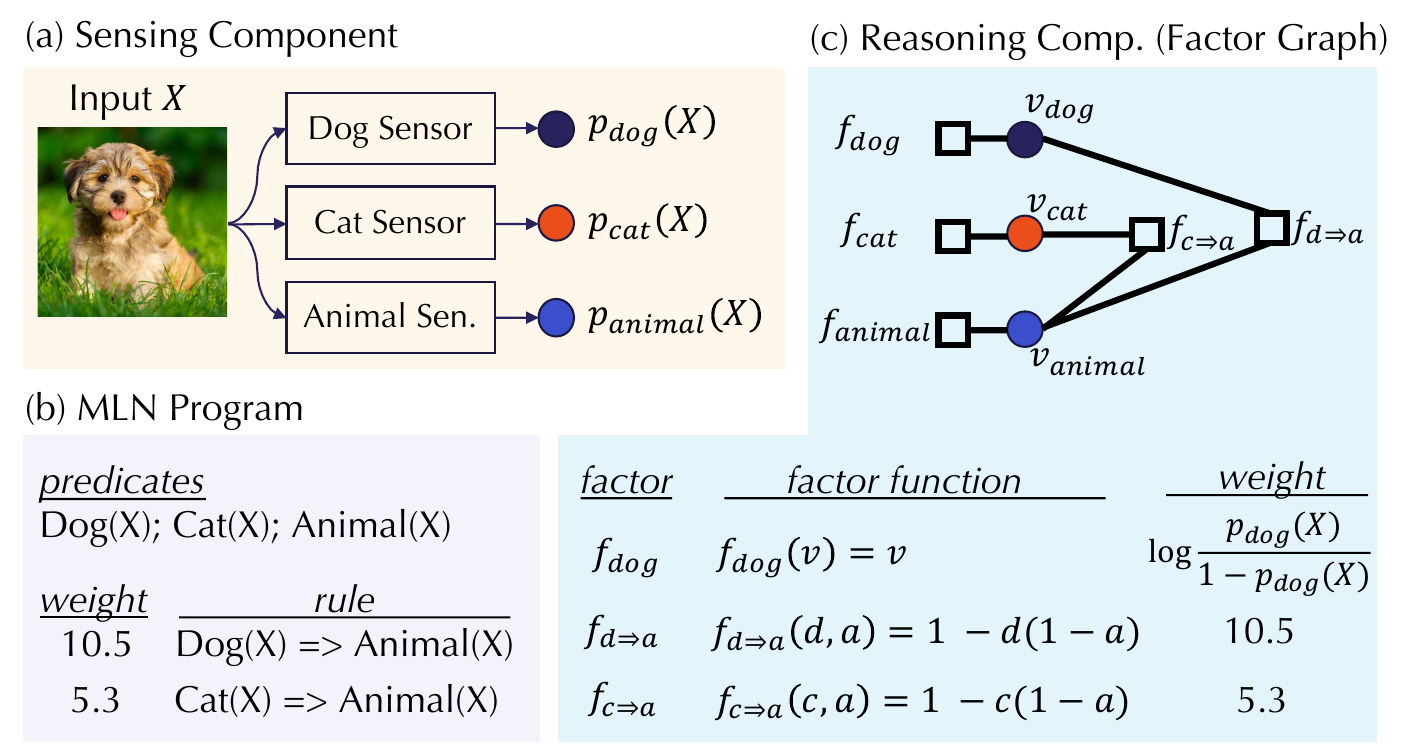}
\vspace{-2em}
\caption{\small A \sr with MLN as the reasoning component.
%, implementing 
%the logic similar to \cite{deng2014large}. 
%In this case, 
%the reasoning component is a grounded MLN represented 
%as a factor graph. There is no interior variable
%(See Section~\ref{sec:mln_pipeline}) in this specific example.
}
\vspace{-1.5em}
\label{fig:mln_pipeline}
%\end{figure}
\end{wrapfigure}

% \vspace{-0.5em}
\vspace{-2mm}
\subsection{Sensing-Reasoning Pipeline}
\label{sec:sensing-reasoning}
% \vspace{-0.5em}

% Sensing-reasoning pipelines have been  used 
% to boost the \textit{accuracy} of pure 
% neural network-based models~\cite{xu2019joint,deng2014large,poon2007joint, 10.1145/3060586}.
\vspace{-2mm}
A \sr contains a set of $n$ sensors
$\{S_i\}_{i \in [n]}$ and a reasoning component $R$.
Each sensor is a binary classifier (for multi-class classifier it corresponds to a group of 
sensors) --- given an input data 
example $X$, each of the sensor $S_i$ outputs 
a probability $p_i(X)$ (i.e., if $S_i$ is 
a neural network, $p_i(X)$ represents 
its output after the final softmax layer).
The reasoning component takes the outputs of
all sensing models as its inputs, and outputs a new Boolean random variable $R(\{p_i(X)\}_{i \in [n]})$.

\vspace{-1mm}
One natural choice of the reasoning component is 
to use a probabilistic graphical model (PGM).
In the following subsection, we will make the reasoning 
component $R$ more concrete by instantiating it as
a Markov logic network (MLN).
The output of a \sr on 
the input data example $X$ is the expectation of the output of reasoning component $R$: $\mathbb{E}[R(\{p_i(X)\}_{i \in [n]})]$.
% \vspace{-3mm}
%\begin{align*}
%\end{align*}
% \vspace{-4mm}

\vspace{-1mm}
{\bf Example.} A \sr provides a generic, principled way
of integrating domain knowledge with the 
output of statistical predictive models such as neural networks.
One such example is \cite{deng2014large} the task of 
ImageNet classification. Here each sensing model
corresponds to the classifier for one specific 
class in ImageNet, e.g., $S_{dog}(X)$ and $S_{animal}(X)$.
% The reasoning component in this specific example outputs
% a Boolean random variable for each specific class.
The reasoning component then encodes domain knowledge 
such that ``{\em If an image is classified as a dog
then it must also be classified as an animal}''
using a PGM. There is no prior work considering the certified robustness of such a knowledge-enabled ML pipeline. 
Figure~\ref{fig:mln_pipeline} illustrates
a concrete \sr,
in which the reasoning component is
implemented as an MLN.
% We also see pipelines of similar
% structures in a range of other applications~
% \cite{poon2007joint,biba2011protein,10.1371/journal.pone.0113523,10.1093/bioinformatics/btv476}.

\vspace{-2mm}
\subsection{Reasoning Component as Markov Logic Networks}
\label{sec:mln_pipeline}
% \vspace{-0.5em}

\vspace{-2mm}
Given the generic definition of a \sr, one can use different models to implement the reasoning components. In this paper,
we focus on Markov logic
networks (MLN), which is a popular way to define a probabilistic graphical model using first-order logic~\cite{RichardsonMarkov}. 
%We discuss other models such as Bayesian networks in the next subsection and  1NN in Appendix.
Concretely, we define the reasoning component
implemented as an MLN, which contains a set of 
weighted first-order logic rules, as illustrated
in Figure~\ref{fig:mln_pipeline}(b). After grounding,
an MLN defines a joint probabilistic 
distribution among a collection of random variables,
as illustrated in Figure~\ref{fig:mln_pipeline}(c).
We adapt the standard MLN semantics to 
a \sr and use a slightly
more general variant compared with the original MLN~\cite{RichardsonMarkov}. 
Each MLN program corresponds to
a factor graph ---
Due to the
space limitation, we will not discuss the
grounding part and point the readers to~\cite{RichardsonMarkov}. We focus on 
defining the result after grounding, i.e., the factor graph. 

Specifically, a grounded MLN is a factor graph
$\mathcal{G} = (\mathcal{V}, \mathcal{F})$, where
$\mathcal{V}$ is a set of Boolean 
random variables. Specific to a \sr, there are two types of random variables
$\mathcal{V} = \mathcal{X} \cup \mathcal{Y}$:
\begin{enumerate}
\vspace{-1mm}
\item {\bf Interface Variables $\mathcal{X} = \{x_i\}_{i \in [n]}$:} Each sensing model 
$S_i$ corresponds to one interface variable 
$x_i$ in the grounded factor graph;
\item {\bf Interior variables $\mathcal{Y} = \{y_i\}_{i \in [m]}$} are other variables introduced 
by the MLN model.
\vspace{-1mm}
\end{enumerate}

Each factor $F \in \mathcal{F}$ contains a
weight $w_{F}$ and a factor function $f_{F}$ defined
over a subset of variables $\bar{\bf v}_{F} \subseteq \mathcal{V}$
that {\em returns $\{0, 1\}$}. There are two sets of factors
$\mathcal{F} = \mathcal{G} \cup \mathcal{H}$:
%\vspace{-3mm}
\begin{enumerate}
\vspace{-1mm}
\item {\bf Interface Factors $\mathcal{G}$:} For each interface variable $x_i$, we create one 
interface factor $G_i$ with weight 
$w_{G_i} = \log [p_i(X) / (1 - p_i(X))]$
and factor function $f_{G_i}(a) = \mathcal{I}[a = 1]$
defined over $\bar{\bf v}_{f_{G_i}} = \{x_i\}$.
\item {\bf Interior Factors $\mathcal{H}$} 
are other factors introduced by the MLN program.
\vspace{-1mm}
%\vspace{-2mm}
\end{enumerate}

\vspace{-1mm}
{\em Remarks: MLN-specific Structure.} Our result applies
to a more general family of factor graphs and are not
necessarily specific to those grounded by MLN. Moreover,
MLN provides an intuitive way of grounding such a 
factor graph with domain knowledge, and factor 
graphs grounded by MLN have certain properties 
that we will use later, e.g., all factors only return 
non-negative values, and there are no unusual 
weight sharing structures.

The above factor graph defines a joint probability
distribution among all variables $\mathcal{V}$. 
We define a {\em possible world} as
a function $\sigma: \mathcal{V} \mapsto \{0, 1\}$ 
that corresponds to one possible assignment of 
values to each random variable.
Let $\Sigma$ denote the set of all 
(exponentially many) possible worlds. 

% The output of a reasoning component
% implemented using MLNs is

The {\em statistical 
inference} process of a reasoning component implemented using MLNs~\cite{RichardsonMarkov} computes the marginal probability
of a given variable $v \in \mathcal{V}$:
\begin{align*}
\vspace{-2mm}
\mathbb{E}[ R_{MLN}(\{p_i(X)\}_{i \in [n]})] =
\Pr{v = 1} 
= Z_1(\{p_i(X)\}_{i \in [n]}) / Z_2(\{p_i(X)\}_{i \in [n]})
\end{align*}
where the partition functions $Z_1$ and $Z_2$
are defined as
\begin{small}
\vspace{-2mm}
\begin{align*}
&Z_1(\{p_i(X)\}_{i \in [n]}) 
= \sum_{\sigma \in \Sigma \wedge \sigma(v) = 1}\exp\left\{ 
\sum_{G_i \in \mathcal{G}} w_{G_i} \sigma(x_i) +
\sum_{H \in \mathcal{H}} w_{H} f_{H}(\sigma(\bar{\bf v}_H))
\right\} \\
&Z_2(\{p_i(X)\}_{i \in [n]})  
= \sum_{\sigma \in \Sigma}\exp\left\{ 
\sum_{G_i \in \mathcal{G}} w_{G_i} \sigma(x_i) +
\sum_{H \in \mathcal{H}} w_{H} f_{H}(\sigma(\bar{\bf v}_H))
\right\}
\vspace{-2mm}
\end{align*}
\end{small}

\vspace{-1mm}
{\bf \underline{Why $w_{G_i} = \log [p_i(X) / (1 - p_i(X))]$?}} When the MLN does not introduce
any interior variables and interior factors, it is easy to see that setting $w_{G_i} = \log [p_i(X) / (1 - p_i(X))]$ ensures that the marginal probability of 
each interface variable equals to the output
of the original sensing model $p_i(X)$. This 
means that if we do not have additional knowledge
in the reasoning component, the pipeline
outputs the \textit{same} distribution as the original
sensing component.

{\bf \underline{Learning Weights
for Interior Factors?}} In this paper, we view all weights for interior factors
as hyperparameters. These weights can 
be learned by maximizing the likelihood 
with weight learning algorithms for MLNs~\cite{10.1007/978-3-540-74976-9_21}.

{\bf \underline{Beyond Marginal Probability for a Single Variable.}} 
We have assumed that the output of a \sr
is the marginal probability distribution of a given random variable
in the grounded factor graph. However, our result can be 
more general --- given a function over possible worlds and outputs
$\{0, 1\}$, the output of a pipeline can be
the marginal probability of such a function. This will not change
the algorithm that we propose later.

\iffalse
\subsection{Reasoning Component as Bayesian Networks}

A Bayesian network (BN) is a probabilistic graphical model that represents a set of variables and their conditional dependencies with a directed acyclic graph.
Let us first consider a Bayesian Network with tree structures, the probability of a random variable being 1 is given by 
\begin{align*}
\Pr{X=1,\{p_i\}}=\sum_{x_1,...x_n} P(1|x_1,...,x_n) \prod_i p_i ^{x_i}(1-p_i)^{1-x_i}.
\end{align*}
In the following sections, we will prove a hardness result of checking robustness in general MLN and BNs and use the above definition to construct an efficient procedure to certify robustness for binary tree BNs.
\fi

\vspace{-0.9em}
\section{Hardness of Certifying Reasoning Robustness}
\vspace{-0.4em}
\label{sec:hardness}

\vspace{-2mm}
{\em Given a reasoning component $R$, how hard is it
to reason about its robustness?} In this section,
we aim at understanding this fundamental question. In order to provide the certified robustness of the reasoning component, which is defined as the lower bound of model predictions for inputs considering an adversarial perturbation with bounded magnitude~\cite{cohen2019certified}, we need to analyze the hardness of this certification problem first.
Specifically, we present the hardness results of 
determining the robustness of the reasoning component defined above, before we can provide our certification algorithm in \Cref{sec:alg}.
We start by defining the counting~\cite{VALIANT1979189} and robustness 
problems on general distribution. We prove that 
counting can be reduced to checking for reasoning robustness, and hence the latter is at least as hard;
% We then use this hardness result to prove 
We then prove the complexities 
of reasoning with MLN.

% \vspace{-0.5em}
\vspace{-2mm}
\subsection{Harness of Certifying General Reasoning Model}
% \vspace{-0.5em}

\vspace{-2mm}
Let $\*X=\{x_1,x_2,\dots,x_n\}$ be a set of variables.
Let $\pi_{\alpha}$ be a distribution over $D^{[n]}$ defined by a set of parameters $\alpha\in P^{[m]}$,
where $D$ is the domain of variables, either discrete or continuous,
and $P$ is the domain of parameters.
We call $\pi$ \emph{accessible} if for any $\sigma\in D^{[n]}$, $\pi_{\alpha}(\sigma)\propto w(\sigma;\alpha)$,
where $w:D^{[n]}\times P^{[m]}\rightarrow  \mathbb R_{\ge 0}  $ is a polynomial-time computable function. 
We will restrict our attention to accessible distributions only.
We use $Q:D^{[n]}\rightarrow\{0,1\}$ to denote a Boolean query, which is a polynomial-time computable function. We define the following two oracles:

\def\Counting{\textnormal{\textsc{Counting}}}

% \vspace{-1mm}
\begin{definition}[\textsc{Counting}]
Given input polynomial-time computable weight function $w(\cdot)$ and query function $Q(\cdot)$, parameters $\alpha$, a real number $\eps>0$, a \textsc{Counting} oracle outputs a real number $Z$ that 
\vspace{-1mm}
\begin{small}
\[
1-\eps\le \frac{Z}{\E[\sigma\sim\pi_{\alpha}]{Q(\sigma)}}\le 1+\eps.
\]
\end{small}
\end{definition}
\vspace{-1mm}

\def\Robustness{\textnormal{\textsc{Robustness}}}

\begin{definition}[\textsc{Robustness}]
Given input polynomial-time computable weight function $w(\cdot)$ and query function $Q(\cdot)$, parameters $\alpha$, two real numbers $\epsilon >0$ and $\delta>0$, a \textsc{Robustness} oracle decides, for any $\alpha'\in P^{[m]}$ such that $\norm{\alpha-\alpha'}_\infty \le \epsilon$, whether the following is true:
% \vspace{-1mm}
\begin{align*}
\begin{small}
  \abs{\E[\sigma\sim\pi_{\alpha}]{Q(\sigma)}-\E[\sigma\sim\pi_{\alpha'}]{Q(\sigma)}}<\delta.
  \end{small}
\end{align*}
\end{definition}
\vspace{-2mm}

\noindent We can prove that \textsc{Robustness} is at least as hard as
\textsc{Counting} by a reduction argument.

\begin{theorem}[$\Counting\reduceT\Robustness$]\label{hardness_from_reduction}
Given polynomial-time computable weight function $w(\cdot)$ and query function $Q(\cdot)$, parameters $\alpha$ and real number $\eps>0$, the instance of \textsc{Counting}, $(w,Q,\alpha,\eps)$ can be determined by up to $O(1/\varepsilon_c^2)$ queries of the \textsc{Robustness} oracle with input perturbation $\epsilon=O(\varepsilon_c)$.
\label{theorem:counting-to-robustness}
\end{theorem}

\vspace{-2mm}
\textit{Proof-sketch.} We define the partition function $Z_i:=\sum_{\sigma:Q(\sigma)=i} w(\sigma;\alpha)$ and $\E[\sigma\sim\pi_{\alpha}]{Q(\sigma)} = Z_1 / (Z_0 + Z_1)$.
 We then construct a new weight function $t(\sigma;\alpha):=w(\sigma; \alpha)\exp(\beta Q(\sigma))$ by introducing an additional parameter $\beta$, such that $\tau_{\beta}(\sigma) \propto t(\sigma;\beta)$, and $\E[\sigma\sim\tau_{\beta}]{Q(\sigma)} = \frac{e^\beta Z_1}{Z_0+e^\beta Z_1}$. Then we consider the perturbation $\beta'=\beta\pm \epsilon$, with $\epsilon=O(\varepsilon_c)$  and query the \textsc{Robustness} oracle with input $(t,Q,\beta,\epsilon ,\delta)$ multiple times to perform a binary search in $\delta$ to estimate $\abs{\E[\sigma\sim\pi_{\beta}]{Q(\sigma)}-\E[\sigma\sim\pi_{\beta'}]{Q(\sigma)}}$. Perform a further ``outer" binary search to find the $\beta$ which maximizes the perturbation. This yields a good estimator for $\log{\frac{Z_0}{Z_1}}$ which in turn gives $\E[\sigma\sim\pi_{\alpha}]{Q(\sigma)}$ with $\varepsilon_c$ multiplicative error. 
We leave detailed proof to~\Cref{apx:hardness}.

%Here $\Phi$ refers to the standard Gaussian CDF and $\Phi^{-1}$ as its inverse.
%From Corollary~\ref{col:smooth}, the output confidence of the sensing component can be bounded based on the input perturbation. Next we will discuss how to bound the reasoning component output given the certified confidence interval of sensing models.

% \vspace{-0.5em}

% \vspace{-0.5em}
\vspace{-2mm}
\subsection{Hardness of Certifying Markov Logic Networks}
\label{sec:hard-mln}
% \vspace{-0.5em}

\vspace{-2mm}
Given Theorem \ref{theorem:counting-to-robustness}, we can now
state the following result specifically for MLNs:

\begin{theorem}[MLN Hardness]
Given an MLN whose grounded factor graph is $\mathcal{G} = (\mathcal{V}, \mathcal{F})$  in which the weights for interface factors are $w_{G_i} = \log p_i(X)/(1-p_i(X))$ and constant thresholds $\delta, \{C_i\}_{i\in[n]}$, deciding whether
\vspace{-4mm}

\begin{small}
\begin{align*}
&\forall \{\epsilon_i\}_{i\in [n]}~~~~
   (\forall i.~|\epsilon_i| < C_i) \implies 
\left|\mathbb{E} R_{MLN}(\{p_i(X)\}_{i \in [n]})
     - \mathbb{E} R_{MLN}(\{p_i(X) + \epsilon_i\}_{i \in [n]})\right| < \delta
\end{align*}
\end{small}
is as hard as estimating 
$\mathbb{E} R_{MLN}(\{p_i(X)\}_{i \in [n]})$ up to $\varepsilon_c$ multiplicative error, with $\epsilon_i = O(\varepsilon_c)$.

\end{theorem}
\vspace{-5mm}
\begin{proof}
	Let $\alpha=[p_i(X)]$, query function $Q(.)= R_{MLN}(.)$ and $\pi_\alpha$ defined by the marginal distribution over interior variables of MLN. Theorem \ref{theorem:counting-to-robustness} directly implies that $O(1/\varepsilon_c^2)$ queries of a \textsc{Robustness} oracle can be used to efficiently estimate $\mathbb{E} R_{MLN}(\{p_i(X)\}_{i \in [n]})$. 
\end{proof}
	
\vspace{-3mm}
\noindent In general, statistical inference in
MLNs is \textsf{\#P}-complete, and checking
robustness for general MLNs is also 
\textsf{\#P}-hard.

\iffalse
\subsection{Bayesian Networks}
Analogously with the above reasoning, we can also state the general hardness result for deciding the robustness of BNs:

\begin{theorem}[BN hardness]\label{BN hardness}
Given a Bayesian network with a set of parameters $\{p_i\}$, a set of perturbation parameters $\{\epsilon_i \}$ and threshold $\delta$, deciding whether
$$|\Pr{X=1; \{p_i\}}-\Pr{X=1; \{p_i+\epsilon_i \}}|<\delta$$ is at least as hard as estimating $\Pr{X=1; \{p_i\}}$ up to $\varepsilon_c$ multiplicative error, with $\epsilon_i = O(\varepsilon_c)$.
\end{theorem}
\begin{proof}
Let $\alpha=[p_i]$, $Q(\sigma)=X$ and $\pi_\alpha$ defined by the the probability distribution of a target random variable. Since $X\in \{0,1\}$, we have $\E[\sigma\sim\pi_{\alpha}]{Q(\sigma)}=\Pr{X=1; \{p_i\}}$. The proof then follows analogously from Theorem \ref{theorem:counting-to-robustness}.
\end{proof}

Based on the hardness analysis of the reasoning robustness, 
we can see that it is challenging to directly certify the robustness of the reasoning component.
However, just as we can approximately certify the robustness of single ML models~\cite{li2020sok}, in the next section, we will present and discuss how to approximately certify the robustness of the reasoning component, and we show that for some structures such as BN trees, the certification could even be tight.

%\zhikuan{Like the case for MLNs above, as exact inference in BNs is \textsf{\#P}-complete, checking robustness for general BNs is also \textsf{\#P}-hard. }

\fi

\vspace{-1em}
\section{Certifying the Robustness of Sensing-Reasoning Pipeline}
% \vspace{-0.5em}
\vspace{-0.5em}
\label{sec:certification}

Given a \sr with $n$ sensors $\{S_i\}_{i \in [n]}$ and a reasoning component $R$, we will first formally define its end-to-end certified robustness and then its connection to the robustness of each component. In particular, based on the above hardness result for \textit{certifying the robustness of the reasoning component} in Section~\ref{sec:hardness}, we will provide an effective certification method to upper/lower bound the certification, taking \textit{any} oracle for the inference of the reasoning component into account. With the certification of the reasoning component, we will finally provide the robustness certification for the sensing-reasoning pipeline by combining the certification of sensing and reasoning components.

\vspace{-1mm}
\begin{definition}[$(C_I, C_E, p)$-robustness]
A \sr with $n$ sensors $\{S_i\}_{i \in [n]}$
and a reasoning component $R$ 
is $(C_I, C_E, p)$-robust on the input $X$,
if for input perturbation $\eta, ||\eta||_p \leq C_I$ 
\vspace{-4mm}
\begin{align*}
\small
\left|\mathbb{E}[R(\{p_i(X)\}_{i \in [n]})] - \mathbb{E}[R(\{p_i(X + \eta)\}_{i \in [n]})] \right| \leq C_E.
\end{align*}
\end{definition}

\vspace{-2mm}
\noindent I.e., a perturbation 
$||\eta||_p < C_I$ on the input only changes
the final pipeline output by at most $C_E$.

\vspace{-1mm}
{\bf Sensing Robustness and Reasoning Robustness.} We
decompose the end-to-end certified robustness of the pipeline into two components.
The first component, which we call the \textit{sensing robustness}, has been 
studied by the research community recently~\cite{kolter2017provable,tjeng2017evaluating,cohen2019certified} ---
given a perturbation $||\eta||_p < C_I$ on the input $X$,
we say each sensor $S_i$ is $(C_I, C_S^{(i)}, p)$-robust 
if
\vspace{-1.8mm}
\begin{align*}
\small
\forall \eta, ||\eta||_p \leq C_I \implies 
|p_i(X) - p_i(X+\eta)| \leq C_S^{(i)}
\end{align*}
\vspace{-1mm}
The robustness of the \textit{reasoning component} R is defined as: Given a perturbation $|\epsilon_i| < C_S^{(i)}$ 
on the output of each sensor $S_i(X)$, we say 
the reasoning component $R$ is $\left(\{C_S^{(i)}\}_{i\in[n]}, C_E\right)$-robust 
if 
% \begin{small}
% \[
\begin{align*}
\vspace{-3mm}
\forall \epsilon_1, ..., \epsilon_n, 
(\forall i.~|\epsilon_i| \leq C_S^{(i)}) \implies  
\left| \mathbb{E}[R(\{p_i(X)\}_{i \in [n]}) ] 
    - \mathbb{E}[R(\{p_i(X)+\epsilon_i\}_{i \in [n]})] \right| \leq C_E.
\end{align*}
% \]
% \end{small}
It is easy to see that when the sensing component is
$\left(C_I, \{C_S^{(i)}\}_{i\in[n]}, p\right)$-robust and the reasoning component
is $\left(\{C_S^{(i)}\}_{i\in[n]}, C_E\right)$-robust on $X$, the \sr is
$(C_I, C_E, p)$-robust. Since the sensing robustness has
been intensively studied by previous work, in this paper, 
we mainly focus on the reasoning robustness and therefore analyze the robustness of the pipeline.

% \vspace{-0.5em}
\vspace{-2mm}
\subsection{Certifying Sensing Robustness}
% \vspace{-0.5em}

\vspace{-2mm}
There are several existing ways to certify the robustness of sensing models, such as Interval Bound Propagation (IBP)~\cite{gowal2018effectiveness}, Randomized Smoothing~\cite{cohen2019certified}, and others~\cite{zhang2018efficient,weng2018towards}. Here we will leverage randomized smoothing to provide an example for certifying the robustness of sensing components.
\begin{corollary}
\label{col:smooth}
Given a sensing model $S_i$,
we construct a \textit{smoothed 
sensing model} 
$g_i(X; \hat{\sigma}) = 
\mathbb{E}_{\xi \sim \mathcal{N}(0, \hat{\sigma}^2)} p_i(X + \xi)$.
With input perturbation $||\eta||_2 \leq C_I$, the smoothed sensing 
model satisfies
\begin{align*}
    \Phi(\Phi^{-1}(g_i(X; \hat{\sigma}))-C_I/\hat{\sigma}) & \leq g_i(X + \eta; \hat{\sigma}) 
    \leq \Phi(\Phi^{-1}(g_i(X; \hat{\sigma}))+C_I/\hat{\sigma})
    \vspace{-2mm}
\end{align*}
where $\Phi$ is the 
Gaussian CDF and $\Phi^{-1}$ as its inverse.

\vspace{-3mm}
\end{corollary}
Thus, the output probability of smoothed sensing model can be bounded given input perturbations.
% Given a set of 
% $n$ sensors, we construct $n$
% smoothed sensing models and use them instead.
Note that the specific 
ways of certifying sensing robustness
is orthogonal to certifying 
reasoning robustness, and one can  plug in different sensing certification strategies.
% \vspace{-0.5em}

\vspace{-2mm}
\subsection{Certifying Reasoning Robustness}
\label{sec:alg}
\vspace{-1mm}
% In this section, we will analyze and provide approximation algorithms to certify the reasoning component based on its specific structures.
%\subsection{Markov Logic Networks}

% Although statistical inference for general MLNs
% are \textsf{\#P}-complete, in practice, there
% are multiple ways to accommodate this problem:
% (1) when the factor graph is of small tree-width,
% one can apply exact inference algorithms~\cite{MAL-001};
% (2) one can construct the MLN program such that exact
% statistical inference is feasible;
% and (3) there exists a range of approximate 
% inference algorithms for MLN inference \cite{10.5555/1620163.1620242}.

% \vspace{-1mm}
Given the hardness results for certifying reasoning robustness in \Cref{sec:hard-mln},
in this paper, we assume that we have access
to an oracle for statistical inference,
and provide a novel algorithm
to certify the reasoning robustness.
I.e., we assume that we are able to
calculate the two partition functions 
$Z_1(\{p_i(X)\}_{i \in [n]})$ and 
$Z_2(\{p_i(X)\}_{i \in [n]})$.

\begin{lemma}[MLN Robustness]\label{MLN Robustness}
Given access to partition functions $Z_1(\{p_i(X)\}_{i \in [n]})$ and 
$Z_2(\{p_i(X)\}_{i \in [n]})$, and maximum perturbations $\{C_i\}_{i\in[n]}$, 
$\forall \epsilon_1, ..., \epsilon_n$, if $\forall i.~|\epsilon_i| < C_i$,
we have that $\forall \lambda_1, ..., \lambda_n \in \mathbb{R}$,
\begin{small}
\begin{align*}
&\max_{\{|\epsilon_i| < C_i\}} \ln \mathbb{E}[R_{MLN}(\{p_i(X) + \epsilon_i\}_{i \in [n]})] 
\leq \max_{\{|\epsilon_i|<C_i\}} \widetilde{Z_1}(\{\epsilon_i\}_{i \in [n]}) 
-\min_{\{|\epsilon_i'|<C_i\}}
\widetilde{Z_2}(\{\epsilon_i'\}_{i \in [n]}) 
\end{align*}
\begin{align*}
&\min_{\{|\epsilon_i| < C_i\}} \ln \mathbb{E}[R_{MLN}(\{p_i(X) + \epsilon_i\}_{i \in [n]})]  
\geq \min_{\{|\epsilon_i|<C_i\}}
\widetilde{Z_1}(\{\epsilon_i\}_{i \in [n]}) %\\
 -
\max_{\{|\epsilon_i'|<C_i\}}
\widetilde{Z_2}(\{\epsilon_i'\}_{i \in [n]}) 
\vspace{-2mm}
\end{align*}
\end{small}
where
\vspace{-2mm}
\begin{small}
\[
\widetilde{Z_r}(\{\epsilon_i\}_{i \in [n]}) = \ln Z_r(\{p_i(X)+\epsilon_i\}_{i \in [n]}) + \sum_i \lambda_i \epsilon_i.
\]
\end{small}
\end{lemma}

%\vspace{-3mm}
We leave the proof to the~\Cref{apx:robustness_mln}. The high-level proof idea is to decouple $Z_1/Z_2$
into two sub-problems via 
a collection of Lagrangian multipliers, i.e., $\{\lambda_i\}$.
For any assignment of 
$\{\lambda_i\}$,
we obtain a valid upper/lower
bound, which reduces the 
certification process to 
the process
of \textit{searching} for
an assignment 
of these multipliers that 
minimize the upper bound (maximize the
lower bound).

\vspace{-1mm}
To efficiently search 
for the optimal assignment of 
$\{\lambda_i\}$, it is crucial to consider the interactions between
these $\{\lambda_i\}$ and 
the corresponding solution
of $\widetilde{Z_r}$,
which hinges on the structure of MLN. In particular, we can prove the following (Detailed proofs and discussions in \Cref{apx:sup_alg1}):

\begin{proposition}[Monotonicity]
When $\lambda_i \geq 0$,
$\widetilde{Z_r}(\{\epsilon_i\}_{i \in [n]})$
monotonically increases w.r.t. $\epsilon_i$; 
When $\lambda_i \leq -1$, $\widetilde{Z_r}(\{\epsilon_i\}_{i \in [n]})$
monotonically decreases w.r.t. $\epsilon_i$.
\end{proposition}

\begin{proposition}[Convexity]
$\widetilde{Z_r}(\{\tilde{\epsilon}_i\}_{i \in [n]})$ is a convex function in $\tilde{\epsilon}_i, \forall i$ with
\begin{small}
$$\tilde{\epsilon}_i=\log\left[\frac{(1-p_i(X))(p_i(X)+\epsilon_i)}{p_i(X)(1-p_i(X)-\epsilon_i)}\right].$$
\end{small}
%\vspace{-0.5em}
\end{proposition}
\textit{Implication.} Given the monotonicity region, the maximal and minimal of $\widetilde{Z_r}$ are achieved at either $\epsilon_i = -C_i$ or $\epsilon_i = C_i$ respectively.
Given the convexity region,  the maximal is achieved at $\epsilon_i \in \{-C_i, C_i\}$, and the minimal is achieved at $\epsilon_i\in\{-C_i, C_i\}$ or at the zero gradient of $\widetilde{Z_r}(\{\tilde{\epsilon}_i\}_{i \in [n]})$.
% \begin{enumerate}
% %\vspace{-4mm}
% \item When $\lambda_i \geq 0$,
% $\widetilde{Z_r}(\{\epsilon_i\}_{i \in [n]})$
% monotonically increases w.r.t. $\epsilon_i$; thus,
% the maximal of $\widetilde{Z_r}(\{\epsilon_i\}_{i \in [n]})$ is achieved at $\epsilon_i = C_i$ and the minimal is achieved at $\epsilon_i = -C_i$.
% When $\lambda_i \leq -1$, $\widetilde{Z_r}(\{\epsilon_i\}_{i \in [n]})$
% monotonically decreases w.r.t. $\epsilon_i$; thus, the maximal is achieved at $\epsilon_i = -C_i$ and the minimal is achieved at $\epsilon_i = C_i$.
% %\vspace{-5mm}
% \item When $\lambda_i \in (-1, 0)$, 
% $\widetilde{Z_r}(\{\tilde{\epsilon}_i\}_{i \in [n]})$ is convex in $\tilde{\epsilon}_i$; thus, the maximal is achieved at $\epsilon_i \in \{-C_i, C_i\}$, and the minimal is achieved at $\epsilon_i\in\{-C_i, C_i\}$ or at the zero gradient of $\widetilde{Z_r}(\{\tilde{\epsilon}_i\}_{i \in [n]})$.
% %\vspace{-4mm}
% \end{enumerate}\bo{use proposition here}
As a result, our analysis leads to the following certification algorithm.
%}
% \noindent As a result, given any $\{\lambda_i\}_{i \in [n]}$,
% we can calculate an upper bound for $\overline{R_{MLN}}$.
% Similarly, we are able to lower bound
% $\underline{R_{MLN}}$ using the same technique.

\iffalse
We leave the proof to the Appendix, %\zz{
where 
we also present further analysis of the structure of $\widetilde{Z_r}$:
\begin{enumerate}
\vspace{-2mm}
\item When $\lambda_i \geq 0$,
$\widetilde{Z_r}(\{\epsilon_i\}_{i \in [n]})$
monotonically increases w.r.t. $\epsilon_i$; Thus,
the maximal is achieved at $\epsilon_i = C_i$ and the minimal is achieved at $\epsilon_i = -C_i$.
When $\lambda_i \leq -1$, $\widetilde{Z_r}(\{\epsilon_i\}_{i \in [n]})$
monotonically decreases w.r.t. $\epsilon_i$; Thus, the maximal is achieved at $\epsilon_i = -C_i$ and the minimal is achieved at $\epsilon_i = C_i$.
\item When $\lambda_i \in (-1, 0)$, the maximal is achieved at $\epsilon_i \in \{-C_i, C_i\}$, and the minimal is achieved at $\epsilon_i\in\{-C_i, C_i\}$ or at the zero gradient of $\widetilde{Z_r}(\{\tilde{\epsilon}_i\}_{i \in [n]})$ with respect to $\tilde{\epsilon}_i=\log\left[\frac{(1-p_i(X))(p_i(X)+\epsilon_i)}{p_i(X)(1-p_i(X)-\epsilon_i)}\right],$ due to the convexity of $\widetilde{Z_r}(\{\tilde{\epsilon}_i\}_{i \in [n]})$ in $\tilde{\epsilon}_i, \forall i$.
\end{enumerate}
%}

\bo{explain}

\noindent As a result, given any $\{\lambda_i\}_{i \in [n]}$,
we can calculate an upper bound for $\overline{R_{MLN}}$.
Similarly, we are able to lower bound
$\underline{R_{MLN}}$ using the same technique.
\fi

\begin{wrapfigure}{r}{0.55\textwidth}
\vspace{-2.6em}
\scriptsize
\begin{minipage}{0.55\textwidth}
\begin{algorithm}[H]
\caption{Algorithms for  
MLN robustness upper bound (algorithm of lower bound is
similar)}
\scriptsize
\label{alg:MLNRobustness}
\begin{algorithmic}[1]
\begin{small}
\vspace{-0.2em}
\INPUT: Oracles calculating $\widetilde{Z_1}$
and $\widetilde{Z_2}$; maximal perturbations $\{C_i\}_{i\in [n]}$.
\OUTPUT: An upper bound for input $R_{MLN}(\{p_i(X) + \epsilon_i\})$ \\
\STATE $\overline{R}_{min} \leftarrow 1$ 
\STATE initialize $\lambda$
\FOR{$b \in \text{search budgets}$}  
\STATE $\lambda \rightarrow \texttt{update}(\{\lambda\}; \lambda_i \in (-\infty, -1] \cup [0, +\infty))$
\FOR{$i=1$ {\bfseries to} $n$}
    \IF{$\lambda_i\ge 0$}
    \STATE $\epsilon_i=C_i$, $\epsilon'_i=-C_i$
    \ELSIF{$\lambda_i\le -1$}
    \STATE $\epsilon_i=-C_i$, $\epsilon'_i=C_i$
    \ENDIF
    \STATE $\overline{R} \leftarrow \widetilde{Z_1}(\{\epsilon_i\}_{i \in [n]}) -
    \widetilde{Z_2}(\{\epsilon_i'\}_{i \in [n]})$
    \STATE $\overline{R}_{min} \leftarrow \min(\overline{R}_{min}, \overline{R})$
\ENDFOR
\ENDFOR
\STATE {\bf return} $\overline{R}_{min}$
\end{small}
\end{algorithmic}
\end{algorithm}
\end{minipage}
%\vspace{-6em}
\end{wrapfigure}

\textbf{Algorithm of Certifying Reasoning Robustness.} Algorithm~\ref{alg:MLNRobustness}
illustrates the detailed algorithm based 
on the above result to upper bound the robustness of MLN. The main step is
to explore different regimes of the 
$\{\lambda_i\}$.
In this paper, we only explore regimes
where $\lambda \in (-\infty, -1] \cup [0, +\infty)$
as this already provides reasonable solutions
in our experiments.
The function
$\texttt{update}(\{\lambda_i\})$
defines the exploration strategy ---
Depending on the scale of the problem, one
can explore $\{\lambda_i\}$ using
grid search, random sampling, or even gradient-based
methods. For experiments in this paper,
we use either grid search or random sampling.
It is an exciting future direction to
understand other efficient exploration and search
strategies.
We leave the detailed explanation
of the algorithm to~\Cref{{apx:sup_alg1}}.

\section{Experiments}
\label{sec:exp}
%\vspace{-3mm}

We conduct intensive experiments on five datasets to evaluate the certified robustness of the sensing-reasoning pipeline. We focus on two tasks
with different modalities: 
\textit{image classification} task on Road Sign dataset created based on GTSRB dataset~\cite{stallkamp2012man}
following the standard setting as~\cite{gurel2021knowledge};
and \textit{information extraction} task with stocks news on text data. 
We also report additional results 
on two other image classification tasks (Word50~\cite{chen2015learning} and PrimateNet, which is a subset of ImageNet ILSVRC2012~\cite{imagenet_cvpr09}) with natural knowledge rules
in~\Cref{adx:word50-section} and~\Cref{adx:primate-section}. We also report results on standard image benchmarks (MNIST and CIFAR10) with manually constructed knowledge rules in~\Cref{adx:constructed-knowledge}. The code is provided at~\url{https://github.com/Sensing-Reasoning/Sensing-Reasoning-Pipeline}.

%To certify the End-to-end robustness of \sr, we conduct a set of experiments on different datasets and ML tasks including an \textit{image classification} task on Road Sign dataset created based on GTSRB dataset~\cite{stallkamp2012man}, PrimateNet dataset based on ImageNet dataset~\cite{imagenet_cvpr09}, Word10 dataset based on Word50 dataset~\cite{chen2015learning}, and an \textit{information extraction} task with stocks news. 

%We mainly demonstrate \textbf{our findings} as follows: (1) With domain knowledge the \sr can even improve the benign performance of a single ML sensing model slightly without considering input perturbation. (2) Considering different ratios of sensing models to be attacked (including the scenario where all sensing models are attacked), the certified robustness of \sr is always higher than that of a single ML sensing model without domain knowledge.

%We note that the certified robustness of \sr also depends on the quality of the domain knowledge we use, and how to extract and construct ``robust'' knowledge would be an interesting future direction.

%(3) Under different input perturbation, our certified robustness bound is relatively tight in practice. Here we will mainly focus on the example of MLN as our reasoning component given that the BN trees could already theoretically provide \textit{tight} certified robustness bound. We evaluate the two image classification and information extraction tasks as below.

%\vspace{-2mm}
\subsection{Experimental Setup}
%\vspace{-1em}

\noindent\textbf{Datasets and Tasks.} For the \textit{road sign classification} task, we follow~\cite{gurel2021knowledge} and use the same 
dataset GTSRB~\cite{stallkamp2012man}, which contains 12 types of German road signs \{"Stop'', "Priority Road'', "Yield'', "Construction Area'', "Keep Right'', "Turn Left'', "Do not Enter'', "No Vihicles'', "Speed Limit 20'', "Speed Limit 50'', "Speed Limit 120'', "End of Previous Limitation''\}.
It consists of 14880 training samples, 972 validation samples, and 3888 testing samples. We also include 13 
additional detectors for knowledge 
integration, detecting attributes such as
whether the border has an octagon shape
(See \Cref{adx:road-sign-section} for a full list).

%\vspace{-1mm}
For the \textit{information extraction} task,  we 
use the HighTech dataset which consists of both daily closing asset price and financial news from \textsl{2006} to \textsl{2013}~\cite{ding2014using}. We choose 9 companies with the most news, resulting in 4810 articles related to 9 stocks filtered by company name. We split the dataset into training and testing days chronologically. 
We define three information extraction tasks as our sensing models: \texttt{StockPrice(Day, Company, Price)}, \texttt{StockPriceChange(Day, Company, Percent)}, \texttt{StockPriceGain(Day, Company)}.
% We leave detailed definitions of
% these relations to Appendix.
The domain knowledge that we integrate depicts
 the relationships between these relations (See \Cref{apx:stock_news} for more details). 

%\vspace{-1mm}
\noindent\textbf{Knowledge Rules.} We 
integrate different types of knowledge 
rules for these two applications.
We provide the full list of knowledge rules
in the \Cref{adx:road-sign-section}.%  \jiawei{it's not in the appendix?}

%\vspace{-1mm}
For \textit{road sign classification}, we
follow~\cite{gurel2021knowledge}, which includes 
two different types of knowledge rules ---
{\em Indication rules} (road sign class $u$ indicates attribute $v$)
and {\em Exclusion rules} (attribute classes $u$ and $v$ with the same general type such as "Shape'', "Color'', "Digit'' or "Content'' are naturally exclusive).

%\vspace{-1mm}
For \textit{information extraction},
we integrate knowledge about 
the relationships between the sensing models
(e.g., \texttt{StockPrice}, \texttt{StockPriceChange},
\texttt{StockPriceGain}).
For example,
the stock prices
of two consecutive days, 
$\texttt{StockPrice}(d_1, Company, p_1)$ and
$\texttt{StockPrice}(d_2, Company, p_2)$,
should be 
consistent with 
$\texttt{StockPriceChange}(d_2, Company, p)$,
i.e., $p = (p_2 - p_1)/p_1$.

%which contains information about how stock 
%price changes on $d_2$.

\noindent\textbf{Implementation Details.}
Throughout the road sign classification experiment, we implement all
sensing models using the GTSRB-CNN~\cite{eykholt2018robust} architecture. During training, we train all sensors with Isotropic Gaussian~$\epsilon \sim \mathcal{N}(0, \hat{\sigma}^2I_d)$ augmented data with 50000 training iterations until converge and tune the training parameters on the validation set, following \cite{cohen2019certified}. We use the SGD-momentum with the initial learning rate as $0.01$ and the weight decay parameter as $10^{-4}$ to train all the sensors for 50000 iterations with $200$ as the batch size, following \cite{gurel2021knowledge}. During certification, we adopt the same smoothing parameter for training to construct the smoothed model based on Monte-Carlo sampling.

%\noindent\textbf{Information Extraction.} 
%\vspace{-1mm}
For information extraction, we use BERT as our model architecture. During training, we use the final hidden state of the first token [CLS] from BERT as the representation of the whole input and apply dropout with probability $p=0.5$ on this final hidden state. Additionally, there is a fully connected layer added on top of BERT for classification. To fine-tune the BERT classifiers for three information tasks, we use the Adam optimizer with the initial learning rate as $10^{-5}$ and the weight decay parameter as $10^{-4}$. We train all the sensors for 30 epochs, and the batch size $32$. 

%we adopt the GTSRB-CNN~\cite{eykholt2018robust} as our main task model's architecture. For each attribute sensor, we still use the GTSRB-CNN architecture but slightly modify the output dimension to be 2 as a binary classifier. We use the raw training set for the main task model's training and construct the corresponding binary classification training set for each attribute sensor with the rule \emph{"If this image contains such attribute?''}. We apply isotropic Gaussian augmentation~$\epsilon \sim \mathcal{N}(0, \sigma^2I)$ during the training, and smooth our models by adopting the randomized smoothing scheme~\cite{xxx} to obtain the certified robustness guarantee of the sensing model's output confidence under the $\ell_2$-norm bounded perturbation.

\noindent\textbf{Evaluation Metrics.} We adopt the standard \emph{certified accuracy} as our evaluation metric, defined by the percentage of instances that can be certified under certain $\ell_p$-norm bounded perturbations. Specifically, given the input $x$ with ground-truth label $y$, once we can certify the bound of the model's output confidence on predicting label $y$ under the norm-bounded perturbation as $[\mathcal{L}, \mathcal{U}]$, the certified accuracy can be defined by:
$ \frac{1}{N}\sum_{i=1}^{N}\mathbbm{I}([\mathcal{L}_i > 0.5])
$ where $\mathbbm{I(\cdot)}$ denotes the indicator function. 
Since each sensing component's certification is performed by randomized smoothing, which yields the failure probability characterized by $\zeta_0$, we will control the failure probability $\zeta$ for the whole \sr pipeline with $n$ sensing models as $\zeta_0 = 1 - (1 - \zeta)^{1/n}$ by applying the union bound. Throughout all the experiments, $\zeta$ is kept to $0.001$ so our end-to-end certification is guaranteed to be correct with at least $99.9\%$ confidence. 
%Thus, when there are $n$ interface variables, the corresponding failure probability will be set to $1-(1-0.001)^{1/n}$ via union bound.}

%\vspace{-4mm}
\begin{table}[t!]
\centering
\caption{{\small{ \bf (Road sign classification)} \textit{Certified accuracy} under different input perturbation magnitudes ($C_I$). Models are smoothed with different Gaussian noises $\epsilon \sim \mathcal{N}(0, \hat{\sigma}^2I_d), \hat{\sigma} \in \{0.12, 0.25, 0.50\}$. Rows with $\ast$ denote the best certified accuracy among all the smoothing parameters for each method. The bold numbers show the higher certified accuracy under the same $(C_I, \hat{\sigma})$ setting and the numbers with underline show the highest certified accuracy for each $C_I$ among different smoothing parameters. (All certificates hold with $p=99.9\%$)} }
\vspace{3mm}
\scalebox{0.9}{
\begin{tabular}{c|c|c|c|c|c}
\toprule
\textbf{Methods}                         & $\hat{\sigma}$ & $C_I = 0.12$        & $C_I = 0.25$        & $C_I = 0.50$        & $C_I = 1.00$        \\ \hline
       & 0.12     & 90.8                & 87.1                & 0.0                 & 0.0                 \\
         Vanilla Smoothing                                   & 0.25     & 89.6                & 88.4                & 71.6                & 0.0                 \\
         (w/o knowledge) & 0.50     & 84.0                & 80.2                & 73.2                & 61.7                \\ %\cline{2-6} 
                                            & $\ast$     & 90.8                & 88.4                & 73.2                & 61.7                \\ \hline
 & 0.12     & {\ul \textbf{96.0}} & \textbf{89.0}       & \textbf{73.2}       & \textbf{24.2}       \\
        Sensing-Reasoning Pipeline & 0.25     & \textbf{93.4}       & {\ul \textbf{91.0}} & \textbf{74.0}       & \textbf{49.2}       \\
(w/ knowledge) & 0.50     & \textbf{89.3}       & \textbf{85.4}       & {\ul \textbf{75.5}} & {\ul \textbf{62.5}} \\ %\cline{2-6} 
                                            & $\ast$     & \textbf{96.0}       & \textbf{91.0}       & \textbf{75.5}       & \textbf{62.5}       \\ \bottomrule
\end{tabular}}
\vspace{-4mm}
\label{tab:road-sign-exp}
%\vspace{-2mm}
\end{table}

%\vspace{-1mm}
\subsection{Results of \textit{Road Sign Classification}}
%\vspace{-1em}
In this section, we evaluate the certified robustness of our \sr under the $\ell_2$-norm bounded perturbation. We first report the $\ell_2$ certified accuracy of our \sr and compare it to
a strong baseline as a vanilla randomized smoothing trained model (without knowledge). Note that it is flexible to replace the sensing component with other robust training algorithms.
We conduct our evaluation under different smoothing parameters $\hat{\sigma} = \{0.12, 0.25, 0.50\}$ and various $\ell_2$ perturbation magnitudes on the input image $C_I = \{0.12, 0.25, 0.50, 1.00\}$ (Table~\ref{tab:road-sign-exp}).  
% Since there are 25 interface variables, the failure probability to each sensing component, named as $\alpha_0$, is set to $1-(1-0.001)^{1/25} = 4.002\times 10^{-5}$ here. 
During certification, we evaluate our certification time per sample with 25 sensors as 5.39s, which shows that the overall certification time is generally acceptable. 

As shown in Table~\ref{tab:road-sign-exp}, we can see that with knowledge integration, \sr achieves consistently higher certified accuracy compared to the baseline smoothed ML model without knowledge under all the perturbation magnitudes $C_I$ and smoothing parameter $\hat{\sigma}$ settings. Under the small perturbation magnitude cases, our improvement is very significant (around $5\%$).
More interestingly, given large $C_I$ but small smoothing parameter $\hat{\sigma}$, vanilla randomized smoothing-based certification directly fails ($0\%$ certified accuracy) due to the looseness of the hypothesis testing bound, while the \sr could still achieve reasonable certified robustness (over $71\%$ on $C_I = 0.50$, $49\%$ on $C_I = 1.00$) under the same $(C_I, \hat{\sigma})$ settings. This indicates a very realistic case: we always \textbf{\emph{under-estimate}} the attacker's ability easily under the real-world setting -- in this case, the \sr could remain robust even provide reasonable certified accuracy with a conservative smoothing parameter.

% Please add the following required packages to your document preamble:
% \usepackage{multirow}
% \usepackage[normalem]{ulem}
% \useunder{\uline}{\ul}{}

\vspace{-0.57em}
\subsection{Results of \textit{Information Extraction}}
\vspace{-0.57em}

In this section, we conduct the certified robustness evaluation on the information extraction task on text data. Since there is no good certification method on discrete NLP data for sensing models, we directly assume the maximal perturbation 
on the output of sensors ($C_S$).
\begin{wraptable}{r}{6.7cm}
\centering
\vspace{-1em}
\caption{\small {{\bf (Information extraction)} \textit{Certified accuracy} under different perturbation magnitudes ($C_S$) based on the sensing models' output uncertainty. (All certificates hold with 99.9\% confidence)  }}
\vspace{-0.57em}
\scalebox{0.63}{
\begin{tabular}{c|c|c|c}
\toprule
\textbf{Methods}                                                                                             & $C_S=0.1$                       & $C_S=0.5$                       & $C_S=0.9$                      \\ \hline
\multirow{2}{*}{\begin{tabular}[c]{@{}c@{}}Vanilla Smoothing\\ (w/o knowledge)\end{tabular}}        & \multirow{2}{*}{99.7}           & \multirow{2}{*}{94.7}           & \multirow{2}{*}{38.4}          \\
                                                                                                    &                                 &                                 &                                \\ \hline
\multirow{2}{*}{\begin{tabular}[c]{@{}c@{}}Sensing-Reasoning Pipeline\\ (w/ knowledge)\end{tabular}} & \multirow{2}{*}{\textbf{100.0}} & \multirow{2}{*}{\textbf{100.0}} & \multirow{2}{*}{\textbf{58.8}} \\
                                                                                                    &                                 &                                 &                                \\ \bottomrule
\end{tabular}
%\vspace{-6mm}
\label{exp:newnlp1}}
\vspace{-2.3mm}
%\end{table}
\end{wraptable}
Table~\ref{exp:newnlp1} shows the 
certified accuracy on the final outputs of the reasoning component. We see that the \sr provides significantly higher certified robustness, and even under a high perturbation magnitude on all sensing models' output confidence ($C_S = 0.5$), which means the \sr can still leverage the knowledge to help enhance the robustness given strong attacker.
To further illustrate intuitively why such knowledge-based reasoning helps, Figure~\ref{fig:nlp_margin}
shows the ``margin'' --- the probability 
of the ground truth class minus
the probability of the wrong class ---
with or without knowledge integration.
We see that, with knowledge integration,
we can significantly increase the number
of examples with a large ``margin''
under adversarial perturbations.
This explains the improvement of certified robustness, which highly relies on such prediction confident margin.

We also conduct experiments on PrimateNet, Word50, MNIST, CIFAR10 datasets for the image classification tasks in~\Cref{adx:primate-section}-~\Cref{adx:constructed-knowledge}. We observe similar results that knowledge integration significantly boosts the certified robustness.

%\begin{table}[t]

\begin{figure*}[t]
\centering
\includegraphics[width=0.98\textwidth]{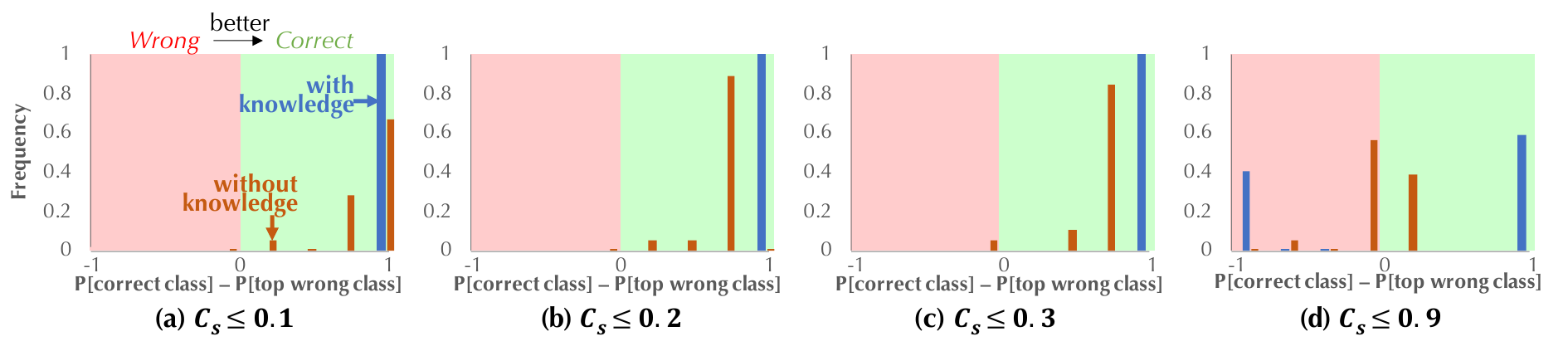}
%\vspace{-1em}
\caption{ \small{{\bf (Information extraction)} Histogram of the \textbf{robustness margin} (the difference between the probability of the correct
class (lower bound) and the top wrong class (upper bound)) under perturbations.
If such a difference is positive, it means
that the classifier makes the right prediction
under perturbations. }}
\label{fig:nlp_margin}
%\vspace{-7mm}
\end{figure*}

\section{Related Work} 
\label{sec:relatedwork}
%\vspace{-0.5em}

%This paper is compatible with two lines of work: the 
%robustness for a single ML models or ensemble, and 
%joint inference with domain knowledge enabled by 
%statistical inference.
%vspace{-4mm}
%\vspace{-2mm}
{\bf Robustness for Single ML model and ML Ensemble.}
Lots of efforts have been made to improve the robustness of single ML or ensemble models. Adversarial training ~\cite{goodfellow2014explaining}, and its variations~\cite{tramer2017ensemble,madry2017towards,xiao2018characterizing} have generally been more successful in practice, but usually come at the cost of accuracy and increased training time~\cite{tsipras2018there,xiao2018characterizing}. 
To further provide certifiable robustness guarantees for ML models, 
various certifiable defenses and robustness verification approaches have been proposed~\cite{kolter2017provable,tjeng2017evaluating,cohen2019certified,li2022dsrs,li2020sok}. Among these strategies, randomized smoothing~\cite{cohen2019certified} has achieved scalable performance. With improvements in training, including pretraining and adversarial training, the certified robustness bound can be further improved~\cite{carmon2019unlabeled,salman2019provably}. 
In addition to the single ML model, some work proposed to promote the diversity of classifiers and therefore develop a robust ML ensemble~\cite{pang2019improving,yangli2021trs,yang2021certified,yang2022on}.
Although promising, these defense approaches, either empirical or theoretical, can only improve the robustness of a single ML or ensemble model. Certifying or improving the robustness of such single or pure ensemble models is very challenging, given that there is no additional information that can be utilized. In addition, the ML learning process usually favors a pipeline that is able to incorporate different sensing components as well as domain knowledge in practice. Thus, certifying the robustness of such pipelines is of great importance.

% \subsubsection*{Robustness for NLP models} Recent studies show that even large-scale pre-trained language models such as BERT \cite{bert} are vulnerable to adversarial attacks \cite{textfooler, comattack}. To defend against textual adversarial attacks, recent work  uses adversarial training as a practical method to defend against adversarial examples, by performing PGD attacks in the embedding space \cite{freelb}, or regularizes the objective function using virtual adversarial training \cite{smart, alum}. However, such defenses fail to certify the robustness, especially when the threat model is unknown. To achieve certified robustness for NLP models, recent studies applied \textit{Interval Bound Propagation} (IBP) \cite{ibp} by considering the worst-case perturbation theoretically to NLP domain \cite{ibp1, ibp2}. However, IBP-based methods rely on strong assumptions of model architecture and are difficult to adapt to recent transformer-based language models.  Other work \cite{safer} adapts \textit{randomized smoothing} \cite{cohen2019certified} to the NLP domain to achieve certified robustness, which essentially replaces the Gaussian noise with synonym words to certify the robustness as long as adversarial word substitution falls into predefined synonym sets. 

%\vspace{-1mm}
{\bf Robustness of End-to-end ML Systems.}
There have been intensive studies on 
joint inference between multiple models, and the  predictions based on joint inference  can help to further improve the clean accuracy of ML pipelines~\cite{xu2019joint,deng2014large,poon2007joint,mccallum-2009-joint,chen-etal-2014-joint-inference,chakrabarti2014joint}, which have been applied to a range of 
real-world applications~\cite{biba2011protein,10.1371/journal.pone.0113523,10.1093/bioinformatics/btv476}. Often, these approaches
use different statistical inference
models such as factor graphs~\cite{MAL-001}, Markov logic networks~\cite{RichardsonMarkov}, and Bayesian 
networks~\cite{10.5555/331969} as a way to integrate 
domain knowledge. In this paper, we take a 
different perspective on this problem ---
instead of treating joint inference as a 
way to improve the {\em clean accuracy}, we
explore the possibility of using it
as exogenous information to improve the end-to-end {\em certified robustness} of ML pipelines.
A recent work~\cite{gurel2021knowledge} explores the empirical robustness improvement via knowledge integration, while there is no robustness guarantee provided.
As we show in this paper, by integrating
domain knowledge, we are able to 
improve the \textit{certified robustness} of the ML pipelines significantly.

% Another line of work tries to analyze the robustness of end-to-end systems with ML components based on their temporal logic properties~\cite{dreossi2019compositional}. However, there is no domain knowledge enhanced robustness that can be improved or certified in these work. 

%\vspace{-4mm}
%\vspace{-1em}
\section{Conclusions}
%\vspace{-0.5em}
We provide the first certifiably robust \sr with knowledge-based logical reasoning. We theoretically prove the certified robustness of such ML pipelines, and provide complexity analysis for certifying the reasoning component. 
Our extensive empirical results demonstrate the certified robustness of \sr, and we believe our work
would shed light on future research towards improving and certifying robustness for general ML frameworks as well as different ways to integrate logical reasoning with statistical learning.
%\input{body}

%\vspace{-2mm}
\paragraph{Acknowledgements} {\scriptsize This work is partially supported by the NSF grant No.1910100,
NSF CNS No.2046726, C3 AI, and the Alfred P. Sloan Foundation.
CZ and the DS3Lab gratefully acknowledge the support from the Swiss State Secretariat for Education, Research and Innovation (SERI) under contract number MB22.00036 (for European Research Council (ERC) Starting Grant TRIDENT 101042665), the Swiss National Science Foundation (Project Number 200021\_184628, and 197485), Innosuisse/SNF BRIDGE Discovery (Project Number 40B2-0\_187132), European Union Horizon 2020 Research and Innovation Programme (DAPHNE, 957407), Botnar Research Centre for Child Health, Swiss Data Science Center, Alibaba, Cisco, eBay, Google Focused Research Awards, Kuaishou Inc., Oracle Labs, Zurich Insurance, and the Department of Computer Science at ETH Zurich.
HG has received funding from the European Research Council (ERC) under the European Union’s Horizon 2020
research and innovation programme (grant agreement No. 947778).}

\newpage
\bibliographystyle{plain}
\bibliography{example_paper,supp_zz}

\begin{thebibliography}{10}

\bibitem{athalye2018obfuscated}
Anish Athalye, Nicholas Carlini, and David Wagner.
\newblock Obfuscated gradients give a false sense of security: Circumventing
  defenses to adversarial examples.
\newblock In {\em International Conference on Machine Learning}, pages
  274--283, 2018.

\bibitem{biba2011protein}
Marenglen Biba, Stefano Ferilli, and Floriana Esposito.
\newblock Protein fold recognition using markov logic networks.
\newblock In {\em Mathematical Approaches to Polymer Sequence Analysis and
  Related Problems}, pages 69--85. Springer, 2011.

\bibitem{carlini2017towards}
Nicholas Carlini and David Wagner.
\newblock Towards evaluating the robustness of neural networks.
\newblock In {\em 2017 ieee symposium on security and privacy (sp)}, pages
  39--57. IEEE, 2017.

\bibitem{carmon2019unlabeled}
Yair Carmon, Aditi Raghunathan, Ludwig Schmidt, Percy Liang, and John~C Duchi.
\newblock Unlabeled data improves adversarial robustness.
\newblock {\em arXiv preprint arXiv:1905.13736}, 2019.

\bibitem{chakrabarti2014joint}
Deepayan Chakrabarti, Stanislav Funiak, Jonathan Chang, and Sofus~A Macskassy.
\newblock Joint inference of multiple label types in large networks.
\newblock {\em arXiv preprint arXiv:1401.7709}, 2014.

\bibitem{chen2015learning}
Liang-Chieh Chen, Alexander Schwing, Alan Yuille, and Raquel Urtasun.
\newblock Learning deep structured models.
\newblock In {\em International Conference on Machine Learning}, pages
  1785--1794. PMLR, 2015.

\bibitem{chen-etal-2014-joint-inference}
Liwei Chen, Yansong Feng, Jinghui Mo, Songfang Huang, and Dongyan Zhao.
\newblock Joint inference for knowledge base population.
\newblock In {\em Proceedings of the 2014 Conference on Empirical Methods in
  Natural Language Processing ({EMNLP})}, pages 1912--1923, Doha, Qatar,
  October 2014. Association for Computational Linguistics.

\bibitem{cohen2019certified}
Jeremy~M Cohen, Elan Rosenfeld, and J~Zico Kolter.
\newblock Certified adversarial robustness via randomized smoothing.
\newblock {\em arXiv preprint arXiv:1902.02918}, 2019.

\bibitem{imagenet_cvpr09}
J.~Deng, W.~Dong, R.~Socher, L.-J. Li, K.~Li, and L.~Fei-Fei.
\newblock {ImageNet: A Large-Scale Hierarchical Image Database}.
\newblock In {\em CVPR09}, 2009.

\bibitem{deng2014large}
Jia Deng, Nan Ding, Yangqing Jia, Andrea Frome, Kevin Murphy, Samy Bengio, Yuan
  Li, Hartmut Neven, and Hartwig Adam.
\newblock Large-scale object classification using label relation graphs.
\newblock In {\em European conference on computer vision}, pages 48--64.
  Springer, 2014.

\bibitem{bert}
Jacob Devlin, Ming{-}Wei Chang, Kenton Lee, and Kristina Toutanova.
\newblock {BERT:} pre-training of deep bidirectional transformers for language
  understanding.
\newblock In Jill Burstein, Christy Doran, and Thamar Solorio, editors, {\em
  NAACL-HLT}, pages 4171--4186. Association for Computational Linguistics,
  2019.

\bibitem{ding2014using}
Xiao Ding, Yue Zhang, Ting Liu, and Junwen Duan.
\newblock Using structured events to predict stock price movement: An empirical
  investigation.
\newblock In {\em Proceedings of the 2014 Conference on Empirical Methods in
  Natural Language Processing (EMNLP)}, pages 1415--1425, 2014.

\bibitem{eykholt2018robust}
Kevin Eykholt, Ivan Evtimov, Earlence Fernandes, Bo~Li, Amir Rahmati, Chaowei
  Xiao, Atul Prakash, Tadayoshi Kohno, and Dawn Song.
\newblock Robust physical-world attacks on deep learning visual classification.
\newblock In {\em Proceedings of the IEEE Conference on Computer Vision and
  Pattern Recognition}, pages 1625--1634, 2018.

\bibitem{fellbaum2012wordnet}
Christiane Fellbaum.
\newblock Wordnet.
\newblock {\em The encyclopedia of applied linguistics}, 2012.

\bibitem{goodfellow2014explaining}
Ian~J. Goodfellow, Jonathon Shlens, and Christian Szegedy.
\newblock Explaining and harnessing adversarial examples.
\newblock {\em ICLR}, 2015.

\bibitem{gowal2018effectiveness}
Sven Gowal, Krishnamurthy Dvijotham, Robert Stanforth, Rudy Bunel, Chongli Qin,
  Jonathan Uesato, Relja Arandjelovic, Timothy Mann, and Pushmeet Kohli.
\newblock On the effectiveness of interval bound propagation for training
  verifiably robust models.
\newblock {\em arXiv preprint arXiv:1810.12715}, 2018.

\bibitem{gurel2021knowledge}
Nezihe~Merve G{\"u}rel, Xiangyu Qi, Luka Rimanic, Ce~Zhang, and Bo~Li.
\newblock Knowledge enhanced machine learning pipeline against diverse
  adversarial attacks.
\newblock {\em ICML}, 2021.

\bibitem{huynh2009max}
Tuyen~N Huynh and Raymond~J Mooney.
\newblock Max-margin weight learning for markov logic networks.
\newblock In {\em Joint European Conference on Machine Learning and Knowledge
  Discovery in Databases}, pages 564--579. Springer, 2009.

\bibitem{jeong2020consistency}
Jongheon Jeong and Jinwoo Shin.
\newblock Consistency regularization for certified robustness of smoothed
  classifiers.
\newblock {\em Advances in Neural Information Processing Systems},
  33:10558--10570, 2020.

\bibitem{kolter2017provable}
J~Zico Kolter and Eric Wong.
\newblock Provable defenses against adversarial examples via the convex outer
  adversarial polytope.
\newblock {\em arXiv preprint arXiv:1711.00851}, 2017.

\bibitem{kuzelka2020complex}
Ondrej Kuzelka.
\newblock Complex markov logic networks: Expressivity and liftability.
\newblock In {\em Conference on Uncertainty in Artificial Intelligence}, pages
  729--738. PMLR, 2020.

\bibitem{li2014feature}
Bo~Li and Yevgeniy Vorobeychik.
\newblock Feature cross-substitution in adversarial classification.
\newblock In {\em Advances in neural information processing systems}, pages
  2087--2095, 2014.

\bibitem{li2020nolinear}
Huichen Li, Linyi Li, Xiaojun Xu, Xiaolu Zhang, Shuang Yang, and Bo~Li.
\newblock Nonlinear gradient estimation for query efficient blackbox attack.
\newblock In {\em International Conference on Artificial Intelligence and
  Statistics (AISTATS 2021)}, Proceedings of Machine Learning Research. PMLR,
  13--15 Apr 2021.

\bibitem{li2020qeba}
Huichen Li, Xiaojun Xu, Xiaolu Zhang, Shuang Yang, and Bo~Li.
\newblock Qeba: Query-efficient boundary-based blackbox attack.
\newblock In {\em Proceedings of the IEEE/CVF Conference on Computer Vision and
  Pattern Recognition}, pages 1221--1230, 2020.

\bibitem{li2020sok}
Linyi Li, Xiangyu Qi, Tao Xie, and Bo~Li.
\newblock Sok: Certified robustness for deep neural networks.
\newblock {\em arXiv}, abs/2009.04131, 2020.

\bibitem{li2021tss}
Linyi Li, Maurice Weber, Xiaojun Xu, Luka Rimanic, Bhavya Kailkhura, Tao Xie,
  Ce~Zhang, and Bo~Li.
\newblock Tss: Transformation-specific smoothing for robustness certification.
\newblock In {\em ACM Conference on Computer and Communications Security (CCS
  2021)}, 2021.

\bibitem{li2022dsrs}
Linyi Li, Jiawei Zhang, Tao Xie, and Bo~Li.
\newblock Double sampling randomized smoothing.
\newblock In {\em International Conference on Machine Learning}, 2022.

\bibitem{li2019robustra}
Linyi Li, Zexuan Zhong, Bo~Li, and Tao Xie.
\newblock Robustra: training provable robust neural networks over reference
  adversarial space.
\newblock In {\em Proceedings of the 28th International Joint Conference on
  Artificial Intelligence}, pages 4711--4717. AAAI Press, 2019.

\bibitem{10.1007/978-3-540-74976-9_21}
Daniel Lowd and Pedro Domingos.
\newblock Efficient weight learning for markov logic networks.
\newblock In Joost~N. Kok, Jacek Koronacki, Ramon Lopez~de Mantaras, Stan
  Matwin, Dunja Mladeni{\v{c}}, and Andrzej Skowron, editors, {\em Knowledge
  Discovery in Databases: PKDD 2007}, pages 200--211, Berlin, Heidelberg, 2007.
  Springer Berlin Heidelberg.

\bibitem{ma2018characterizing}
Xingjun Ma, Bo~Li, Yisen Wang, Sarah~M Erfani, Sudanthi Wijewickrema, Grant
  Schoenebeck, Dawn Song, Michael~E Houle, and James Bailey.
\newblock Characterizing adversarial subspaces using local intrinsic
  dimensionality.
\newblock {\em arXiv preprint arXiv:1801.02613}, 2018.

\bibitem{madry2017towards}
Aleksander Madry, Aleksandar Makelov, Ludwig Schmidt, Dimitris Tsipras, and
  Adrian Vladu.
\newblock Towards deep learning models resistant to adversarial attacks.
\newblock {\em arXiv preprint arXiv:1706.06083}, 2017.

\bibitem{10.1093/bioinformatics/btv476}
Emily~K. Mallory, Ce~Zhang, Christopher Ré, and Russ~B. Altman.
\newblock {Large-scale extraction of gene interactions from full-text
  literature using DeepDive}.
\newblock {\em Bioinformatics}, 32(1):106--113, 09 2015.

\bibitem{mccallum-2009-joint}
Andrew McCallum.
\newblock Joint inference for natural language processing.
\newblock In {\em Proceedings of the Thirteenth Conference on Computational
  Natural Language Learning ({C}o{NLL}-2009)}, page~1, Boulder, Colorado, June
  2009. Association for Computational Linguistics.

\bibitem{pang2019improving}
Tianyu Pang, Kun Xu, Chao Du, Ning Chen, and Jun Zhu.
\newblock Improving adversarial robustness via promoting ensemble diversity.
\newblock {\em arXiv preprint arXiv:1901.08846}, 2019.

\bibitem{10.5555/331969}
Judea Pearl.
\newblock {\em Causality: Models, Reasoning, and Inference}.
\newblock Cambridge University Press, USA, 2000.

\bibitem{pearl2011bayesian}
Judea Pearl.
\newblock Bayesian networks.
\newblock 2011.

\bibitem{10.1371/journal.pone.0113523}
Shanan~E. Peters, Ce~Zhang, Miron Livny, and Christopher Ré.
\newblock A machine reading system for assembling synthetic paleontological
  databases.
\newblock {\em PLOS ONE}, 9(12):1--22, 12 2014.

\bibitem{poon2007joint}
Hoifung Poon and Pedro Domingos.
\newblock Joint inference in information extraction.
\newblock In {\em AAAI}, volume~7, pages 913--918, 2007.

\bibitem{qiu2020semanticadv}
Haonan Qiu, Chaowei Xiao, Lei Yang, Xinchen Yan, Honglak Lee, and Bo~Li.
\newblock Semanticadv: Generating adversarial examples via
  attribute-conditioned image editing.
\newblock In {\em European Conference on Computer Vision}, pages 19--37.
  Springer, 2020.

\bibitem{richardson2006markov}
Matthew Richardson and Pedro Domingos.
\newblock Markov logic networks.
\newblock {\em Machine learning}, 62(1-2):107--136, 2006.

\bibitem{RichardsonMarkov}
Matthew Richardson and Pedro Domingos.
\newblock Markov logic networks.
\newblock {\em Machine Learning}, 62(1-2):107--136, 2006.

\bibitem{salman2019provably}
Hadi Salman, Jerry Li, Ilya Razenshteyn, Pengchuan Zhang, Huan Zhang, Sebastien
  Bubeck, and Greg Yang.
\newblock Provably robust deep learning via adversarially trained smoothed
  classifiers.
\newblock In {\em Advances in Neural Information Processing Systems}, pages
  11289--11300, 2019.

\bibitem{santurkar2021editing}
Shibani Santurkar, Dimitris Tsipras, Mahalaxmi Elango, David Bau, Antonio
  Torralba, and Aleksander Madry.
\newblock Editing a classifier by rewriting its prediction rules.
\newblock {\em Advances in Neural Information Processing Systems},
  34:23359--23373, 2021.

\bibitem{shafahi2019adversarial}
Ali Shafahi, Mahyar Najibi, Amin Ghiasi, Zheng Xu, John Dickerson, Christoph
  Studer, Larry~S Davis, Gavin Taylor, and Tom Goldstein.
\newblock Adversarial training for free!
\newblock {\em arXiv preprint arXiv:1904.12843}, 2019.

\bibitem{stallkamp2012man}
Johannes Stallkamp, Marc Schlipsing, Jan Salmen, and Christian Igel.
\newblock Man vs. computer: Benchmarking machine learning algorithms for
  traffic sign recognition.
\newblock {\em Neural networks}, 32:323--332, 2012.

\bibitem{tjeng2017evaluating}
Vincent Tjeng, Kai Xiao, and Russ Tedrake.
\newblock Evaluating robustness of neural networks with mixed integer
  programming.
\newblock {\em arXiv preprint arXiv:1711.07356}, 2017.

\bibitem{tramer2020adaptive}
Florian Tramer, Nicholas Carlini, Wieland Brendel, and Aleksander Madry.
\newblock On adaptive attacks to adversarial example defenses.
\newblock {\em arXiv preprint arXiv:2002.08347}, 2020.

\bibitem{tramer2017ensemble}
Florian Tram{\`e}r, Alexey Kurakin, Nicolas Papernot, Dan Boneh, and Patrick
  McDaniel.
\newblock Ensemble adversarial training: Attacks and defenses.
\newblock {\em ICLR}, 2018.

\bibitem{tsipras2018there}
Dimitris Tsipras, Shibani Santurkar, Logan Engstrom, Alexander Turner, and
  Aleksander Madry.
\newblock Robustness may be at odds with accuracy).
\newblock {\em ICLR 2019}, 2018.

\bibitem{VALIANT1979189}
L.G. Valiant.
\newblock The complexity of computing the permanent.
\newblock {\em Theoretical Computer Science}, 8(2):189--201, 1979.

\bibitem{MAL-001}
Martin~J. Wainwright and Michael~I. Jordan.
\newblock Graphical models, exponential families, and variational inference.
\newblock {\em Foundations and Trends® in Machine Learning}, 1(1–2):1--305,
  2008.

\bibitem{weng2018towards}
Lily Weng, Huan Zhang, Hongge Chen, Zhao Song, Cho-Jui Hsieh, Luca Daniel,
  Duane Boning, and Inderjit Dhillon.
\newblock Towards fast computation of certified robustness for relu networks.
\newblock In {\em International Conference on Machine Learning}, pages
  5276--5285, 2018.

\bibitem{xiao2019advit}
Chaowei Xiao, Ruizhi Deng, Bo~Li, Taesung Lee, Benjamin Edwards, Jinfeng Yi,
  Dawn Song, Mingyan Liu, and Ian Molloy.
\newblock Advit: Adversarial frames identifier based on temporal consistency in
  videos.
\newblock In {\em Proceedings of the IEEE/CVF International Conference on
  Computer Vision}, pages 3968--3977, 2019.

\bibitem{xiao2018characterizing}
Chaowei Xiao, Ruizhi Deng, Bo~Li, Fisher Yu, Mingyan Liu, and Dawn Song.
\newblock Characterizing adversarial examples based on spatial consistency
  information for semantic segmentation.
\newblock In {\em Proceedings of the European Conference on Computer Vision
  (ECCV)}, pages 217--234, 2018.

\bibitem{xu2020automatic}
Kaidi Xu, Zhouxing Shi, Huan Zhang, Yihan Wang, Kai-Wei Chang, Minlie Huang,
  Bhavya Kailkhura, Xue Lin, and Cho-Jui Hsieh.
\newblock Automatic perturbation analysis for scalable certified robustness and
  beyond.
\newblock {\em Advances in Neural Information Processing Systems}, 33, 2020.

\bibitem{xu2019joint}
Zhe Xu, Ivan Gavran, Yousef Ahmad, Rupak Majumdar, Daniel Neider, Ufuk Topcu,
  and Bo~Wu.
\newblock Joint inference of reward machines and policies for reinforcement
  learning.
\newblock {\em arXiv preprint arXiv:1909.05912}, 2019.

\bibitem{yang2020randomized}
Greg Yang, Tony Duan, J~Edward Hu, Hadi Salman, Ilya Razenshteyn, and Jerry Li.
\newblock Randomized smoothing of all shapes and sizes.
\newblock In {\em International Conference on Machine Learning}, pages
  10693--10705. PMLR, 2020.

\bibitem{yang2021certified}
Zhuolin Yang, Linyi Li, Xiaojun Xu, Bhavya Kailkhura, Tao Xie, and Bo~Li.
\newblock On the certified robustness for ensemble models and beyond.
\newblock {\em ICLR}, 2021.

\bibitem{yang2022on}
Zhuolin Yang, Linyi Li, Xiaojun Xu, Bhavya Kailkhura, Tao Xie, and Bo~Li.
\newblock On the certified robustness for ensemble models and beyond.
\newblock In {\em International Conference on Learning Representations}, 2022.

\bibitem{yangli2021trs}
Zhuolin Yang, Linyi Li, Xiaojun Xu, Shiliang Zuo, Qian Chen, Pan Zhou, Benjamin
  I.~P. Rubinstein, Ce~Zhang, and Bo~Li.
\newblock Trs: Transferability reduced ensemble via promoting gradient
  diversity and model smoothness.
\newblock In {\em Neural Information Processing Systems (NeurIPS 2021)}, 2021.

\bibitem{yang2021trs}
Zhuolin Yang, Linyi Li, Xiaojun Xu, Shiliang Zuo, Qian Chen, Pan Zhou, Benjamin
  I~P Rubinstein, Ce~Zhang, and Bo~Li.
\newblock Trs: Transferability reduced ensemble via promoting gradient
  diversity and model smoothness.
\newblock In {\em Advances in Neural Information Processing Systems}, 2021.

\bibitem{YangImp2022}
Zhuolin Yang, Zhikuan Zhao, Boxin Wang, Jiawei Zhang, Linyi Li, Hengzhi Pei,
  Bojan Karlaš, Ji~Liu, Heng Guo, Ce~Zhang, and Bo~Li.
\newblock Improving certified robustness via statistical learning with logical
  reasoning.
\newblock {\em NeurIPS}, 2022.

\bibitem{10.1145/3060586}
Ce~Zhang, Christopher R\'{e}, Michael Cafarella, Christopher De~Sa, Alex
  Ratner, Jaeho Shin, Feiran Wang, and Sen Wu.
\newblock Deepdive: Declarative knowledge base construction.
\newblock {\em Commun. ACM}, 60(5):93–102, April 2017.

\bibitem{zhang2018efficient}
Huan Zhang, Tsui-Wei Weng, Pin-Yu Chen, Cho-Jui Hsieh, and Luca Daniel.
\newblock Efficient neural network robustness certification with general
  activation functions.
\newblock In {\em Advances in neural information processing systems}, pages
  4939--4948, 2018.

\bibitem{zhang2021progressive}
Jiawei Zhang, Linyi Li, Huichen Li, Xiaolu Zhang, Shuang Yang, and Bo~Li.
\newblock Progressive-scale boundary blackbox attack via projective gradient
  estimation.
\newblock {\em ICML}, 2022.

\end{thebibliography}
%\balance

\iftrue

%%%%%%%%%%%%%%%%%%%%%%%%%%%%%%%%%%%%%%%%%%%%%%%%%%%%%%%%%%%%%%%%%%%%%%%%%%%%%%%
%%%%%%%%%%%%%%%%%%%%%%%%%%%%%%%%%%%%%%%%%%%%%%%%%%%%%%%%%%%%%%%%%%%%%%%%%%%%%%%
% APPENDIX
%%%%%%%%%%%%%%%%%%%%%%%%%%%%%%%%%%%%%%%%%%%%%%%%%%%%%%%%%%%%%%%%%%%%%%%%%%%%%%%
%%%%%%%%%%%%%%%%%%%%%%%%%%%%%%%%%%%%%%%%%%%%%%%%%%%%%%%%%%%%%%%%%%%%%%%%%%%%%%%

\iftrue

\newpage
\section*{Checklist}

%%% BEGIN INSTRUCTIONS %%%
%The checklist follows the references.  Please
%read the checklist guidelines carefully for information on how to answer these
%questions.  For each question, change the default \answerTODO{} to \answerYes{},
%\answerNo{}, or \answerNA{}.  You are strongly encouraged to include a {\bf
%justification to your answer}, either by referencing the appropriate section of
%your paper or providing a brief inline description.  For %example:
%\begin{itemize}
%  \item Did you include the license to the code and datasets? \answerYes{See Section~\ref{gen_inst}.}
%  \item Did you include the license to the code and datasets? \answerNo{The code and the data are proprietary.}
%  \item Did you include the license to the code and datasets? \answerNA{}
%\end{itemize}
%Please do not modify the questions and only use the provided macros for your
%answers.  Note that the Checklist section does not count towards the page
%limit.  In your paper, please delete this instructions block and only keep the
%Checklist section heading above along with the questions/answers below.
%%% END INSTRUCTIONS %%%

\begin{enumerate}

\item For all authors...
\begin{enumerate}
  \item Do the main claims made in the abstract and introduction accurately reflect the paper's contributions and scope?
    \answerYes{}
  \item Did you describe the limitations of your work? 
    \answerYes{} We have mentioned the future improvement of our work in the related work part.
  \item Did you discuss any potential negative societal impacts of your work?
    \answerYes{} This work will not infer obvious negative societal impacts.
  \item Have you read the ethics review guidelines and ensured that your paper conforms to them?
    \answerYes{}
\end{enumerate}

\item If you are including theoretical results...
\begin{enumerate}
  \item Did you state the full set of assumptions of all theoretical results?
    \answerYes{} The assumptions have been all mentioned in the main paper and appendices.
        \item Did you include complete proofs of all theoretical results?
    \answerYes{} The whole proofs are provided in Appendix~\ref{apx:hardness} - \ref{apx:sup_alg1}.
\end{enumerate}

\item If you ran experiments...
\begin{enumerate}
  \item Did you include the code, data, and instructions needed to reproduce the main experimental results (either in the supplemental material or as a URL)?
    \answerYes{} The code is provided at~\url{https://github.com/Sensing-Reasoning/Sensing-Reasoning-Pipeline}.
  \item Did you specify all the training details (e.g., data splits, hyperparameters, how they were chosen)?
    \answerYes{} All the training details have been provided in the Appendix~\ref{adx:road-sign-section} -  \ref{sec:apx-partial-knowleddge}.
        \item Did you report error bars (e.g., with respect to the random seed after running experiments multiple times)? 
    \answerYes{} The confidence of the reported certification results in the paper is guaranteed to be at least $99.9\%$, as mentioned in our main paper.
        \item Did you include the total amount of compute and the type of resources used (e.g., type of GPUs, internal cluster, or cloud provider)? 
    \answerYes{} The detailed information is mentioned in Appendix~\ref{adx:road-sign-section}.
\end{enumerate}

\item If you are using existing assets (e.g., code, data, models) or curating/releasing new assets...
\begin{enumerate}
  \item If your work uses existing assets, did you cite the creators?
    \answerYes{}
  \item Did you mention the license of the assets?
    \answerYes{}
  \item Did you include any new assets either in the supplemental material or as a URL?
    \answerYes{}
  \item Did you discuss whether and how consent was obtained from people whose data you're using/curating?
    \answerYes{} We only use public and commonly used data.
  \item Did you discuss whether the data you are using/curating contains personally identifiable information or offensive content?
    \answerYes{} We only use public and commonly used data.
\end{enumerate}

\item If you used crowdsourcing or conducted research with human subjects...
\begin{enumerate}
  \item Did you include the full text of instructions given to participants and screenshots, if applicable?
    \answerNA{}
  \item Did you describe any potential participant risks, with links to Institutional Review Board (IRB) approvals, if applicable?
    \answerNA{}
  \item Did you include the estimated hourly wage paid to participants and the total amount spent on participant compensation?
    \answerNA{}
\end{enumerate}

\end{enumerate}

\onecolumn

\appendix
\newpage

\def\Counting{\textnormal{\textsc{Counting}}}
\def\Robustness{\textnormal{\textsc{Robustness}}}

\section{Hardness of General Distribution }
\label{apx:hardness}
We first recall the following definitions:

\noindent \textbf{Counting.}
Given input polynomial-time computable weight function $w(\cdot)$ and query function $Q(\cdot)$, parameters $\alpha$, a real number $\eps>0$, a \textsc{Counting} oracle outputs a real number $Z$ such that 
\[
1-\eps\le \frac{Z}{\E[\sigma\sim\pi_{\alpha}]{Q(\sigma)}}\le 1+\eps.
\]

\def\Robustness{\textnormal{\textsc{Robustness}}}

\noindent \textbf{Robustness.}
Given input polynomial-time computable weight function $w(\cdot)$ and query function $Q(\cdot)$, parameters $\alpha$, two real numbers $\epsilon >0$ and $\delta>0$, a \textsc{Robustness} oracle decides, for any $\alpha'\in P^{[m]}$ such that $\norm{\alpha-\alpha'}_\infty \le \epsilon$, whether the following is true:
\begin{align*}
  \abs{\E[\sigma\sim\pi_{\alpha}]{Q(\sigma)}-\E[\sigma\sim\pi_{\alpha'}]{Q(\sigma)}}<\delta.
\end{align*}

\def\reduceT{\le_\texttt{t}}
\subsubsection*{Proof of Theorem \ref{hardness_from_reduction}}

\newtheorem*{theorem hardness_from_reduction}{Theorem \ref{hardness_from_reduction}}
\begin{theorem hardness_from_reduction}[$\Counting\reduceT\Robustness$]
Given polynomial-time computable weight function $w(\cdot)$ and query function $Q(\cdot)$, parameters $\alpha$ and real number $\eps>0$, the instance of \textsc{Counting}, $(w,Q,\alpha,\eps)$ can be determined by up to $O(1/\varepsilon_c^2)$ queries of the \textsc{Robustness} oracle with input perturbation $\epsilon=O(\varepsilon_c)$.
\label{supp:theorem:counting-to-robustness}
\end{theorem hardness_from_reduction}
\begin{proof}
  Let $(w,Q,\alpha,\eps)$ be an instance of \Counting.
  Define a new distribution $\tau_{\beta}$ over $\*X$ with a single parameter $\beta\in\mathbb{R}$ such that
$
    \tau_{\beta}(\sigma) \propto t(\sigma;\beta),
$
  where
$
    t(\sigma;\beta) = w(\sigma;\alpha)\exp(\beta Q(\sigma)).
$
  Since $Q$ is polynomial-time computable,
  $\tau_{\beta}$ is accessible for any $\beta$.
  We will choose $\beta$ later. For $i\in\{0,1\}$, define
$
    Z_i:=\sum_{\sigma:Q(\sigma)=i} w(\sigma;\alpha).
$
  Then we have 
  {
  $$\E[\sigma\sim\pi_{\alpha}]{Q(\sigma)} = \frac{Z_1}{Z_0+Z_1},\quad \E[\sigma\sim\tau_{\beta}]{Q(\sigma)} = \frac{e^\beta Z_1}{Z_0+e^\beta Z_1}.$$
  }
We further define
  \begin{align*}
    Y^+(\beta,x)&:=\E[\sigma\sim\tau_{\beta+x}]{Q(\sigma)}-\E[\sigma\sim\tau_{\beta}]{Q(\sigma)} \\
    &= \frac{e^x e^\beta Z_1}{Z_0+ e^x e^\beta Z_1} - \frac{e^\beta Z_1}{Z_0+e^\beta Z_1}\\
    & = \frac{(e^{x}-1)e^\beta Z_0Z_1}{(Z_0+ e^x e^\beta Z_1)(Z_0+ e^\beta Z_1)}= \frac{(e^{x}-1)e^\beta}{R+(e^x+1)e^\beta+\frac{e^x e^{2\beta}}{R}},
  \end{align*}
  where $R:=\frac{Z_0}{Z_1}$,
  and similarly
  \begin{align*}
    Y^-(\beta,x)&:=\E[\sigma\sim\tau_{\beta}]{Q(\sigma)}-\E[\sigma\sim\tau_{\beta-x}]{Q(\sigma)} \\
    &= \frac{e^\beta Z_1}{Z_0+e^\beta Z_1}-\frac{e^{-x} e^\beta Z_1}{Z_0+ e^{-x} e^\beta Z_1}\\
    & = \frac{(1-e^{-x})e^\beta Z_0Z_1}{(Z_0+ e^{-x} e^\beta Z_1)(Z_0+ e^\beta Z_1)} = \frac{(e^{x}-1)e^\beta}{e^xR+(e^{x}+1)e^\beta+\frac{e^{2\beta}}{R}}.
  \end{align*}
  Easy calculation implies that for $x>0$, $Y^+(\beta,x)>Y^-(\beta,x)$ if and only if $R>e^{\beta}$. Note that
  \begin{align*}
    Y^+(\beta,x) &= \frac{e^x-1}{Re^{-\beta}+e^x+1+\frac{e^{x+\beta}}{R}}\le\frac{e^{x/2}-1}{e^{x/2}+1};\\
    Y^-(\beta,x) &= \frac{e^x-1}{Re^{x-\beta}+e^x+1+\frac{e^{\beta}}{R}}\le\frac{e^{x/2}-1}{e^{x/2}+1}.
  \end{align*}
  The two maximum are achieved when $R=e^{\beta\pm x/2}$.
  We will choose $x=O(\eps)$. Define 
  \begin{align*}
    Y(\beta):=& \max\{Y^+(\beta,x),Y^-(\beta,x)\}\\
    =&
    \begin{cases}
      \frac{e^x-1}{Re^{-\beta}+e^x+1+\frac{e^{x+\beta}}{R}} & \text{if $e^\beta<R$;}\\
      \frac{e^x-1}{Re^{x-\beta}+e^x+1+\frac{e^{\beta}}{R}} & \text{if $e^\beta\ge R$.}\\      
    \end{cases}
  \end{align*}
  This function $Y$ is increasing in $[0,\log R-x/2]$, decreasing in $[\log R-x/2,\log R]$,
  increasing in $[\log R, \log R+x/2]$ again, and decreasing in $[\log R+x/2,\infty)$ once again.

  Our goal is to estimate $R$.
  For any fixed $\beta$, we will query the \Robustness{} oracle with parameters $(t,Q,\beta,x,\delta)$.
  Using binary search in $\delta$, we can estimate the function $Y(\beta)$ above efficiently with additive error $\eps'$ with at most $O(\log\frac{1}{\eps'})$ oracle calls.
  We use binary search once again in $\beta$ so that it stops only if $Y(\beta_0)\ge \frac{e^{x/2}-1}{e^{x/2}+1}-\varepsilon_0$ for some $\beta_0$ and the accuracy $\varepsilon_0\le\frac{e^{x/2}-1}{2(e^{x/2}+1)}$ is to be fixed later.
  In particular, $Y(\beta_0)\ge\frac{e^{x/2}-1}{2(e^{x/2}+1)}$.
  Note that here $\varepsilon_0$ is the accumulated error from binary searching twice.
  
  We claim that $\beta_0$ is a good estimator for $\log R$.
  First assume that $e^{\beta_0}<R$, which implies that
  \begin{align*}
    \frac{e^{x/2}-1}{e^{x/2}+1} - Y(\beta_0)& = \frac{e^{x/2}-1}{e^{x/2}+1} - \frac{e^x-1}{Re^{-\beta_0}+e^x+1+\frac{e^{x+\beta_0}}{R}} \\
    &= \frac{(e^x-1)}{(e^{x/2}+1)^2(Re^{-\beta_0}+e^x+1+\frac{e^{x+\beta_0}}{R})} \times \\
    &~~~~\left( \sqrt{Re^{-\beta_0}}-\sqrt{\frac{e^{x+\beta_0}}{R}} \right)^2\\
    & = \frac{Y(\beta_0)}{(e^{x/2}+1)^2} \left( \sqrt{Re^{-\beta_0}}-\sqrt{\frac{e^{x+\beta_0}}{R}} \right)^2 \le \varepsilon_0.
  \end{align*}
  Thus,
  \begin{align*}
    \abs{\sqrt{Re^{-\beta_0}}-\sqrt{\frac{e^{x+\beta_0}}{R}}} \le \sqrt{\frac{2\varepsilon_0(e^{x/2}+1)^3}{e^{x/2}-1}}.
  \end{align*}
  Let $\rho:=Re^{-\beta_0}$. Note that $\rho>1$.
  We choose $\varepsilon_0:=\frac{1}{2}\left(\frac{e^{x/2}-1}{e^{x/2}+1} \right)^3$.
  Then $\abs{\sqrt{\rho}-\sqrt{e^x/\rho}} < e^{x/2}-1$.
  If $\rho\ge e^x$, then $\abs{\sqrt{\rho}-\sqrt{e^x/\rho}}\ge e^{x/2}-1$, a contradiction.
  Thus, $\rho < e^x$.
  It implies that $1<\frac{R}{e^{\beta_0}}<e^x$.
  Similarly for the case of $e^{\beta_0}>R$, we have that $e^{-x}<\frac{R}{e^{\beta_0}}<1$.
  Thus in both cases, we have our estimator $e^{-x}<\frac{R}{e^{\beta_0}}<e^x$.
  
  Finally, to estimate $\E[\sigma\sim\pi_{\alpha}]{Q(\sigma)} =\frac{1}{1+R}$ with multiplicative error $\eps$,
  we only need to pick $x:=\log(1+\eps)=O(\eps)$.
\end{proof}

\section{Robustness of MLN}
\label{apx:robustness_mln}
\subsubsection*{Lagrange multipliers} Before proving the robustness result of MLN, we first briefly review the technique of Lagrange multipliers for constrained optimization: Consider following problem $\texttt{P}$,\[
\texttt{P:}~~~~\max_{x_1, x_2} f(x_1) + g(x_2),
 ~~~~s.t., x_1 = x_2, ~~~ h(x_1), k(x_2) \ge 0.
\]
\noindent Introducing another real variable 
$\lambda$, we define
following problem $\texttt{P'}$,\[
\texttt{P':}~~~~\max_{x_1, x_2} f(x_1) + g(x_2) + \lambda(x_1 - x_2),
 ~~~~s.t., h(x_1), k(x_2) \ge 0.
\]
For all $\lambda$,
let $(x_1^*, x_2^*)$ be the
solution of \texttt{P} and let
$(\bar{x_1}, \bar{x_2})$ be the
solution of \texttt{P'}, we have\[
f(x_1^*) + g(x_2^*) \leq 
f(x_1^*) + g(x_2^*) + \lambda(x_1^* - x_2^*)
\leq
f(\bar{x_1}) + g(\bar{x_2}) + \lambda(\bar{x_1} - \bar{x_2})
\]
\subsubsection*{Proof of Lemma \ref{MLN Robustness}}
\newtheorem*{lemmaMLN}{Lemma \ref{MLN Robustness}}
\begin{lemmaMLN}[MLN Robustness]
Given access to partition functions $Z_1(\{p_i(X)\}_{i \in [n]})$ and 
$Z_2(\{p_i(X)\}_{i \in [n]})$, and a maximum perturbations $\{C_i\}_{i\in[n]}$, 
$\forall \epsilon_1, ..., \epsilon_n$, if $\forall i.~|\epsilon_i| < C_i$,
we have that $\forall \lambda_1, ..., \lambda_n \in \mathbb{R}$,
\begin{align*}
\max_{\{|\epsilon_i < C_i|\}} \ln \mathbb{E}[R_{MLN}(\{p_i(X) + \epsilon_i\}_{i \in [n]})] 
\leq& \max_{\{|\epsilon_i|<C_i\}}
\widetilde{Z_1}(\{\epsilon_i\}_{i \in [n]}) \\
 -&
\min_{\{|\epsilon_i'|<C_i\}}
\widetilde{Z_2}(\{\epsilon_i'\}_{i \in [n]}) 
\end{align*}
\begin{align*}
\min_{\{|\epsilon_i < C_i|\}} \ln \mathbb{E}[R_{MLN}(\{p_i(X) + \epsilon_i\}_{i \in [n]})] \geq & \min_{\{|\epsilon_i|<C_i\}}
\widetilde{Z_1}(\{\epsilon_i\}_{i \in [n]}) \\
 -&
\max_{\{|\epsilon_i'|<C_i\}}
\widetilde{Z_2}(\{\epsilon_i'\}_{i \in [n]}) 
\end{align*}
where\[
\widetilde{Z_r}(\{\epsilon_i\}_{i \in [n]}) = \ln Z_r(\{p_i(X)+\epsilon_i\}_{i \in [n]}) + \sum_i \lambda_i \epsilon_i.
\]
\end{lemmaMLN}
\begin{proof}
Consider the upper bound, we have 
\begin{small}\begin{align*}
\max_{\{|\epsilon_i < C_i|\}}\ln\mathbb{E}[R_{MLN}(\{p_i(X) + \epsilon_i\}_{i \in [n]})]\nonumber =& \max_{\{|\epsilon_i < C_i|\}}\ln\left(\frac{Z_1(\{p_i(X)+\epsilon_i \}_{i\in [n]})}{Z_2(\{p_i(X)+\epsilon_i \}_{i\in [n]})}\right) \nonumber\\
=&\max_{\{\epsilon_i\}, \{\epsilon_i'\}} \ln Z_1(\{p_i(X)+\epsilon_i \}_{i\in [n]})\\ &- \ln Z_2(\{p_i(X)+\epsilon_i' \}_{i\in [n]})\nonumber\\
& s.t.,  \epsilon_i = \epsilon_i', ~~~~ 
|\epsilon_i|,|\epsilon_i'|  \leq C_i.
\end{align*}\end{small}
Introducing Lagrange multipliers $\{\lambda_i\}$. Note that any choice of $\{\lambda_i\}$ corresponds to a valid upper bound. Thus $\forall \lambda_1, ...,\lambda_n \in \mathbb{R}$, we can reformulate the above into
\begin{small}\begin{align*}
\max_{\{|\epsilon_i < C_i|\}}\ln\mathbb{E}[R_{MLN}(\{p_i(X) + \epsilon_i\}_{i \in [n]})]\nonumber
&\le \max_{\{\epsilon_i\}, \{\epsilon_i'\}} 
\ln Z_1(\{p_i(X)+\epsilon_i \}_{i\in [n]}) \\
&- \ln Z_2(\{p_i(X)+\epsilon_i' \}_{i\in [n]})\\
&+\sum_i\lambda_i (\epsilon_i - \epsilon_i'),\nonumber\\ 
&s.t., |\epsilon_i|, |\epsilon_i'| \leq C_i.
\end{align*}\end{small}
Define
\begin{align*}
&\widetilde{Z_1}(\{\epsilon_i\}_{i \in [n]}) = \ln Z_1(\{p_i(X)+\epsilon_i\}_{i \in [n]}) + \sum_i \lambda_i \epsilon_i;\nonumber\\
&\widetilde{Z_2}(\{\epsilon_i'\}_{i \in [n]}) = \ln Z_1(\{p_i(X)+\epsilon'_i\}_{i \in [n]}) + \sum_i \lambda_i \epsilon'_i,
\end{align*}
We have the claimed upper-bound,
\begin{align*}
\max_{\{|\epsilon_i < C_i|\}} \ln \mathbb{E}[R_{MLN}(\{p_i(X) + \epsilon_i\}_{i \in [n]})] 
\leq & \max_{\{|\epsilon_i|<C_i\}}
\widetilde{Z_1}(\{\epsilon_i\}_{i \in [n]}) \\
 &-
\min_{\{|\epsilon_i'|<C_i\}}
\widetilde{Z_2}(\{\epsilon_i'\}_{i \in [n]}). 	
\end{align*}
Similarly, the lower-bound can be written in terms of Lagrange multipliers, and $\forall \lambda_1, ...,\lambda_n \in \mathbb{R}$, we have
\begin{small}
\begin{align*}
\min_{\{|\epsilon_i < C_i|\}}\ln\mathbb{E}[R_{MLN}(\{p_i(X) + \epsilon_i\}_{i \in [n]})]\nonumber \ge &\min_{\{\epsilon_i\}, \{\epsilon_i'\}} 
\ln Z_1(\{p_i(X)+\epsilon_i \}_{i\in [n]}) \\ 
&- \ln Z_2(\{p_i(X)+\epsilon_i' \}_{i\in [n]})\\
&+\sum_i\lambda_i (\epsilon_i - \epsilon_i'),\nonumber\\ 
&~~~~s.t., |\epsilon_i|, |\epsilon_i'| \leq C_i.
\end{align*}\end{small}
Hence we have the claimed lower-bound,
\begin{small}
\begin{align*}
\min_{\{|\epsilon_i| < C_i\}} \ln \mathbb{E}[R_{MLN}(\{p_i(X) + \epsilon_i\}_{i \in [n]})] 
\geq & \min_{\{|\epsilon_i|<C_i\}} 
\widetilde{Z_1}(\{\epsilon_i\}_{i \in [n]})\\ %\\
 & -
\max_{\{|\epsilon_i'|<C_i\}}
\widetilde{Z_2}(\{\epsilon_i'\}_{i \in [n]}). 	
\end{align*}\end{small}
\end{proof}
\section{Supplementary Results for Algorithm 1}
\label{apx:sup_alg1}

\newtheorem*{proposition*}{Proposition}

\begin{proposition*}[Monotonicity]
When $\lambda_i \geq 0$,
$\widetilde{Z_r}(\{\epsilon_i\}_{i \in [n]})$
monotonically increases w.r.t. $\epsilon_i$; 
When $\lambda_i \leq -1$, $\widetilde{Z_r}(\{\epsilon_i\}_{i \in [n]})$
monotonically decreases w.r.t. $\epsilon_i$.
\end{proposition*}
\begin{proof}
Recall that by definition we have
\begin{align*}
& Z_r(\{p_i(X)+\epsilon_i \}_{i \in [n]}) 
= \\
&\sum_{\sigma \in \mathcal{I}_r}\exp\left\{ 
\sum_{G_i \in \mathcal{G}} w_{G_i}(p_i(X)+\epsilon_i ) \sigma(x_i) +
\sum_{H \in \mathcal{H}} w_{H} f_{H}(\sigma(\bar{\bf v}_H))
\right\}
\end{align*}
where $w_{G_i}(p_i(x)) = \log [p_i(X) / (1 - p_i(X))]$ and $\mathcal{I}_1=\Sigma \wedge \{\sigma(v) = 1\}$ and $\mathcal{I}_2=\Sigma$.
We can rewrite the perturbation on $p_i(X)$ as a perturbation on $w_{G_i}$:
$
	 w_{G_i}(p_i(X)+\epsilon_i) = w_{G_i} + \tilde{\epsilon}_i,
$

where 
\begin{align*}
\tilde{\epsilon}_i=\log\left[\frac{(1-p_i(X))(p_i(X)+\epsilon_i)}{p_i(X)(1-p_i(X)-\epsilon_i)}\right].	
\end{align*} 
Note that $\tilde{\epsilon}_i$ is monatomic in $\epsilon_i$. We also have
\begin{align*}
 \ln \mathbb{E}[R_{MLN}(\{p_i(X) + \epsilon_i\}_{i \in [n]})] =  \ln \mathbb{E}[R_{MLN}(\{w_{G_i}(X) + \tilde{\epsilon}_i\}_{i \in [n]})] 	
\end{align*}
We can hence apply the same Lagrange multiplier procedure as in the above proof of Lemma 6 and conclude that 
\begin{align*}
\widetilde{Z_r}(\{\epsilon_i\}_{i \in [n]}) :=& \ln Z_r(\{p_i(X)+\epsilon_i\}_{i \in [n]}) + \sum_i \lambda_i \epsilon_i	\nonumber \\
=& \ln Z_r(\{w_{G_i}(X)+ \tilde{\epsilon}_i\}_{i \in [n]}) + \sum_i \lambda_i \tilde{\epsilon}_i,
\end{align*}
where $\epsilon_i\in[-C_i,C_i]$ $\tilde{\epsilon}_i\in [-C_i',C_i']$ with $C_i'=\log\left[\frac{(1-p_i(X))(p_i(X)+C_i)}{p_i(X)(1-p_i(X)-C_i)}\right]$.
We are now in the position to rewrite $\widetilde{Z_r}$ as a function of $\tilde{\epsilon}_i$ and obtain
\begin{small}
\begin{align*}
&\widetilde{Z_r}(\{\epsilon_i'\}_{i \in [n]}) \\
=&\ln\sum_{\sigma \in \mathcal{I}_r}\exp\left\{ 
\sum_{G_i \in \mathcal{G}} (w_{G_i}+\tilde{\epsilon}_i)\sigma(x_i) +
\sum_{H \in \mathcal{H}} w_{H} f_{H}(\sigma(\bar{\bf v}_H))
\right\} + \sum_i\lambda_i\tilde{\epsilon}_i\nonumber\\
=&  \ln\sum_{\sigma \in \mathcal{I}_r}\exp\left\{ 
\sum_{G_i \in \mathcal{G}} w_{G_i}\sigma(x_i) +\sum_i(\sigma(x_i)+\lambda_i) \tilde{\epsilon}_i+
\sum_{H \in \mathcal{H}} w_{H} f_{H}(\sigma(\bar{\bf v}_H))
\right\}
\end{align*}\end{small}
Since $\sigma(x_i)\in\{0,1\}$, when $\lambda_i\ge 0$, $\sigma(x_i)+\lambda_i\ge 0$ and $\widetilde{Z_r}$ monotonically increases in $\tilde{\epsilon}_i$ and hence in $\epsilon_i$. When $\lambda_i\le -1$, $\sigma(x_i)+\lambda_i\le 0$ and $\widetilde{Z_r}$ monotonically decreases in $\tilde{\epsilon}_i$ and hence in $\epsilon_i$.
\end{proof}
\begin{proposition*}[Convexity]
$\widetilde{Z_r}(\{\tilde{\epsilon}_i\}_{i \in [n]})$ is a convex function in $\tilde{\epsilon}_i, \forall i$ with
$$\tilde{\epsilon}_i=\log\left[\frac{(1-p_i(X))(p_i(X)+\epsilon_i)}{p_i(X)(1-p_i(X)-\epsilon_i)}\right].$$
\end{proposition*}
\begin{proof}
We take the second derivative of $\widetilde{Z_r}$ with respect to $\tilde{\epsilon}_i$,
\begin{tiny}
\begin{align*}
&\frac{\partial^2\widetilde{Z_r}}{\partial\epsilon_1^2}
=\\
&\frac{\sum_{\sigma \in \mathcal{I}_r}(\sigma(x_i)+\lambda_i )^2\exp\left\{ 
\sum_{G_j \in \mathcal{G}} w_{G_j}\sigma(x_j) +\sum_j(\sigma(x_j)+\lambda_j) \tilde{\epsilon}_j+
\sum_{H \in \mathcal{H}} w_{H} f_{H}(\sigma(\bar{\bf v}_H))
\right\}}{\sum_{\sigma \in \mathcal{I}_r}\exp\left\{ 
\sum_{G_j \in \mathcal{G}} w_{G_j}\sigma(x_j) +\sum_j(\sigma(x_j)+\lambda_j) \tilde{\epsilon}_j+
\sum_{H \in \mathcal{H}} w_{H} f_{H}(\sigma(\bar{\bf v}_H))
\right\}}\nonumber\\
&-\\
&\left[\frac{\sum_{\sigma \in \mathcal{I}_r}(\sigma(x_i)+\lambda_i )\exp\left\{ 
\sum_{G_j \in \mathcal{G}} w_{G_j}\sigma(x_j) +\sum_j(\sigma(x_j)+\lambda_j) \tilde{\epsilon}_j+
\sum_{H \in \mathcal{H}} w_{H} f_{H}(\sigma(\bar{\bf v}_H))
\right\}}{\sum_{\sigma \in \mathcal{I}_r}\exp\left\{ 
\sum_{G_j \in \mathcal{G}} w_{G_j}\sigma(x_j) +\sum_j(\sigma(x_j)+\lambda_j) \tilde{\epsilon}_j+
\sum_{H \in \mathcal{H}} w_{H} f_{H}(\sigma(\bar{\bf v}_H))
\right\}}\right]^2.
\end{align*}
\end{tiny}
The above is simply the variance of $\sigma(x_i)+\lambda_i$, namely $\mathbb{E}\left[(\sigma(x_i)+\lambda_i)^2\right]-\mathbb{E}\left[\sigma(x_i)+\lambda_i\right]^2\ge 0$. The convexity of $\widetilde{Z_r}$ in $\tilde{\epsilon}_i$ follows.
\end{proof}

\section{Image Classification on Road Sign Dataset}
\label{adx:road-sign-section}

All the experiments shown in Appendix~\ref{adx:road-sign-section} -  \ref{sec:apx-partial-knowleddge} are run on 4 RTX 2080 Ti GPUs.

\noindent\textbf{Task and Dataset.} For road sign classification task, the whole dataset can be viewed as a subset of GTSRB dataset~\cite{stallkamp2012man}, which contains 12 types of German road signs \{"Stop'', "Priority Road'', "Yield'', "Construction Area'', "Keep Right'', "Turn Left'', "Do not Enter'', "No Vihicles'', "Speed Limit 20'', "Speed Limit 50'', "Speed Limit 120'', "End of Previous Limitation''\}, with 14880 training samples, 972 validation samples and 3888 testing samples in total. Besides the road sign classes, we construct $13$ attribute classes as follows:
\begin{itemize}
    \item Border shape classes: "Octagon'', "Square'', "Triangle'', "Circle''.
    \item Border color classes: "Red'', "Blue'', "Black''.
    \item Digit classes: "Digit 20'', "Digit 50'', "Digit 120''.
    \item Content classes: "Left'', "Right'', "Blank''.
\end{itemize}

Based on the indication direction from road sign classes to attribute classes, and the exclusive relationship between attribute classes with the same type, we develop the following two types of knowledge rules as follows:

\begin{itemize}
    \item \underline{Indication rules} $(u, v)$: Road sign class $u$ indicates attribute $v$.
    \item \underline{Exclusion rules} $(u, v)$: Attribute classes $u$ and $v$ with the same type ("Shape'', "Color'', "Digit'' or "Content'') are naturally exclusive. (e.g., One road sign can not have "Octagon'' shape and "Triangle'' shape at the same time.)
\end{itemize}

%We view the main task model as a traditional feed-forward network while the value of its normalized output vector on each dimension can be regarded as the probability of independent binary variables. Each attribute sensor, which is a binary classifier, also outputs the probability of each attribute. After that, we use MLN as our reasoning component integrating the outputs of the main task model and attribute sensors by using knowledge rules.

%\noindent\textbf{Implementation details.} We adopt the GTSRB-CNN~\cite{xxx} as our main task model's architecture. For each attribute sensor, we still use the GTSRB-CNN architecture but slightly modify the output dimension to be 2 as a binary classifier. We use the raw training set for the main task model's training and construct the corresponding binary classification training set for each attribute sensor with the rule \emph{"If this image contains such attribute?''}. We apply isotropic Gaussian augmentation~$\epsilon \sim \mathcal{N}(0, \sigma^2I)$ during the training, and smooth our models by adopting the randomized smoothing scheme~\cite{xxx} to obtain the certified robustness guarantee of the sensing model's output confidence under the $\ell_2$-norm bounded perturbation. Detailed training parameters are left in Appendix xxx. %Specifically, we smoothed our model by adding the isotropic Gaussian noise $\epsilon \sim \mathcal{N}(0, \sigma^2I)$ to the training images with and tune on the performance on validation set to avoid overfitting

\noindent\textbf{Knowledge.} We construct our first-order logical rules based on our predefined indication and exclusion knowledge as follows:

\begin{itemize}
    \item \underline{Indication edge} $u \implies v$: if one object belongs to road sign class $u$, it should have attribute $u$:
    \begin{align}
         x_u \wedge \neg x_v = \texttt{False}
    \end{align}
    \item \underline{Exclusion edge} $u \oplus v$: On object can not have attribute $u$ and $v$ at the same time:
        \begin{align}
        x_u \land x_v = \texttt{False}
    \end{align}
\end{itemize}

\textbf{Intuitive Example.} Following the HEX graph-based knowledge structure and rules, we will show several adversary scenarios which could be mitigated through the inference reasoning phase. For instance, if the ``Construction Area" object is attacked to be ``Stop Sign" while other sensing nodes remain unaffected, like the border shape is still detected as the ``Triangle'' shape. Then the indication knowledge rule (The ``Stop Sign" object should have the ``Octagon'' border shape) and the exclusive knowledge rule (No class can have the ``Triangle'' border shape and ``Octagon'' shape at the same time) would be violated. Such violation of the knowledge rules would discourage our pipeline to predict ``Stop Sign" as what the attacker wants. However, the \sr may not distinguish the ``Yield", and ``Construction Area" classes if the attacker fooled the ``Construction Area" sensing completely, which shows the limitation of such structural knowledge, and more knowledge would be required in this case to help improve the robustness.

\newpage
\begin{figure*}[!t]
\centering
\includegraphics[width=0.8\textwidth]{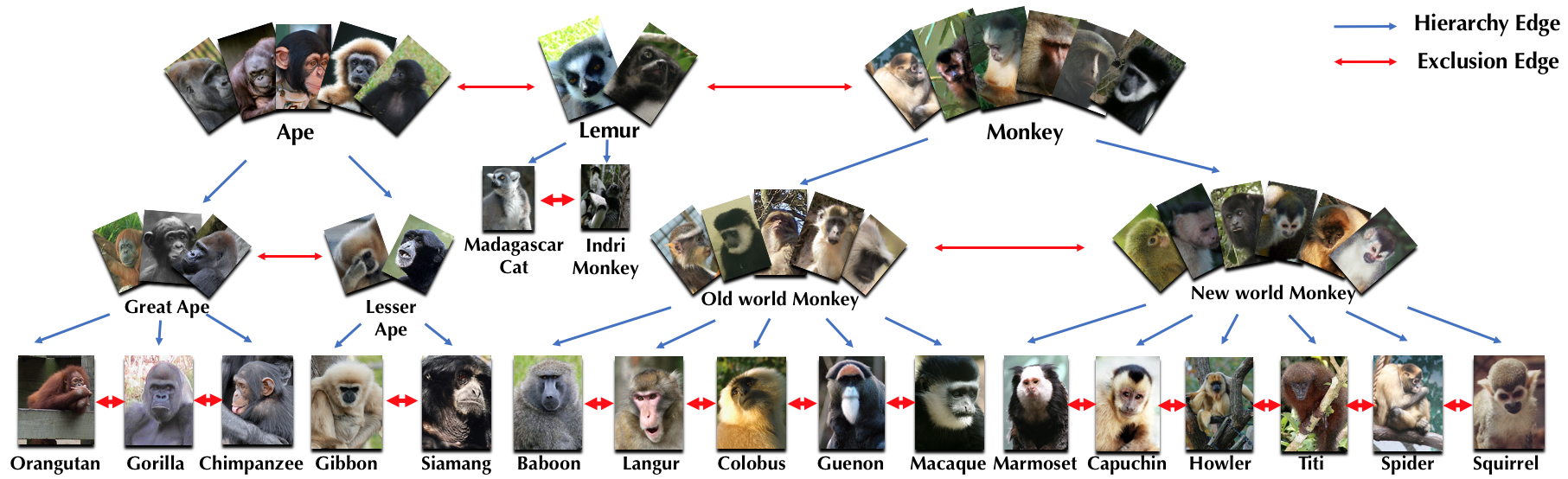}
\caption{\small {\bf PrimateNet}. The knowledge structure of PrimateNet dataset. The \textcolor{blue}{\textbf{Blue}} arrows represent the Hierarchical rules between different classes, and the \textcolor{red}{\textbf{Red}} arrows  represent the Exclusive rules. Some exclusive rules are omitted due to the space limit.}
\label{fig:hex}
\vspace{-1em}
\end{figure*}

\section{Information Extraction on Stock News}
\label{apx:stock_news}
% \vspace{-2mm}
To further evaluate the certified robustness of the reasoning component, in this section we focus on the setting where perturbations can be directly added to the output of sensors (i.e., input of the reasoning component), using information extraction task as an example.

\textbf{Tasks and Dataset.} We consider information extraction (or relation extraction) tasks in NLP based on the stock news dataset --- HighTech dataset, which consists of both daily closing asset price and financial news from \textsl{2006} to \textsl{2013}~\cite{ding2014using}. We choose 9 companies with the highest volume of news coverage, resulting in 4810 articles related to 9 specific stocks that are filtered by company name. We split the stock news dataset into training and testing days chronologically. 
Here we define three information extraction tasks as sensing models: \texttt{StockPrice(Day, Company, Price)}, \texttt{StockPriceChange(Day, Company, Percent)}, \texttt{StockPriceGain(Day, Company)}.
% We leave detailed definitions of
% these relations to Appendix.
The domain knowledge that we incorporate illustrates the connections between these sensing models.
We provide detailed descriptions of these tasks as below.

\begin{itemize}[wide = 0pt]
\item \textbf{StockPrice(day, company, price)} In this task, we extract the daily closing price of a company's stock from the news articles. To achieve this, we first extract numbers in each sentence of the article as candidate relations. We then label each relation based on the given daily closing asset price as follows: we label the relation whose number starts with "\$" and has the smallest difference with the given closing price as positive, and label all others as negative. With these labeled samples, We train a BERT-based binary classifier \cite{bert} as our sensing model to identify whether the given number is the closing price of the stock and output the prediction confidence. The input form of our BERT sensor is ``[CLS] sentence [SEP] company [SEP] price''.

\item \textbf{StockPriceChange(day, company, percentage)} In this task, we extract the percentage change in a company's stock closing price from a collection of news articles. To achieve this, we first extract numbers in each sentence of the article as candidate relations, and label each relation using the percentage change of the closing asset price from the previous day and the current day. With these labeled samples, we train a BERT-based binary classifier as our sensing model to identify whether the given percent number represents the change rate of the stock on a given day, and output the prediction confidence. The input form of our BERT sensor is ``[CLS] sentence [SEP] company [SEP] percentage''.

\item \textbf{StockPriceGain(day, company, gain)} In this task, we extract the information about whether the closing price of a company's stock rose or fell on a particular day given the news articles. We identify each sentence containing a stock name and numbers starting with "\$" as a candidate relation, and assess each relation by counting the number of positive and negative words in the sentence to determine whether it indicates a rise or fall in the stock price. Specifically, we label one relation as positive 
when \textit{Count(positive word) \textgreater\; Count(negative words)};
and negative 
when \textit{Count(positive word) \textless \; Count(negative words)}. 
With these labeled samples, we train a BERT-based binary classifier as our sensing model to output the confidence. The input form of our BERT sensor is ``[CLS] sentence [SEP] company''.

%\item \textbf{StockPriceGain(day, company, gain)} In this task, we want to extract information about whether the closing price of the stock rose or fell on the day based on the news article. We treat each sentence with the stock name and the numbers which start with "\$" as a candidate. 
%Then we judge each relationship by whether it indicates the stock price rose or fell by counting the positive and negative words in the sentence. 
%We label the relation as positive:
%when \textit{Count(positive word) \textgreater \; Count(negative words)};
%and negative: 
%when \textit{Count(positive word) \textless \; Count(negative words)}. 
%We train a BERT-base classifier as the sensing model and output the confidence. % The input form of the BERT classifier is as "[CLS] sentence [SEP] company".
\end{itemize}

\textbf{Implementation Details.} We employ one BERT-based binary classifier for each information extraction task.  To train the BERT classifiers, we adopt the final hidden state of the first token [CLS] from BERT as the representation of the whole input, and apply dropout with probability $p=0.5$ on this final hidden state. We add a fully connected layer on top of BERT for classification. To fine-tune the BERT classifiers for the three information extraction tasks, we adopt the Adam optimizer with a learning rate of $10^{-5}$, weight decay of $10^{-4}$, and train our classifiers for $30$ epochs with the batch size as $32$. 

\textbf{Knowledge.} For each news article that pertains to a specific company, we define $d$ as the current date,  $p_1$ as the stock price on the current date, and $p_0$ as the stock price on the previous day. We also use $y$ to present the binary variable that indicates whether the company's stock price rose or fell ($y=0$ for ``fell'' and $y=1$ for ``rose''), and  $\beta$ the percentage change of stock price change. We define the following knowledge rules that the extracted tuple $(p_0, p_1, y, \beta)$ 
should satisfy:
\begin{itemize}[leftmargin=*]
    \item \textbf{Rule 1}: The extracted stock price $p_0$ and $p_1$ (from the \texttt{StockPrice} sensor) should be consistent with the stock price change indicator $y$ (from the \texttt{StockPriceGain} sensor): 
    \begin{align}
        y = \mathbbm{I}[p_1 - p_0 > 0]  
    \end{align}
    \item \textbf{Rule 2}: The extracted stock price $p_0$ and $p_1$ (from the \texttt{StockPrice} sensor) should approximately follow the percentage change of stock price $\beta$ (from the \texttt{StockPriceChange} sensor):
    \begin{align}
    p_1  \approx p_0 \times [1 + (-1)^{\mathbbm{I}[p_1 - p_0 > 0]} \times \beta]
\end{align}
\end{itemize}

\textbf{Threat Model.}  Given the perturbation bound $C_S$, we attack our sensing-reasoning pipeline by adding perturbation $\delta$ ($\delta \leq C_S$) on perturbed sensing model’s confidence value $p$ directly. The perturbed confidence value, named as $p'$, should satisfy $p' \in [p - C_S, p + C_S]$. In the threat model we consider, attackers can add perturbations to all sensing models' output prediction confidence.

\vspace{1mm}

% \begin{wraptable}{r}{0.5\textwidth}
% \centering
% \vspace{-2em}
% \small
% \caption{\small \textbf{(NLP)} Certified Robustness and Certified Ratio for approaches when all sensing models are attacked.}
% \scalebox{0.8}{
% \begin{tabular}{c|cc|cc}
% \toprule
%     & \multicolumn{2}{c|}{\bf With knowledge}                              & \multicolumn{2}{c}{\bf Without knowledge}                           \\ \hline
% $C_S$ & Cert. Robustness  & Cert. Ratio & Cert. Robustness  & Cert. Ratio \\ \hline
% \hline
% % 0.0 & \textbf{1.0000} & \textbf{1.0000} & 0.9133 & 0.9133              \\ 
% 0.1 & \textbf{1.0000} & \textbf{1.0000} & 0.9969 & 0.9969              \\ 
% 0.5 & \textbf{1.0000} & \textbf{1.0000} & 0.9474 & 0.9474                   \\ 
% 0.9 & \textbf{0.5882} & \textbf{0.5882} & 0.3839 & 0.3839            \\ \bottomrule
% \end{tabular}}
% \label{exp:newnlp1}
% \vspace{-1.2em}
% \end{wraptable}

% \textbf{Evaluation Metrics.} We train the above three BERT sensing models to perform stock information extraction, and connect the sensing models with an MLN based on the knowledge above. What we focus on is the accuracy of \texttt{StockPrice}: whether the extracted relation with top-$1$ confidence value is the ground truth. We use the same evaluation metrics \textbf{Certified Robustness} and \textbf{Certified Ratio} defined before to evaluate the robust accuracy of our  pipeline.

\vspace{1mm}
\textbf{Intuitive Example.} 
Here we show an intuitive example of how our knowledge can help improve the prediction robustness of these information extraction models under adversarial attacks. Assume  our sensing models extract the correct stock price information $(p^*_0, p^*_1, y^*, \beta^*)$, where price $p^*_0 > p^*_1$ and the stock price change is ``fell" ($y=0$) by $\beta\%$. Now if the first stock price extraction sensor is attacked to output an incorrect prediction $p'_0$ such that $p'_0 < p^*_1$ while other sensors remain intact; $p'_0$ 
will violate our knowledge rules 1 and 2. Specifically, the stock price change $p'_0 - p^*_1 < 0$ is inconsistent with stock price change prediction $y = 0$, i.e., $p'_0 - p^*_1 > 0$. As a result, our reasoning component will reduce the confidence of the wrong prediction $p'_0$ and increase the confidence of the ground truth $p^*_0$ to be consistent with the knowledge rules; and therefore potentially recover the correct prediction of $p^*_0$. 

\section{Image Classification on PrimateNet Dataset}
\label{adx:primate-section}
\textbf{Task and Dataset.} We aim to evaluate the certified robustness of our \sr on large-scale dataset such as ImageNet ILSVRC2012~\cite{imagenet_cvpr09}. In particular, to obtain domain knowledge for the images, we select 18 Primate animal categories to form a PrimateNet dataset, containing \{Orangutan, Gorilla, Chimpanzee, Gibbon, Siamang, Madagascar cat, Woolly indris, Guenon, Baboon, Macaque, Langur, Colobus, Marmosets, Capuchin monkey, Howler monkey, Titi monkey, Spider monkey, Squirrel monkey\}. Moreover, we create 7 internal classes \{Greater ape, Lesser ape, Ape, Lemur, Old-world monkey, New-world monkey, Monkey\} to construct the hierarchical structure according to the WordNet~\cite{fellbaum2012wordnet}. With such a hierarchical structure, we can build the Primate-class Hierarchy and Exclusion(HEX) graph based on the concepts from \cite{deng2014large} as shown in Fig~\ref{fig:hex}. Within the HEX graph, we develop two types of knowledge rules described as follows:
\begin{itemize}
    \item \underline{Hierarchy rules} $(u, v)$: class $u$ subsumes class $v$ (e.g. Great Ape subsumes Gorilla);
    \item \underline{Exclusion edge} $(u, v)$: class $u$ and class $v$ are naturally exclusive (e.g. Gorilla cannot belong to Great Ape and Lesser Ape at the same time).
\end{itemize}

We consider each class in the HEX graph as the prediction of one sensing model in the \sr, and 
we construct 25 sensing models as the leaf and internal nodes in the HEX graph. Here we use the MLN as our reasoning component connecting to these sensing models.

\textbf{Implementation details.} For each leaf sensing model, we utilize 1300 images from the ILSVRC2012 training set and 50 images from the ILSVRC2012 dev set. We split the 1300 images into 1000 images for training and 300 for testing. For each internal node, we uniformly sample the training images from all its children nodes' training images to form its training set with the same size   1300, since there are no specific instances belonging to internal nodes' categories in PrimateNet. 

During training, we utilize the sensing DNN model for each node in the knowledge hierarchy to output the probability value given the input images. The models consist of a pre-trained ResNet18 feature extractor concatenated by two Fully-Connected layers with ReLU activation. In order to provide the certified robustness of the end-to-end \sr, we adapt the randomized smoothing strategy mentioned in~\cite{cohen2019certified} to certify the robustness of sensing models, and then compose it with the certified robustness of the reasoning component. Specifically, we smoothed our sensing models by adding the isotropic Gaussian noise $\epsilon \sim \mathcal{N}(0, \sigma^2I)$ to the training images during training. We train each sensing model for 80 epochs with the Adam optimizer (initial learning rate is set to $2 \times 10^{-4}$) and evaluate the sensing models' performance on the validation set containing 50 images after every training epoch to avoid over-fitting. During testing, we certify the robustness of trained sensing models with the same smoothing parameter $\sigma$ used during model training.

\textbf{Knowledge.} The knowledge used in this task includes the \textit{hierarchical} and \textit{exclusive} relationships between different categories of the sensing predictions. For instance, the category ``Ape" would include all the instances classified as ``Greater ape, Lesser ape" (hierarchical); and there should not be any intersection for instances predicted as ``Monkey" or ``Lemur" (exclusive). Thus, we build our knowledge rules based on the structural relationships such as hierarchy and exclusion knowledge:

\begin{itemize}[leftmargin=*]
    \item \textbf{Hierarchy edge} $u\implies v$: If one object belongs to class $u$, it should belong to class $v$ as well:
    \begin{align}
         x_u \wedge \neg x_v = \texttt{False}
    \end{align}
    \item \textbf{Exclusion edge} $u \oplus v$: One object should not belong to class $u$ and class $v$ at the same time:
    \begin{align}
        x_u \land x_v = \texttt{False}
    \end{align}
\end{itemize}

In our experiments, we have $22$ hierarchy edges and $35$ exclusive edges as our knowledge rules, build upon $25$ sensing models in total.

%ere we aim to evaluate the certified robustness bound of our Sensing-Reasoning pipeline on such large-scale dataset with structural knowledge. 
% the data samples whose top-1 choices given by the OCR engine in all positions are correct to make sure that the perturbation induces actual ``attacks".
% When $\alpha=1$ which means that all the sensing models could be attacked, we select equations with length smaller or equal to 5. We add random noise to these instances to simulate the real-world scenarios and obtain 105 data samples in total.

%After constructing the HEX graph, we build our Sensing-Reasoning pipeline with 25 sensing models on each MLN node, following the rules of the Hierarchy edges and Exclusion edges. 
%For each leaf node, we select 1300 images from ImageNet~\cite{imagenet_cvpr09} for training and testing. For each internal node, we uniformly sample the training data from all its child nodes' training set to form its training set with the same size, since there's no specific instances belongs to internal nodes' categories in PrimateNet. More details please refer to Appendix.

\textbf{Threat Model.} 
% We formulate our framework as a joint or distributed classification system and each sensing model has been placed remotely. 
In this paper, we consider a strong attacker who has access to perturbing several sensing models' input instances during inference time. To perform the attack, the attacker will add perturbation $\delta$, bounded by $C_I$ under the $\ell_2$ norm, onto the test instance against the victim sensing models: $||\delta||_2 < C_I$.
In particular, we consider the attacker to attack $\alpha$ percent of the total sensing models.

Since we apply randomized smoothing  to sensing models during  training, for each sensing model, we can certify
the output probability $p'$ as a function of 
the original confidence $p$, the bound of the perturbation $C_I$, and smoothing parameter $\sigma$ according to Corollary 2 as below:
$$
    p' \in [\Phi(\Phi^{-1}(p)-C_I/\sigma), \Phi(\Phi^{-1}(p)+C_I/\sigma)].
$$

\textbf{Evaluation Metrics.} 
To evaluate the certified robustness of \sr, we focus on the standard \textit{certified accuracy} on a given test set, and the \textit{certified ratio} measuring the percentage of instances that could be certified within a certain perturbation magnitude/radius.

Based on the previous analysis, given the $\ell_2$ based perturbation bound $C_I$, we can certify the output probability of the \sr as $[\mathcal{L}, \mathcal{U}]$. In order to evaluate the certified robustness of \sr, we define the \textbf{Certified Robustness}, measuring the percentage of instances that could be certified to make correct prediction within a perturbation radius, to evaluate the certified robustness following existing work~\cite{cohen2019certified}, which is formally defined as:
$$ \frac{\sum_{i=1}^{N}\mathbbm{I}(([\mathcal{U}_i < 0.5]\land [y_i = 0])\lor
([\mathcal{L}_i > 0.5]\land [y_i = 1]))
}{N},$$
where $N$ refers to the number instances and $y_i$ the ground truth label of the given instance $i$. $\mathbb{I}(\cdot)$ is an indicator function which outputs 1 when its argument takes value true and 0 otherwise.

Moreover, we report the \textbf{Certified Ratio} to measure the percentage of instances that could be certified as a consistent 
prediction
within a perturbation radius (even the consistent prediction might be wrong). The Certified Ratio is defined as:

$$ \frac{\sum_{i=1}^{N}\mathbbm{I}([\mathcal{U}_i < 0.5]\lor
[\mathcal{L}_i > 0.5])
}{N}.$$

Here the lower and upper bounds of the output probability $\mathcal{L}_i$ and $\mathcal{U}_i$ indicate the binary prediction of each sensing model. We assume when the output probability is less than 0.5, it outputs 0.

%We report the \emph{End-to-end accuracy} measuring the percentage of instances could be predicted correctly under the adversary scenario. To evaluate the certified robustness of different machine learning approaches, we calculate the \emph{certified ratio} following the standard setup~\cite{cohen2019certified}, which measure the percentage of instances that are guaranteed to be certified to specific class under the certain perturbation $C_I$ and attack ratio $\alpha$. 

\vspace{1mm}
\textbf{Intuitive Example.} Following the HEX graph-based knowledge structure and rules, we will show several adversary scenarios which could be mitigated through the inference reasoning phase. For instance, based on Figure~\ref{fig:hex}, if one ``Gorilla" object is attacked to be ``Siamang" while other sensing nodes remain unaffected, the hierarchical knowledge rule (An object belongs to ``Great Ape" class cannot belong to ``Siamang" class) and the exclusive knowledge rule (No object could belong to ``Great Ape" and ``Siamang" classes at the same time) would be violated. Such violation of the knowledge rules would discourage our pipeline to predicting ``Siamang" as what the attacker wants. However, the \sr may not distinguish the ``Orangutan", ``Gorilla", and ``Chimpanzee" classes if the attacker fooled the ``Gorilla" sensing completely, which shows the limitation of such structural knowledge, and more knowledge would be required in this case to help improve the robustness.

\begin{table}[t]
\centering
\vspace{-1em}
\small
\caption{\small  \emph{Benign} accuracy (i.e. $C_I=0,\alpha =0$) of models with and without knowledge under different smoothing parameters $\sigma$ evaluated on PrimateNet.}
\begin{tabular}{c|c|c}
\toprule

$\sigma$ & \textbf{With knowledge}  & \textbf{Without knowledge} \\ \hline
\hline
0.12 & \textbf{0.9670}                & 0.9638             \\ 
0.25 & \textbf{0.9612}               & 0.9554                        \\ 
0.50 & \textbf{0.9435}               & 0.9371                   \\ \bottomrule
\end{tabular}
\label{tab:benign-imagenet}
\vspace{-2em}
\end{table}

\begin{table}
\centering
% \scriptsize
\caption{\small Certified Robustness and Certified Ratio under different perturbation magnitude $C_I$ and sensing model attack ratio $\alpha$ on PrimateNet.  The sensing models are smoothed with Gaussian noise $\epsilon \sim \mathcal{N}(0, \hat{\sigma}^2 I_d)$ with different smoothing parameter $\sigma$. }
\label{tab:attack-imagenet-smooth-main}

{ \bf (a) $\hat{\sigma} = 0.12$}

\vspace{0.3em}

\scalebox{0.7}{
\begin{tabular}{c|c|cc|cc}
\toprule
\multicolumn{2}{c|}{}            & \multicolumn{2}{c|}{\bf With knowledge} & \multicolumn{2}{c}{\bf Without knowledge} \\ \hline
$C_I$                & $\alpha$ & Cert. Robustness      & Cert. Ratio     & Cert. Robustness       & Cert. Ratio       \\ \hline \hline
\multirow{4}{*}{0.12} & 10\%     & \textbf{0.8849}            & \textbf{0.9419}          & 0.5724             & 0.5724            \\
                      & 20\%     & \textbf{0.8078}            & \textbf{0.8609}          & 0.5717             & 0.5717            \\
                      & 30\%     & \textbf{0.7508}            & \textbf{0.7988}          & 0.5706             & 0.5706            \\
                      & 50\%     & \textbf{0.6236}            & \textbf{0.6647}          & 0.5706             & 0.5706            \\ \hline
\multirow{4}{*}{0.25} & 10\%     & \textbf{0.7888}            & \textbf{0.8428}          & 0.2342             & 0.2342            \\
                      & 20\%     & \textbf{0.6226}            & \textbf{0.6657}          & 0.2320             & 0.2320            \\
                      & 30\%     & \textbf{0.5225}            & \textbf{0.5596}          & 0.2309             & 0.2309            \\
                      & 50\%     & \textbf{0.3594}            & \textbf{0.3824}          & 0.2268             & 0.2268            \\ \bottomrule
\end{tabular}}

\vspace{1em}

{ \bf (b)  $\hat{\sigma} = 0.25$}\\
 \vspace{0.3em}
 \scalebox{0.7}{
\begin{tabular}{c|c|cc|cc}
\toprule
\multicolumn{2}{c|}{}            & \multicolumn{2}{c|}{\textbf{With knowledge}} & \multicolumn{2}{c}{\textbf{Without knowledge}} \\ \hline
$C_I$                & $\alpha$ & Cert. Robustness          & Cert. Ratio          & Cert. Robustness            & Cert. Ratio           \\ \hline \hline
\multirow{4}{*}{0.25} & 10\%     & \textbf{0.8498}       & \textbf{0.9499}      & 0.5314                  & 0.5314                \\
                      & 20\%     & \textbf{0.7608}       & \textbf{0.8952}      & 0.5302                  & 0.5302                \\
                      & 30\%     & \textbf{0.7217}       & \textbf{0.8048}      & 0.5294                  & 0.5294                \\
                      & 50\%     & \textbf{0.6026}       & \textbf{0.6747}      & 0.5235                  & 0.5235                \\ \hline
\multirow{4}{*}{0.50} & 10\%     & \textbf{0.7622}       & \textbf{0.8489}      & 0.2024                  & 0.2024                \\
                      & 20\%     & \textbf{0.5988}       & \textbf{0.6467}      & 0.2024                  & 0.2024                \\
                      & 30\%     & \textbf{0.5324}       & \textbf{0.5541}      & 0.2010                  & 0.2010                \\
                      & 50\%     & \textbf{0.3417}       & \textbf{0.3615}      & 0.2000                  & 0.2000                \\ \bottomrule
\end{tabular}}

 \vspace{1em}

{ \bf (c)  $\hat{\sigma} = 0.50$}\\
 \vspace{0.3em}
 \scalebox{0.7}{
\begin{tabular}{c|c|cc|cc}
\toprule
\multicolumn{2}{c|}{}            & \multicolumn{2}{c|}{\textbf{With knowledge}} & \multicolumn{2}{c}{\textbf{Without knowledge}} \\ \hline
$C_I$                & $\alpha$ & Cert. Robustness          & Cert. Ratio          & Cert. Robustness            & Cert. Ratio           \\ \hline \hline
\multirow{4}{*}{0.50} & 10\%     & \textbf{0.8288}       & \textbf{0.9449}      & 0.4762                  & 0.4762                \\
                      & 20\%     & \textbf{0.7407}       & \textbf{0.8488}      & 0.4749                  & 0.4749                \\
                      & 30\%     & \textbf{0.6907}       & \textbf{0.7968}      & 0.4736                  & 0.4736                \\
                      & 50\%     & \textbf{0.5581}       & \textbf{0.6395}      & 0.4635                  & 0.4635                \\ \hline
\multirow{4}{*}{1.00} & 10\%     & \textbf{0.7307}       & \textbf{0.8448}      & 0.1679                  & 0.1679                \\
                      & 20\%     & \textbf{0.5285}       & \textbf{0.6336}      & 0.1615                  & 0.1615                \\
                      & 30\%     & \textbf{0.4347}       & \textbf{0.5375}      & 0.1612                  & 0.1612                \\
                      & 50\%     & \textbf{0.2624}       & \textbf{0.3318}      & 0.1584                  & 0.1584                \\ \bottomrule
\end{tabular}}

\vspace{-1.5em}

\end{table}

%\begin{figure*}[t]
%\centering
%\includegraphics[width=1.0\textwidth]{figures/primatenet_0.12.pdf}
%\vspace{-1em}
%\caption{\small {\bf (PrimateNet)} Histogram of the \textbf{robustness margin} (difference between the probability of the correct
%class (lower bound) and the top wrong class (upper bound)) under perturbation.
%If such a difference is positive, it means
%that the classifier makes the right prediction
%under perturbation. Evaluation is made under smoothing parameter $\sigma=0.12$ with $\ell_2$ perturbation scale $C_I = \{\sigma, 2\sigma\}$. The ratio of the attacked sensors $\alpha$ equals to $10\%, 20\%, 30\%, 50\%$ from Left to Right.}
%\label{fig:primatenet-0.12}
%\vspace{-1.5em}
%\end{figure*}

%will bring penalties to ``Siamang" worlds' weight and Sensing-Reasoning pipeline would pay less attention predicting to ``Siamang" or even flip back to ``Gorilla" class under such attack.

\textbf{Evaluation Results.}
We evaluate the robustness of the \sr  compared with the baseline which is consist of 25 randomized smoothed sensing models for each Primate categories (without knowledge). We evaluate the average certified robustness of both under benign and adversarial scenarios with different smoothing parameter $\hat{\sigma}\in \{0.12, 0.25, 0.50\}$ and $\ell_2$ perturbation bound $C_I = \{\hat{\sigma}, 2\hat{\sigma}\}$. The evaluation results are shown in Table~\ref{tab:attack-imagenet-smooth-main} and  Table~\ref{tab:benign-imagenet}.
% We put more detailed results in our Appendix.

First, we evaluate both the \sr and the smoothed ML model with benign test data as shown in Table~\ref{tab:benign-imagenet}. It is interesting that the \sr with knowledge even outperforms the single model without knowledge about $0.7\%$ over different randomized smoothing parameter $\sigma$. It shows that even without attacks, the knowledge could help to improve the classification accuracy slightly, indicating that the domain knowledge integration can help relax the tradeoff between benign accuracy and robustness.
 
Next, we evaluate the certified robustness of \sr and the smoothed ML model considering different smoothing parameters $\hat{\sigma}=\{0.12, 0.25, 0.50\}$ and the input perturbation bound $C_I=\{\hat{\sigma}, 2\hat{\sigma}\}$ in
 Table~\ref{tab:attack-imagenet-smooth-main}.
We can see that when the attack ratio of sensing models $\alpha$ is small, both the {Certified Robustness} and {Certified Ratio} of \sr are significantly higher than that of the baseline smoothed ML model. In the meantime, when the sensing attack ratio $\alpha$ is large (e.g. $50\%$) both the \sr and baseline smoothed ML model obtain low Certified Robustness and Certified Ratio, and their performance gap becomes small.
 
%  , given the smoothing parameter $\sigma=\{0.12, 0.25, 0.50\}$ and the input $\ell_2$ perturbation bound $C_I=\{\sigma, 2\sigma\}$ respectively, when the ratio of attacked sensing models $\alpha$ is small, both \emph{End-to-end accuracy} and \emph{certified ratio} of Sensing-Reasoning pipeline are much higher than the baseline ML models. In the meantime, when the ratio of attacked sensing models $\alpha$ is large (e.g $50\%$), both the Sensing-Reasoning pipeline and ML models obtain quite low \emph{End-to-end accuracy} and \emph{Certified ratio} and the performance gap becomes smaller.
 
%, while the ML model without knowledge achieves higher accuracy and robustness.
This is interesting and intuitive, since if a large percent of sensing models are attacked, such structure-based knowledge, for which the solution to a given regular expression is not unique, would have higher confidence to prefer the other (wrong) side of the prediction. As a result, it is interesting for future work to identify more ``robust" knowledge which is resilient against the large attack ratio of sensing models, in addition to the hierarchical structure knowledge.

We also find that when $C_I/\hat{\sigma}$ is small ($C_I=\hat{\sigma}$), the model with knowledge can perform consistently better than the baseline ML models. When $C_I/\hat{\sigma}$ is large ($C_I=2\hat{\sigma}$), the performance gap becomes even larger. This phenomenon indicates that \sr could demonstrate its strength of robustness compared to the traditional smoothed DNN against an adversary with stronger ability.

To further evaluate the strength of our certified robustness, we calculate the \textit{robustness margin} --- the difference between the lower bound of the true class probability and the upper bound of the top wrong class probability under different perturbation scales --- to inspect the robustness certification (larger difference infer stronger certification). Figure~\ref{fig:primatenet-0.25} shows the histogram of the robustness margin for the model with and without knowledge under smoothing parameter $\hat{\sigma}=0.25$ and different perturbation scale $C_I$. We leave histogram figures under other $\sigma$ settings in Appendix.

\begin{figure*}[!t]
\centering
\includegraphics[width=0.9\textwidth]{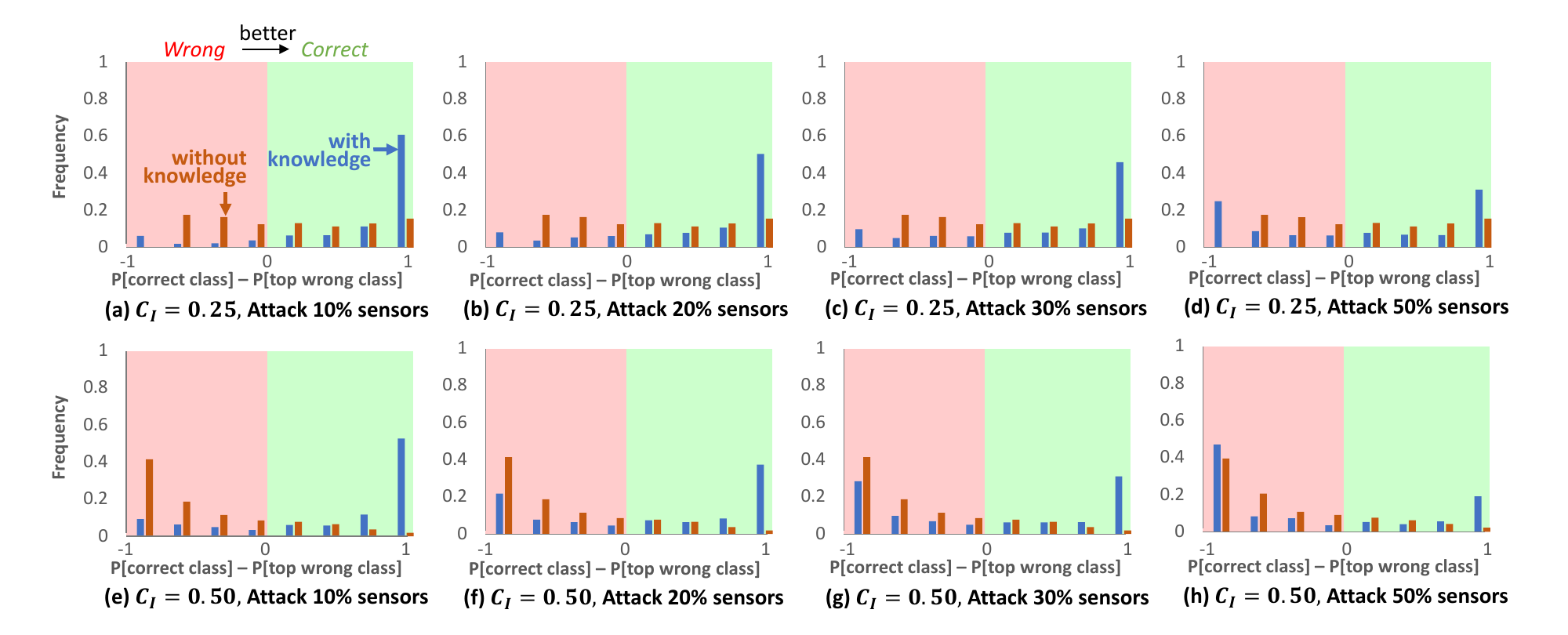}
\vspace{-1.5em}
\caption{\small {\bf (PrimateNet)} Histogram of the \textbf{robustness margin} (difference between the probability of the correct
class (lower bound) and the top wrong class (upper bound)) under perturbation.
If such a difference is positive, it means
that the classifier makes the right prediction
under perturbation. Evaluation is made under smoothing parameter $\sigma=0.25$ with $\ell_2$ perturbation scale $C_I = \{\sigma, 2\sigma\}$. The ratio of the attacked sensors $\alpha$ equals to $10\%, 20\%, 30\%, 50\%$.}
\label{fig:primatenet-0.25}
\vspace{-0.7em}
\end{figure*}

From Figure~\ref{fig:primatenet-0.25}, we can see that under different adversary scenarios, more instances could receive the positive margin (i.e correct prediction) with \sr, which indicates its robustness. Moreover, we find that the \sr could output a large margin value with high frequency under various attacks. That means, it can certify the robustness of the ground truth class with high confidence, which is challenging for current certified robustness approaches for single ML models.

In addition, to evaluate the utility of different knowledge, we also develop \sr by using only one type of 
knowledge (hierarchical or exclusive relationship only) and the results are shown in \Cref{sec:apx-partial-knowleddge}. We observe that using partial knowledge, the robustness of \sr would decrease compared with that using the full knowledge.

\section{Image Classification on Word50 Dataset}
\label{adx:word50-section}

\noindent\textbf{Task and Dataset.} In addition, we also conduct experiments on Word50 dataset~\cite{chen2015learning}, which is created by randomly selecting 50 words and each consisting of five characters. Here we only pick 10 words from it to reduce the computation complexity, and the goal is to classify these 10 words. All the character images are of size $28\times28$ and perturbed by scaling, rotation, and translation. The background of the characters is blurry by inserting different patches, which makes it a quite challenging task. For reference, Some word images sampled from the dataset are shown in~\Cref{fig:word10}. The interesting property of this dataset is that the character combination is given as the prior knowledge, which can be integrated into our \sr. The training, validation, and test sets contain $2,049$, $408$, and $423$ variations of word styles respectively. 

Similar to the classification task on Road Sign dataset, we develop the following two types of knowledge rules as follows:

\begin{itemize}
    \item \underline{Deduction rules} $(u, v_i)$: word $u$ contains character $v_i$ on the $i$th position of the word.
    \item \underline{Exclusion rules} $(u_i, v_i)$: character $u_i$ and character $v_i$ are naturally exclusive on the $i$th position of the word.
\end{itemize}

\noindent\textbf{Implementation details.} Multi-layer perceptrons (MLPs) are used as the main model architecture for the main task that the classification of the 10 words, which is the same to~\cite{chen2015learning}, and the input is the concatenation of the images of 5 characters which consist of a full word. As for the extra knowledge, we train another five MLP models for the classification of the character on each position of the input word, then the corresponding output dimensions for each such character classifier is $26$. While during the inference, we will only pick the top2 of the output from each character classifier, so the final input dimension to the MLN is $10+5\times2=20$ dimensions. Thus, to keep the certification probability the same as the baseline, the $\zeta_0$ here will be set to $1-(1-0.001)^{(1/20)}=5.002 \times 10^{-5}$.

For these sensing models, we adapt the randomized smoothing strategy~\cite{cohen2019certified} to give the certified robustness guarantee of their output confidence under the $\ell_2$-norm bounded perturbation. The $w_H$ is set to $2$ for the deduction rules, and the corresponding $f_H$ is the identity function; while for the exclusion rules, the $w_H$ is set to $-\infty$, and the $f_H$ here is the negation function, namely, $f_H(v)=1-v$.

\noindent\textbf{Knowledge.} We construct our first-order logical rules based on our predefined Deduction and Exclusion knowledge rules:

\begin{itemize}
    \item \underline{Deduction edge} $u \implies v_i$: 
    \begin{align}
         x_u \wedge \neg x_{v_i} = \texttt{False}
    \end{align}
    \item \underline{Exclusion edge} $u_i \oplus v_i$: 
        \begin{align}
        x_{u_i} \land x_{v_i} = \texttt{False}
    \end{align}
\end{itemize}

\noindent\textbf{Threat Model.} 
Same to the setting of the experiments on the Stop Sign dataset, here we consider a stronger attack scenario where the attacker can attack the main task model and all the attribute sensors with \textbf{different} $\ell_2$-norm bounded perturbation $\delta: ||\delta||_2 < C_I$ at the same time. Later on, we can see our \sr could still achieve higher end-to-end certified robustness under even harder cases.

Given the $\ell_2$-norm bound $C_I$, for each sensing model, we can bound its output probability $p'$ under such perturbation, given the original probability $p$ and the certification smoothing parameter $\sigma$ according to Corollary 2 as below:
$$
    p' \in [\Phi(\Phi^{-1}(p)-C_I/\sigma), \Phi(\Phi^{-1}(p)+C_I/\sigma)].
$$

\noindent\textbf{Evaluation Metrics.} We adopt the standard \emph{certified accuracy} as our evaluation metric, defined by the percentage of instances that can be certified under any $\ell_2$-norm bounded perturbation $\delta: ||\delta||_2 < C_I$. Specifically, given the input $x$ with ground-truth label $y$, we can certify the bound of confidence on predicting label $y$ as $[\mathcal{L}, \mathcal{U}]$ for either a vanilla randomize smoothing-based model or our \sr. After that, the certified accuracy can be defined by:
$ \frac{1}{N}\sum_{i=1}^{N}\mathbbm{I}([\mathcal{L}_i > 0.5])
$ where $\mathbbm{I}(\cdot)$ denotes the indicator function.

\noindent\textbf{Intuitive Example.} To make the example more clear, here we use~\emph{pos}(\emph{'a'}, i) to represent that the character \emph{'a'} is in the i\emph{th} position of the word. Then during the inference, given an input word image, we assume the top2 characters returned from the character classifiers for each position is \emph{'s,m'}, \emph{'n,b'}, \emph{'a,o'}, \emph{'q,a'}, \emph{'k,c'}, which are shown in the order of the position. Now, for word~\emph{'snack'}, the corresponding first-order logical form of its deduction rules would be $\emph{'snack'} \implies$ \emph{pos}(\emph{'s'}, $1$),  $\emph{'snack'} \implies$ \emph{pos}(\emph{'n'}, $2$),  $\emph{'snack'} \implies$ \emph{pos}(\emph{'a'}, $3$) and  $\emph{'snack'} \implies$ \emph{pos}(\emph{'k'}, $5$); while for other words like~\emph{'macaw'}, the corresponding rules would be $\emph{'macaw'} \implies$ \emph{pos}(\emph{'m'}, $1$) and $\emph{'macaw'} \implies$ \emph{pos}(\emph{'a'}, $4$). Notice, if the character of the specific word is not shown in the top2 returned characters of its corresponding position, then there will be no deduction rule built for this word and this character. At the meantime, when we consider the possible worlds that satisfy $ \sigma(x_{\emph{snack}}) \land \sigma(v_{\emph{pos}(\emph{'q'}, 4}) = 1$, we will still consider it as a violation of the exclusive rules. In other words, even if the character~\emph{'c'} is not shown in the top2 characters returned from the knowledge classifier in fourth position and thus we do not build the deduction rule $\emph{'snack'} \implies$ \emph{pos}(\emph{'c'}, $4$) explicitly at this time as said above, this rule is still assumed to be true underlyingly.

\noindent\textbf{Evaluation Results.} We evaluate the robustness of the \sr and compare it to the baseline as a vanilla randomized smoothed main task model (without knowledge). We train our models under different smoothing parameters $\hat{\sigma} = \{0.12, 0.25, 0.50\}$ and evaluate our \sr under various $\ell_2$ perturbation magnitude $C_I = \{0.12, 0.25, 0.50, 1.00\}$. Results are show in Table~\ref{tab:word10-exp}, and as we can see, with extra knowledge, the performance is improved tremendously which strongly demonstrates the potential of the \sr.

% Please add the following required packages to your document preamble:
% \usepackage{multirow}
% \usepackage[normalem]{ulem}
% \useunder{\uline}{\ul}{}

\begin{figure*}[!t]
	\begin{minipage}{0.5\textwidth}
\begin{table}[H]
\centering
\caption{Certified accuracy under different perturbation magnitude $C_I$ on Word10 dataset. The sensing models are smoothed with Gaussian noise $\epsilon \sim \mathcal{N}(0, \hat{\sigma}^2I_d)$ with different smoothing parameter $\hat{\sigma}$. Rows with $\ast$ denote the best certified accuracy among all the $\hat{\sigma} \in \{0.12, 0.25, 0.50\}$. (All certificates holds with 99.9\% confidence)}
\vspace{2mm}
\scalebox{0.55}{
\begin{tabular}{c|c|c|c|c|c}
\toprule
Methods                         & $\sigma$ & $C_I = 0.12$        & $C_I = 0.25$        & $C_I = 0.50$        & $C_I = 1.00$        \\ \hline
         & 0.12     & 58.2                & 49.2                & 0.0                 & 0.0                 \\
Vanilla Smoothing  & 0.25     & 51.8                & 42.3                & 25.3                & 0.0                 \\
(w/o knowledge)  & 0.50     & 42.6                & 33.1                & 19.1                & 2.6                \\ %\cline{2-6} 
                                            & $\ast$     & 58.2                & 49.2                & 25.3                & 2.6                \\ \hline
 & 0.12     & \textbf{88.7} & \textbf{77.8}       & \textbf{30.7}       & \textbf{0.0}       \\
Sensing-Reasoning Pipeline & 0.25     & {\ul\textbf{95.0}}       & {\ul \textbf{90.8}} & \textbf{52.5}       & \textbf{2.8}       \\
(w/ knowledge)  & 0.50     & \textbf{91.5}       & \textbf{86.8}       & {\ul \textbf{69.3}} & {\ul \textbf{6.4}} \\ %\cline{2-6} 
                                            & $\ast$     & \textbf{95.0}       & \textbf{90.8}       & \textbf{69.3}       & \textbf{6.4}       \\ \bottomrule
\end{tabular}}
\label{tab:word10-exp}
\end{table}
\end{minipage}
\hspace{5mm}
\begin{minipage}{0.45\textwidth}
\begin{figure}[H]
			\begin{center}
		\includegraphics[width=0.8\textwidth]{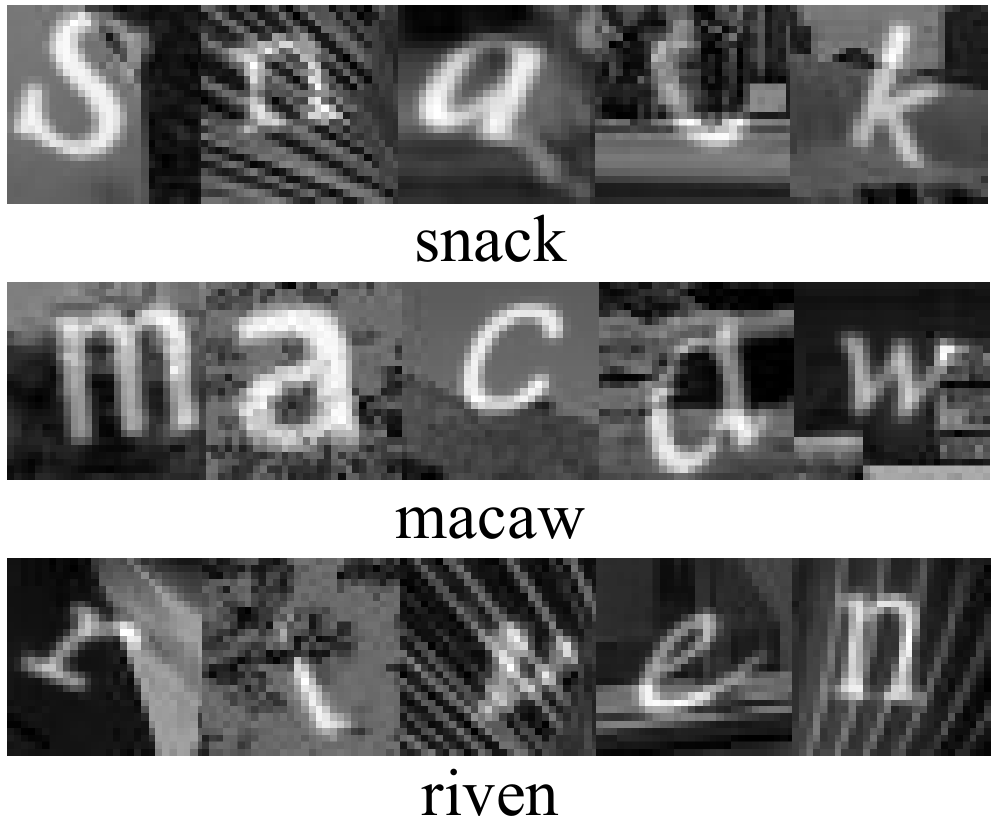}
			\end{center}
			\vspace{-4mm}
			\caption{Several word images sampled from Word50 dataset.}
			\label{fig:word10}
		\end{figure}
	\end{minipage}
\end{figure*}

\section{Image Classification with Constructed Knowledge Rules}
\label{adx:constructed-knowledge}

For natural image datasets with no apparent knowledge rules, we can still apply our \sr based on some generated simple knowledge rules such as redundancy rules. For instance, we test on MNIST and CIFAR10 dataset by constructing basic rules as follows: for \underline{MNIST}, we construct five pseudo attributes and randomly assign them to four different digits, so that each digit will exactly contain two pseudo attributes; for \underline{CIFAR10}, we randomly generate ten pseudo attributes, and each pseudo attribute will be randomly assigned to 3 to 7 different categories. We build the indication rules between each pseudo attribute and its corresponding digits, and the exclusion rules between different digit classes.

During the training, we adopt the SOTA Consistency training~\cite{jeong2020consistency} as our sensing model training method, and build our sensing-reasoning pipeline on top of these pretrained sensing models.

From the results shown in~\Cref{tab:mnist} and~\Cref{tab:cifar}, we can see that the sensing-reasoning pipeline beats the SOTA baselines in terms of the certified robustness even with the simple and generated knowledge rules. Generally, we should expect higher certified robustness by integrating with natural and meaningful knowledge rules (e.g., road sign classification and information extraction tasks as shown in our paper).

\begin{table}[!htbp]
\centering
\caption{ (\textbf{MNIST}) \textit{Certified accuracy} under different input perturbation magnitudes ($C_I$).  
}
\scalebox{0.58}{
\begin{tabular}{c|c|c|c|c|c|c|c|c|c}
\toprule
Methods                                   & $C_I=0.00$ & $C_I=0.25$ & $C_I=0.50$ & $C_I=0.75$ & $C_I=1.00$ & $C_I=1.25$ & $C_I=1.50$ & $C_I=1.75$ & $C_I=2.00$ \\ \hline
Consistency Training         & 99.5       & 98.9       & \bf 98.0       & 96.0       & 93.0 & 87.8 & 78.5 & 60.5 & 41.7       \\ \hline
Sensing-Reasoning Pipeline (Consistency) & \bf 99.6       & \bf 98.2       & 97.6       & \bf 96.3       & \bf 93.5 & \bf 88.2 & \bf 78.9 & \bf 61.2 & \bf 43.2       \\ \bottomrule
\end{tabular}}

\label{tab:mnist}
\end{table}

\begin{table}[!htbp]
\centering
\caption{ (\textbf{CIFAR10}) \textit{Certified accuracy} under different input perturbation magnitudes ($C_I$).  
}
\scalebox{0.58}{
\begin{tabular}{c|c|c|c|c|c|c|c|c|c}
\toprule
Methods                                   & $C_I=0.00$ & $C_I=0.25$ & $C_I=0.50$ & $C_I=0.75$ & $C_I=1.00$ & $C_I=1.25$ & $C_I=1.50$ & $C_I=1.75$ & $C_I=2.00$ \\ \hline
Consistency Training         & 77.8       & 68.8       & \bf 57.4       & 43.8       & 36.2 & 29.5 & 22.9 & 19.7 & 16.6       \\ \hline
Sensing-Reasoning Pipeline (Consistency) & \bf 78.4       & \bf 70.4       & 56.2       & \bf 46.0       & \bf 37.4 & \bf 29.6 & \bf 25.2 & \bf 21.8 & \bf 18.8       \\ \bottomrule
\end{tabular}}

\label{tab:cifar}
\end{table}

%\begin{table}[!htbp]
%\centering
%\caption{ (\textbf{MNIST}) \textit{Certified accuracy} under different input perturbation magnitudes ($C_I$).  
%}
%\scalebox{0.8}{
%\begin{tabular}{c|c|c|c|c|c}
%\toprule
%Methods                                   & $C_I=0.00$ & $C_I=0.25$ & $C_I=0.50$ & $C_I=0.75$ & $C_I=1.00$ \\ \hline
%Vanilla Smoothing         & 99.1       & 97.9       & 96.6       & 94.7       & 90.0       \\ \hline
%Sensing-Reasoning Pipeline & \bf 99.6       & \bf 99.2       & \bf 97.3       & \bf 95.4       & \bf 91.1       \\ \bottomrule
%\end{tabular}}

%\label{tab:mnist}
%\end{table}

\newpage
\section{Ablation Study on Partial Knowledge Enrichment.}
\label{sec:apx-partial-knowleddge}
In PrimateNet experiments, we also investigate how Hierarchy knowledge and Exclusive knowledge would affect the \emph{End-to-end} robustness of our \sr individually. We compare the certified robustness and certified ratio of our \sr enriched by \{No knowledge; Hierarchy knowledge only; Exclusive knowledge only; Hierarchy + Exclusive knowledge\} and the results are shown in Table~\ref{tab:apx-certified-robustness} and~\ref{tab:apx-certified-ratio}.

From the results, we can see while partial knowledge enrichment would lead to fragile robustness under severe scenarios ($\alpha=0.5$), complete knowledge enrichment could achieve much better robustness compared to \sr without knowledge enrichment. This indicates that incomplete (or weak) knowledge, which is easy to break and hard to recover under severe adversarial scenarios, could even harm the robustness of our \sr. How to explore good and robust knowledge to enrich our \sr could be our interesting future direction.

\begin{table}[!h]
\centering
% \scriptsize
\caption{\small \textbf{Certified Robustness} with different perturbation magnitude $C_I$ and sensing model attack ratio $\alpha$ on PrimateNet.  The sensing models are smoothed with Gaussian noise $\epsilon \sim \mathcal{N}(0, \hat{\sigma}^2 I_d)$ with different smoothing parameter $\sigma$. Here \textbf{``Hierarchy.''} refers to the \sr enriched by hierarchy knowledge only while \textbf{``Exclusive.''} the exclusive knowledge only. \textbf{``Combined.''} shows the \sr enriched by both domain knowledge. }
\label{tab:apx-certified-robustness}

{ \bf (a) $\hat{\sigma} = 0.12$}

\vspace{0.3em}

\scalebox{0.81}{
\begin{tabular}{c|c|cccc}
\toprule
$C_I$                & $\alpha$           & {\bf No knowledge} & {\bf Hierarchy.} & {\bf Exclusive.} & {\bf Combined.} \\ \hline \hline
\multirow{4}{*}{0.12} & 10\%     & 0.5724            & 0.7912          & 0.7020             & \bf 0.8849            \\
                      & 20\%     & 0.5717            & 0.6932          & 0.6236             & \bf 0.8078            \\
                      & 30\%     & 0.5706            & 0.6280          & 0.5624             & \bf 0.7508            \\
                      & 50\%     & 0.5706            & 0.4868          & 0.4320             & \bf 0.6236            \\ \hline
\multirow{4}{*}{0.25} & 10\%     & 0.2342            & 0.6670          & 0.5232             & \bf 0.7888            \\
                      & 20\%     & 0.2320            & 0.4704          & 0.3468             & \bf 0.6226            \\
                      & 30\%     & 0.2309            & 0.3632          & 0.3158             & \bf 0.5225            \\
                      & 50\%     & 0.2268            & 0.2122          & 0.2004             & \bf 0.3594            \\ \bottomrule
\end{tabular}}

\vspace{1em}

{ \bf (b)  $\hat{\sigma} = 0.25$}\\
 \vspace{0.3em}
 \scalebox{0.81}{
\begin{tabular}{c|c|cccc}
\toprule
$C_I$                & $\alpha$           & {\bf No knowledge} & {\bf Hierarchy.} & {\bf Exclusive.} & {\bf Combined.} \\ \hline \hline
\multirow{4}{*}{0.25} & 10\%     & 0.5314       & 0.7766      & 0.6998                  & \bf 0.8498                \\
                      & 20\%     & 0.5302       & 0.6810      & 0.6002                  & \bf 0.7608                \\
                      & 30\%     & 0.5294       & 0.6278      & 0.5464                  & \bf 0.7217                \\
                      & 50\%     & 0.5235       & 0.4924      & 0.4126                  & \bf 0.6026                \\ \hline
\multirow{4}{*}{0.50} & 10\%     & 0.2024       & 0.6754      & 0.5196                  & \bf 0.7622                \\
                      & 20\%     & 0.2024       & 0.4636      & 0.3298                  & \bf 0.5988                \\
                      & 30\%     & 0.2010       & 0.3680      & 0.2870                  & \bf 0.5324                \\
                      & 50\%     & 0.2000       & 0.2204      & 0.1652                  & \bf 0.3417                \\ \bottomrule
\end{tabular}}

 \vspace{1em}

{ \bf (c)  $\hat{\sigma} = 0.50$}\\
 \vspace{0.3em}
 \scalebox{0.81}{
\begin{tabular}{c|c|cccc}
\toprule
$C_I$                & $\alpha$           & {\bf No knowledge} & {\bf Hierarchy.} & {\bf Exclusive.} & {\bf Combined.} \\ \hline \hline
\multirow{4}{*}{0.50} & 10\%     & 0.4762       & 0.7412      & 0.6952                  & \bf 0.8288                \\
                      & 20\%     & 0.4749       & 0.6120      & 0.5884                  & \bf 0.7407                \\
                      & 30\%     & 0.4736       & 0.5410      & 0.5002                  & \bf 0.6907                \\
                      & 50\%     & 0.4635       & 0.4040     & 0.3862                  & \bf 0.5581                \\ \hline
\multirow{4}{*}{1.00} & 10\%     & 0.1679       & 0.6000      & 0.4838                  & \bf 0.7307                \\
                      & 20\%     & 0.1615       & 0.3834      & 0.3184                  & \bf 0.5285                \\
                      & 30\%     & 0.1612       & 0.2920      & 0.2362                  & \bf 0.4347                \\
                      & 50\%     & 0.1584       & 0.1498      & 0.1404                  & \bf 0.2624                \\ \bottomrule
\end{tabular}}

\vspace{-1.5em}

\end{table}

\begin{table}[!bh]
\centering
% \scriptsize
\caption{\small \textbf{Certified Ratio} with different perturbation magnitude $C_I$ and sensing model attack ratio $\alpha$ on PrimateNet.  The sensing models are smoothed with Gaussian noise $\epsilon \sim \mathcal{N}(0, \hat{\sigma}^2 I_d)$ with different smoothing parameter $\sigma$. Here \textbf{``Hierarchy.''} refers to the \sr enriched by hierarchy knowledge only while \textbf{``Exclusive.''} the  exclusive knowledge only. \textbf{``Combined.''} shows the \sr enriched by both domain knowledge. }
\label{tab:apx-certified-ratio}

{ \bf (a) $\hat{\sigma} = 0.12$}

\vspace{0.3em}

\scalebox{0.81}{
\begin{tabular}{c|c|cccc}
\toprule
$C_I$                & $\alpha$           & {\bf No knowledge} & {\bf Hierarchy.} & {\bf Exclusive.} & {\bf Combined.} \\ \hline \hline
\multirow{4}{*}{0.12} & 10\%     & 0.5724           & 0.8714          & 0.7320             & \bf 0.9419            \\
                      & 20\%     & 0.5717            & 0.7586          & 0.6442             & \bf 0.8609            \\
                      & 30\%     & 0.5706            & 0.6850          & 0.5928             & \bf 0.7988            \\
                      & 50\%     & 0.5706            & 0.5270          & 0.4642             & \bf 0.6647            \\ \hline
\multirow{4}{*}{0.25} & 10\%     & 0.2342            & 0.7330          & 0.5482             & \bf 0.8428            \\
                      & 20\%     & 0.2320            & 0.5150          & 0.3842             & \bf 0.6657            \\
                      & 30\%     & 0.2309            & 0.4011          & 0.3422             & \bf 0.5596            \\
                      & 50\%     & 0.2268            & 0.2322          & 0.2262             & \bf 0.3824            \\ \bottomrule
\end{tabular}}

\vspace{1em}

{ \bf (b)  $\hat{\sigma} = 0.25$}\\
 \vspace{0.3em}
 \scalebox{0.81}{
\begin{tabular}{c|c|cccc}
\toprule
$C_I$                & $\alpha$           & {\bf No knowledge} & {\bf Hierarchy.} & {\bf Exclusive.} & {\bf Combined.} \\ \hline \hline
\multirow{4}{*}{0.25} & 10\%     & 0.5314       & 0.9102      & 0.7254                  & \bf 0.9499                \\
                      & 20\%     & 0.5302       & 0.7910      & 0.6226                  & \bf 0.8952                \\
                      & 30\%     & 0.5294       & 0.7322      & 0.5878                  & \bf 0.8048                \\
                      & 50\%     & 0.5235       & 0.5670      & 0.4302                  & \bf 0.6747                \\ \hline
\multirow{4}{*}{0.50} & 10\%     & 0.2024       & 0.7998      & 0.5322                  & \bf 0.8489                \\
                      & 20\%     & 0.2024       & 0.5512      & 0.3490                  & \bf 0.6467                \\
                      & 30\%     & 0.2010       & 0.4440      & 0.3266                  & \bf 0.5541                \\
                      & 50\%     & 0.2000       & 0.2632      & 0.1734                  & \bf 0.3635                \\ \bottomrule
\end{tabular}}

 \vspace{1em}

{ \bf (c)  $\hat{\sigma} = 0.50$}\\
 \vspace{0.3em}
 \scalebox{0.81}{
\begin{tabular}{c|c|cccc}
\toprule
$C_I$                & $\alpha$           & {\bf No knowledge} & {\bf Hierarchy.} & {\bf Exclusive.} & {\bf Combined.} \\ \hline \hline
\multirow{4}{*}{0.50} & 10\%     & 0.4762       & 0.8924      & 0.7128                  & \bf 0.9449                \\
                      & 20\%     & 0.4749       & 0.7370      & 0.6144                  & \bf 0.8488                \\
                      & 30\%     & 0.4736       & 0.6552      & 0.5462                  & \bf 0.7968                \\
                      & 50\%     & 0.4635       & 0.4938      & 0.4324                  & \bf 0.6395                \\ \hline
\multirow{4}{*}{1.00} & 10\%     & 0.1679       & 0.7374      & 0.5204                  & \bf 0.8448                \\
                      & 20\%     & 0.1615       & 0.4906      & 0.3398                  & \bf 0.6336                \\
                      & 30\%     & 0.1612       & 0.3850      & 0.2926                  & \bf 0.5375                \\
                      & 50\%     & 0.1584       & 0.1996      & 0.1628                  & \bf 0.3318                \\ \bottomrule
\end{tabular}}

\vspace{-1.5em}

\end{table}

\section{Reasoning Component as Bayesian Networks}

A Bayesian network (BN) is a probabilistic graphical model that represents a set of variables and their conditional dependencies with a directed acyclic graph.
Let us first consider a Bayesian Network with tree structures, the probability of a random variable being 1 is given by 
\begin{align*}
\Pr{X=1,\{p_i\}}=\sum_{x_1,...x_n} P(1|x_1,...,x_n) \prod_i p_i ^{x_i}(1-p_i)^{1-x_i}.
\end{align*}
In the following subsections, we will prove a hardness result of checking robustness in general MLN and BNs and use the above definition to construct an efficient procedure to certify robustness for binary tree BNs.
\subsection{Hardness of Certifying Bayesian Networks}
Analogously with the above reasoning, we can also state the general hardness result for deciding the robustness of BNs:

\begin{theorem}[BN hardness]\label{BN hardness}
Given a Bayesian network with a set of parameters $\{p_i\}$, a set of perturbation parameters $\{\epsilon_i \}$ and threshold $\delta$, deciding whether
$$|\Pr{X=1; \{p_i\}}-\Pr{X=1; \{p_i+\epsilon_i \}}|<\delta$$ is at least as hard as estimating $\Pr{X=1; \{p_i\}}$ up to $\varepsilon_c$ multiplicative error, with $\epsilon_i = O(\varepsilon_c)$.
\end{theorem}
\begin{proof}
Let $\alpha=[p_i]$, $Q(\sigma)=X$ and $\pi_\alpha$ defined by the the probability distribution of a target random variable. Since $X\in \{0,1\}$, we have $\E[\sigma\sim\pi_{\alpha}]{Q(\sigma)}=\Pr{X=1; \{p_i\}}$. The proof then follows analogously from Theorem \ref{theorem:counting-to-robustness}.
\end{proof}

Based on the hardness analysis of the reasoning robustness, 
we can see that it is challenging to directly certify the robustness of the reasoning component.
However, just as we can approximately certify the robustness of single ML models~\cite{li2020sok}, in the next section, we will present and discuss how to approximately certify the robustness of the reasoning component, and we show that for some structures such as BN trees, the certification could even be tight.

\subsection{Certifying Bayesian Networks}%\label{sec:beyond}
Apart from MLNs, we also aim to reason about the robustness
for Bayesian networks with binary tree structures, and derive an efficient algorithm to provide the \textit{tight} upper and lower bounds of reasoning robustness. Concretely, we introduce the set of perturbation $\{\epsilon_i\}$ on $\{p_i\}$ and consider the maximum resultant probability: 
%\small{
\begin{small}
\begin{align*}
&\max_{\epsilon_1...\epsilon_n} \sum_{x_1,...x_n} P(1|x_1,...,x_n) \prod_i (p_i + \epsilon_i)^{x_i}(1-p_i - \epsilon_i)^{1-x_i}\nonumber\\
=&\max_{\epsilon_1...\epsilon_n}\sum_{x_1,...x_{n-1}} \Bigg(\prod_{i<n} (p_i + \epsilon_i)^{x_i}(1-p_i - \epsilon_i)^{1-x_i} \Bigg) \times \nonumber\\
&~~~~~~~~~~~~~~~~~~~~\Bigg( 
     P(1|x_1,...,x_{n-1}, 0) (1-p_n - \epsilon_n) \nonumber+ P(1|x_1,...,x_{n-1}, 1) (p_n + \epsilon_n)\Bigg) \nonumber\\
=&\max_{\epsilon_1...\epsilon_n}\sum_{x_1,...x_{n-1}} \left(\prod_{i<n} (p_i + \epsilon_i)^{x_i}(1-p_i - \epsilon_i)^{1-x_i} \right) \times \nonumber\\
&~~~~~~~~~~~~~~~~~~~~\Bigg(P(1|x_1,...,x_{n-1}, 0)+      
\Big(P(1|x_1,...,x_{n-1}, 1)\nonumber-P(1|x_1,...,x_{n-1}, 0)\Big) (p_n + \epsilon_n)
\Bigg).
\end{align*}
\end{small}
%}

In the above we have isolated the last variable in the expression. Without additional structure, the above optimisation over perturbation is hard as stated in Theorem \ref{BN hardness}. However, if additionally we require the Bayesian network to be binary trees, we show that the optimisation over perturbation and the checking of robustness of the model is trackable. We summarise the procedure for checking robustness of binary tree structured BNs in the following theorem with the proof.
\begin{lemma}[Binary BN Robustness]\label{Binary BN Robustness}
Given a Bayesian network with binary tree structure, and the set of parameters $\{p_i\}$, the probability of a variable $X=1$,
\begin{align*}
\Pr{X=1,\{p_i\}}=\sum_{x_1, x_2} P(1|x_1, x_2) \prod_i p_i^{x_i}(1-p_i)^{1-x_i}	
\end{align*}
is $\delta_b$-robust, where 
\begin{align*}
\delta_b &=\max\bigg\{\Big|\Pr{X=1,\{p_i\}}- F_{max}\Big|,\Big|\Pr{X=1,\{p_i\}}-F_{min}\Big| \bigg\},\nonumber\\
with~~~~~~
&F_{max} = \max_{y_1,y_2}A_0 + A_1(y_1 + y_2) + (A_2 - A_1)y_1y_2, \nonumber\\
&F_{min} = \min_{y_1,y_2}A_0 + A_1(y_1 + y_2) + (A_2 - A_1)y_1y_2, ~~~~ \nonumber\\
&s.t., y_i \in [p_i - C_i, p_i + C_i],
\end{align*}
Where $A_0=P(1| 0, 0)$, $A_1= P(1 | 0, 1)  - P(1 | 0, 0)$ and $A_2= P(1 | 1, 1)  - P(1 | 0, 1)$ are all pre-computable constants given the parameters of the Bayesian network.
\end{lemma}

\subsubsection*{Proof of Lemma \ref{Binary BN Robustness}}

\newtheorem*{lemmaBNR}{Lemma \ref{Binary BN Robustness}}
\begin{proof}
We explicitly write out the probability subject to perturbation,
\begin{small}
\begin{align*}
&\Pr{X=1,\{p_i+\epsilon_i\}} \\
=&\sum_{x_1, x_2} P(1|x_1, x_2) \prod_i (p_i+\epsilon_i)^{x_i}(1-p_i-\epsilon_i)^{1-x_i}\nonumber\\	
=&(p_1 + \epsilon_1) \sum_{x_2} P(1 | 1, x_2) (p_2 + \epsilon_2)^{x_2} (1 - p_2 - \epsilon_2)^{1- x_2} \nonumber\\
~&+ (1 - p_1- \epsilon_1) \sum_{x_2} P(1 | 0, x_2) (p_2 + \epsilon_2)^{x_2} (1 - p_2 - \epsilon_2)^{1- x_2} \nonumber\\ 
=&\sum_{x_2} P(1 | 0, x_2) (p_2 + \epsilon_2)^{x_2} (1 - p_2 - \epsilon_2)^{1- x_2} \nonumber\\
~&+ (p_1+ \epsilon_1) \left( \sum_{x_2} \Big( P(1 | 1, x_2) - P(1 | 0, x_2) \Big)(p_2 + \epsilon_2)^{x_2} (1 - p_2 - \epsilon_2)^{1- x_2} \right) \nonumber\\
=& \sum_{x_2} P(1 | 0, x_2) (p_2 + \epsilon_2)^{x_2} (1 - p_2 - \epsilon_2)^{1- x_2} \nonumber\\
~&+ (p_1+ \epsilon_1) \bigg( 
\Big( P(1 | 1, 1)  - P(1 | 0, 1) \Big) (p_2 + \epsilon_2)  + \Big( P(1 | 1, 0)  - P(1 | 0, 0) \Big) (1 - p_2 - \epsilon_2)\bigg) \nonumber\\
=& P(1| 0, 0) + \Big(P(1 | 0, 1) - P(1| 0, 0)\Big) (p_2 + \epsilon_2) + (p_1+ \epsilon_1)\times \nonumber\\
~& \bigg( 
P(1 | 1, 0)  - P(1 | 0, 0)  +  \Big(P(1 | 1, 1)  - P(1 | 0, 1)  -  P(1 | 1, 0)  + P(1 | 0, 0) \Big) (p_2 + \epsilon_2)\bigg).
\end{align*}\end{small}
It follows that the robustness problem boils down to finding the maximum and minimum of $F=A_0 + A_1y_2 + A_1y_1 + (A_2 - A_1)y_1y_2$, with $y_i=p_i+\epsilon_i$.
\end{proof}

Specifically, in order to compute $F_{max}$ and $F_{min}$, we take partial derivatives of F:
\begin{align*}
\frac{\partial F}{\partial y_1}=A_1+(A_2-A_1)y_2,\nonumber\\
\frac{\partial F}{\partial y_2}=A_1+(A_2-A_1)y_1.
\end{align*}
Setting $\frac{\partial F}{\partial y_1}=\frac{\partial F}{\partial y_2}=0$ leads to $y_1^*=y_2^*=\frac{A_1}{A_1-A_2}$. In order to check if $y_i^*$ correspond to maximum or minimum. evaluate $\frac{\partial^2 F}{\partial y_i^2}=A_2-A_1$. We have the following scenarios: 
\begin{itemize}
	\item If $y_i^*\in [p_i - C_i, p_i + C_i]$ and $A_2-A_1>0$, then $y_i^*$ correspond to a minimum.
	\item If $y_i^*\in [p_i - C_i, p_i + C_i]$ and $A_2-A_1<0$, then $y_i^*$ correspond to a maximum.
	\item If  $y_i^* \notin [p_i - C_i, p_i + C_i]$, then $y_i$ is monotonic in the range of $[p_i - C_i, p_i + C_i]$ and the maximum or minimum are found at $p_i \pm C_i$.
\end{itemize}
Having shown the robustness of probability of one node in the Bayesian network, the robustness of the whole network can be computed recursively from the bottom to the top.

\fi
%\input{appendix}

%%%%%%%%%%%%%%%%%%%%%%%%%%%%%%%%%%%%%%%%%%%%%%%%%%%%%%%%%%%%%%%%%%%%%%%%%%%%%%%
%%%%%%%%%%%%%%%%%%%%%%%%%%%%%%%%%%%%%%%%%%%%%%%%%%%%%%%%%%%%%%%%%%%%%%%%%%%%%%%

\end{document}

% This document was modified from the file originally made available by
% Pat Langley and Andrea Danyluk for ICML-2K. This version was created
% by Iain Murray in 2018, and modified by Alexandre Bouchard in
% 2019 and 2021 and by Csaba Szepesvari, Gang Niu and Sivan Sabato in 2022. 
% Previous contributors include Dan Roy, Lise Getoor and Tobias
% Scheffer, which was slightly modified from the 2010 version by
% Thorsten Joachims & Johannes Fuernkranz, slightly modified from the
% 2009 version by Kiri Wagstaff and Sam Roweis's 2008 version, which is
% slightly modified from Prasad Tadepalli's 2007 version which is a
% lightly changed version of the previous year's version by Andrew
% Moore, which was in turn edited from those of Kristian Kersting and
% Codrina Lauth. Alex Smola contributed to the algorithmic style files.